%% file: main.tex
\documentclass[12pt]{article}
\pdfoutput = 1

\usepackage{natbib}
\usepackage[utf8]{inputenc} % allow utf-8 input
\usepackage[T1]{fontenc}    % use 8-bit T1 fonts
\usepackage[unicode=true,psdextra]{hyperref}
\usepackage{url}            % simple URL typesetting
\usepackage{array, booktabs, multirow}       % professional-quality tables
\usepackage{amsfonts}       % blackboard math symbols
\usepackage{nicefrac}       % compact symbols for 1/2, etc.
\usepackage{microtype}      % microtypography
\usepackage{xcolor}         % colors

\usepackage{amsmath,bm, amsthm}
\usepackage{graphicx,subfig}
\usepackage{algorithm}
\usepackage{algorithmic}
\graphicspath{{figures/}}
\usepackage{subcaption}
\usepackage{arydshln}

\graphicspath{{figures/}}
\allowdisplaybreaks

\usepackage[margin=2.5cm, bottom=2.5cm]{geometry}
\usepackage{tikz}
\usetikzlibrary{bayesnet}
\usetikzlibrary{arrows}
\usepackage{color}
\usepackage{caption}
\usepackage{subcaption}
\usetikzlibrary{backgrounds}
\usetikzlibrary{shapes.geometric}
\tikzstyle{arrow} = [thick,->,>=stealth]
\usetikzlibrary{shapes.geometric, arrows}

\DeclareCaptionStyle{italic}[justification=centering]{labelfont={bf},textfont={it},labelsep=colon}
\usepackage{mathtools}
\mathtoolsset{showonlyrefs}
\usepackage{todonotes}
\usepackage{array}
\usepackage{thmtools,thm-restate}
\usepackage{pifont}
\newcommand{\cmark}{\ding{51}}
\newcommand{\xmark}{\ding{55}}

\newcolumntype{R}[1]{>{\raggedleft\let\newline\\\arraybackslash\hspace{0pt}}m{#1}}
\newcolumntype{L}[1]{>{\raggedright\let\newline\\\arraybackslash\hspace{0pt}}m{#1}}
\newcolumntype{C}[1]{>{\centering\let\newline\\\arraybackslash\hspace{0pt}}m{#1}}

\theoremstyle{definition}

\renewcommand{\deg}{\text{deg }}
\newcommand{\density}{\text{density}}

\newcommand{\Var}{\text{Var}}
\newcommand{\Cov}{\text{Cov}}

\mathtoolsset{showonlyrefs}
\usepackage{thm-restate}

\DeclareMathOperator*{\argmin}{arg\,min}

\newtheorem{theorem}{Theorem}[section]
\newtheorem{lemma}[theorem]{Lemma}

\newtheorem{proposition}[theorem]{Proposition}
\newtheorem{definition}[theorem]{Definition}

\newtheorem{assumption}[theorem]{Assumption}
\newtheorem{procedure}[theorem]{Procedure}

\newcommand{\cheng}[1]{}
\newcommand{\SK}[1]{}
\newboolean{twocolumn}
\setboolean{twocolumn}{false}  % or true

\tikzstyle{boxprocess} = [rectangle, rounded corners, text width=4cm,  minimum height=1cm,text centered, draw=black]

\title{Graphon Mixtures}
\author{Sevvandi Kandanaarachchi, Cheng Soon Ong}
\date{}

\begin{document}

\maketitle

\begin{abstract}
 Social networks have a small number of large hubs, and a large number of small dense communities.
    We propose a generative model that captures both hub and dense structures.
    Based on recent results about graphons on line graphs, our model is a graphon mixture,
    enabling us to generate sequences of graphs where each graph is a combination of sparse and dense graphs.
    We propose a new condition on sparse graphs (the max-degree), which enables us to identify hubs.
    We show theoretically that we can estimate the normalized degree of the hubs, as well as
    estimate the graphon corresponding to sparse components of graph mixtures.
    We illustrate our approach on synthetic data, citation graphs, and social networks,
    showing the benefits of explicitly modeling sparse graphs.
\end{abstract}

\input{1_introduction}

\input{2_notation_and_prelim}
\input{3_square_degree_related}
\input{4_graph_mixtures}

\input{5_special_case_mixtures}
\input{6_estimating_U}

\input{7_estimating_infinite_U}
\input{8_generalization_and_experiments}
\input{9_conclusions}
% \input{Definitions_for_later_use}
% \input{7_other_things}
% \input{7_conclusions}

\bibliographystyle{agsm}
\bibliography{references}

\newpage
\appendix
\input{10_A_Appendix_Background}
\input{10_B_Appendix_Mixtures}
\input{10_C_Appendix_Top_k_Degrees}
\input{10_D_Finite_U}
\input{10_E_Infinite_U}
\input{10_F_Experiments}
\end{document}

%% file: 1_introduction.tex
\section{Introduction}\label{sec:intro}
Edge density is a graph attribute that helps to characterize graph sequences into dense or sparse graphs. %Graphs are categorized by edge density into dense and sparse graphs.  
A graph sequence is called dense if the edges grow quadratically with the nodes, and it is called sparse if the edges grow sub-quadratically with the nodes. A characteristic of dense graphs is that nodes are more connected to each other.  A characteristic of sparse graphs is that they are  less connected. Simple examples of sparse graphs include stars, paths and rings.

Observed sparse graphs such as social networks exhibit two contrasting types of behavior: a small number of high degree nodes called \textit{hubs} inducing sparsity and a large number of small, dense communities \citep{Zang8280512,zhao2021community}. Such hubs and communities are also observed in complex multi-functional networks such as  neurological networks \citep{van2013network, schwarz2008community}. 
%Often sparse graphs of interest have considerable hub structure. For example, social networks have a small number of nodes with extremely large numbers of connections, while the vast majority of nodes have a much smaller number of connections. 
%In addition to the hub structure, community structure is also a common feature of social networks. 
% Typically, in a community most nodes are connected to each other. 
A hub, viewed in isolation can be thought of as a star graph and a  community by itself can be thought of as a dense graph. The prevalence of these two components in both biological and social networks is evidence that they are fundamental structures in evolving graphs. 
%is an interesting observation.   %Thus, many sparse graphs exhibit a combination of: (1) extreme sparsity induced by hub vertices and (2) small dense communities. 

Motivated by these considerations we consider graph mixtures -- a combination of two graphs, one sparse and one dense generated from two different types of graph generators: graphons.  A graphon is a symmetric and measurable function $W:[0,1]^2 \rightarrow [0,1]$. %A graphon can be thought of as a graph blueprint and can be used to generate graphs of an arbitrarily large size. 
Graphons are used to learn the underlying structure or representation of graphs \citep{xu2021learning} and recently graphons are used in neural operators and graph neural networks \citep{levie2023graphon, tieu2024temporal, cheng2023equivariant}. A graphon can be thought of as a graph blueprint and can be used to generate graphs of an arbitrarily large size. As a consequence of the Aldous-Hoover theorem \citep{aldous1981representations}, graphs generated from a graphon $W$ are dense.  Thus, while it straightforward to generate dense graphs from a graphon $W$, generating sparse graphs from a graphon is not quite that simple.  

There have been several extensions to graphons to model sparse graphs. \citet{Caron_Panero_Rousseau_2023} model graphs as exchangeable point processes and extend the classical graphon framework to the sparse regime. Their graphons are defined on $\mathbb{R}_+^2$ instead of on the unit square $[0,1]^2$. Research led by Borgs and Chayes tackle sparsity in different ways. For example, \citet{borgs2018sparse} consider `stretched' and `rescaled' graphons that can represent sparse graphs.  \citet{veitch2015class} introduce graphexes, a triple describing isolated edges, infinite stars and a graphon defined on $\mathbb{R}_+^2$. 
\citet{kandandOng2024graphons} model sparse graphs by considering graphons of line graphs. %Line graphs map edges to vertices and connects edges when edges in the original graph share a vertex. 
We use results from \citet{kandandOng2024graphons} and \citet{Janson2016321} to define graphon mixtures, which can generate sparse and dense graphs depending on the mixture properties.

Our approach for sparse graph generation considers line graphs. Line graphs, also known as non-backtracking graphs \citep{krzakala2013spectral} are obtained by mapping edges to vertices. A line graph of a graph $G_n$ denoted by $L(G_n)$ maps edges of $G_n$ to vertices of $L(G_n)$ and connects vertices in $L(G_n)$ if the corresponding edges in $G_n$ share a vertex (Figure \ref{fig:subfig1}). \citet{kandandOng2024graphons} showed that for a subset of sparse graphs, which they called \textit{square-degree}, the line graphs are dense. 
%the line graph operation transforms a sparse graph sequence to a dense sequence. 
Hence, for  square-degree graphs, by taking line graphs, we are back in the space of dense graphs which we know how to model using the standard graphon. However, this hardly tells us anything about the structure of the line graph graphon. \citet{Janson2016321} showed that line graph limits are disjoint clique graphs and further showed that disjoint clique graphs can be modelled by a sequence of positive numbers, which he called a \textit{mass-partition}.  %called a mass-partition $\bm{p} = (p_1, p_2, \ldots ) $ with $p_1 \geq p_2 \geq \ldots $ with $p_i \geq 0$ and $\sum_i p_i \leq 1$. 
%This result describes the structure of the line graph graphon well. 
This result combined with the inverse line graph operation enable us to define graphon mixtures. 

The contributions of this paper are: (1)  we propose graphon mixtures -- a novel approach that explicitly models sparsity and thereby generates graph mixtures; (2) given a graph mixture, we estimate the degree of hubs of unseen graphs and the graphon corresponding to the sparse components, and we provide theoretical guarantees on our estimates that have exponential or polynomial convergence; (3) we show empirically that our approach is useful on synthetic and real data. All proofs are presented in the Appendix.

%% file: 2_notation_and_prelim.tex
\section{Limits of sparse graphs and line graphs}
\label{sec:notation}

We review the setting of graphons, in particular with respect to the recent results on graphons on line graphs~\citep{kandandOng2024graphons}. %All proofs are presented in the Appendix. 

% \begin{definition}[\bf Dense graph sequences]\label{def:dense}
%     A sequence of graphs $\{G_n\}_n$ is dense if the number of edges $m$ grow quadratically with the number of nodes $n$, i.e., 
%     \begin{equation*}
%          \liminf_{n \to \infty} \frac{m}{n^2} = c > 0 \, .
%     \end{equation*}
% %    We denote the set of all dense graph sequences by $D$.
% \end{definition}
% \begin{definition}[\bf Sparse graph sequences]\label{def:sparse}
%      A sequence of graphs $\{G_n\}_n$ is sparse if the number of edges $m$ grow sub-quadratically with the number of nodes $n$, i.e., 
%      \begin{equation}
%          \lim_{n \to \infty} \frac{m}{n^2} = 0 \, . 
%      \end{equation}
% %    We denote the set of all sparse graph sequences by $S$.
% \end{definition}

\begin{definition}[\bf Dense and sparse graph sequences]\label{def:denseandsparse}
    A sequence of graphs $\{G_n\}_n$ is dense if the number of edges $m$ grow quadratically with the number of nodes $n$, i.e., $\liminf_{n \to \infty} \frac{m}{n^2} = c > 0 \, .$
    A sequence of graphs $\{G_n\}_n$ is sparse if the number of edges $m$ grow sub-quadratically with the number of nodes $n$, i.e., $ \lim_{n \to \infty} \frac{m}{n^2} = 0 \, . $
\end{definition}

 A graphon is a symmetric and measurable function  $W:[0,1]^2 \rightarrow [0,1]$. Graphs and graphons are linked by the adjacency matrix. % The \textit{empirical graphon} (Definition \ref{def:empiricalgraphon}) of a graph is obtained by replacing the adjacency matrix with a unit square where the $(i,j)$th entry of the adjacency matrix is replaced with a square of size $1/n \times 1/n$.  
 The \textit{empirical graphon} (Definition \ref{def:empiricalgraphon}) of a graph is obtained by scaling the adjacency matrix to the unit square and coloring smaller squares of size $1/n \times 1/n$ black if the corresponding entry in the adjacency matrix is 1 and coloring it white if the entry is 0.  Given a graphon $W$, a graph with $n$ nodes for any $n \in {\mathbb{N}}$ can be generated as follows:

\begin{definition}\label{def:wrandomgraphs}Uniformly pick $x_1, x_2, \ldots x_n$ from $[0,1]$. A \textbf{W-random graph} $\mathbb{G}(n,W)$ has the vertex set $1, 2, \ldots n$ and vertices $i$ and $j$ are connected with probability $W(x_i, x_j)$.
\end{definition}

Given a sequence of graphs $\{G_n\}_n$, convergence is defined using the \textit{cut norm} (Definition \ref{def:cut1}) and the \textit{cut metric} (Definition \ref{def:cut2}). \cite{borgs2008convergent} showed that every convergent graph sequence converges to a graphon $W$. Therefore, graphons are graph limits. The problem is that all sparse graph sequences converge to $W = 0$. Hence, the standard graphon is not representative of sparse graphs, in the sense that graphs generated from  $W = 0$ are isolated nodes without any edges. 

\subsection{Sparse graphs are stars in the limit}

\begin{figure}[ht]
    \centering
    % First subfigure
    \subfloat[]{%
        \includegraphics[width=0.45\linewidth,trim= 0 0.5cm 0.8cm 0.3cm ,clip]{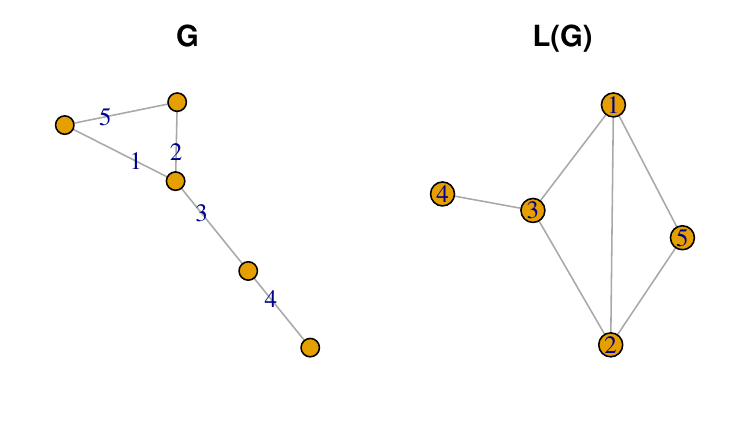}%  % , height = 0.1\textheight
        \label{fig:subfig1}%
    }
    \hfill  % Create horizontal space between subfigures
    % Second subfigure
    \subfloat[]{%
        \includegraphics[width=0.45\linewidth, trim=0 0.5cm 0.5cm 0.3cm, clip]{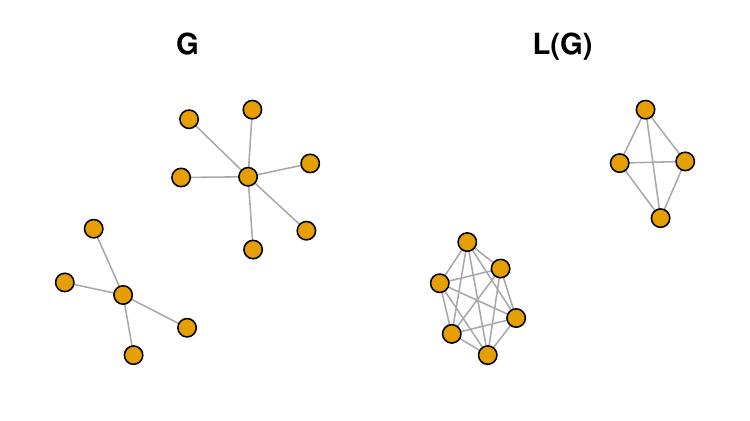}%
        \label{fig:subfig2}%
    }
    \caption{(a) A graph $G$ and its line graph $L(G)$. Vertices in $L(G)$ are edges of $G$. Vertices of $L(G)$ are connected if the corresponding edges in $G$ share a vertex. (b) Line graphs of disjoint stars are disjoint cliques. Thus, the inverse line graph  of disjoint cliques are disjoint stars. }
    \label{fig:graphsandlinegraphs}
\end{figure}

To address this problem, more mathematically intricate constructions and convergences have been defined \citep{borgs2018sparse, Caron_Panero_Rousseau_2023}. \cite{kandandOng2024graphons} showed that if sparse graphs satisfy a certain condition, then their line graphs (Definition \ref{def:linegraph}, Figure \ref{fig:subfig1}) are dense.  They called this condition the \textit{square-degree property} (Definition \ref{def:square}). The set of graph sequences satisfying the square-degree property is denoted by $S_q$. They showed that if $\{G_n\}_n \in S_q$ and the line graphs $\{L(G_n)\}_n$ converge, then they converge to a graphon $U \neq 0$. Here, we elucidate the structure of $U$, based on an earlier result by \cite{Janson2016321}.

% \begin{definition}[\bf Square-degree property]\label{def:square} Let $\{G_n\}_n$ denote a sequence of graphs.  Then $\{G_n\}_n$ exhibits the square-degree property if there exists some $c_1  > 0 $ and $N_0 \in \mathbb{N}$ such that for all $n \geq N_0$ 
% \[ \sum \deg v_{i, n}^2 \geq c_1 \left( \sum \deg v_{i, n} \right)^2  \, . 
% \]  
% % Alternatively, 
% % \[ \liminf_{n \to \infty} \frac{\sum \deg v_{i, n}^2 }{ \left(\sum \deg v_{i, n} \right)^2} = c_1 > 0 \, . 
% % \]
% The set of graph sequences satisfying the square-degree property is denoted by $S_q$. %, i.e. if $\{G_n\}_n$ satisfies $Sq$ then $\{G_n\}_n \in S_q$.
% \end{definition}

 % This ties to the work done by \cite{Janson2016321}

\begin{theorem}\label{thm:Jansonmain} {\bf(\cite{Janson2016321} Thm 8.3) }
 A graph limit is a line graph limit if and only if it is a disjoint clique graph limit.     
\end{theorem}

A clique is a subset of nodes in the graph that are connected to each other. This result is important because as shown in Figure \ref{fig:subfig2}, the inverse line graph of a disjoint clique graph is a disjoint star graph, and stars are the most sparse of the graphs that exhibit power-law degrees. Therefore, the line graph limit explains sparsity in a fundamental way via the inverse line graph operation. It is worth noting that if a graph is a line graph, then its inverse exists and is unique \citep{whitney1932congruent} with the only exception when the line graph is a triangle. 
\cite{Janson2016321} represented disjoint clique graph limits by a sequence of proportions called the \textit{mass-partition} (Definition \ref{def:masspartitionJanson}) where the $j$th element of the mass-partition gives the proportion of nodes in the $j$th disjoint clique. A mass-partition $\bm{p} = \{p_i\}_{i = 1}^{\infty}$  uniquely defines a disjoint clique graphon 
$W_{\bm{p}}^\mathcal{M}$ (Definition \ref{def:masspartitiongraphon}) resembling a block-diagonal matrix (see $U$ in Figure \ref{fig:infographic}).   
We combine disjoint clique graph limits of \citet{Janson2016321} with  converging sparse square-degree graphs $S_q$ (Definition \ref{def:square})  of \citet{kandandOng2024graphons} in the following lemma.
Furthermore, we provide a new definition (so called max-degree) which is equivalent to square-degree property.

\begin{lemma}\label{lemma:converginginSq}
      Let $\{G_n\}_n \in S_q$ (Definition \ref{def:square}) and let $H_m = L(G_n)$ where $L$ denotes the line graph operation. If $\{H_m\}_m$ converges to $U$ then $U$ is a disjoint clique graphon and there exists a unique mass-partition $\bm{p}$ that describes $U$. 
      That is, $U = W_{\bm{p}}^\mathcal{M}$ (see Definition \ref{def:masspartitiongraphon}) for some mass-partition $\bm{p}$. 
\end{lemma}
\begin{proof}
Immediate from Theorem \ref{thm:Jansonmain} and \ref{thm:Janson1} \citep{Janson2016321}.
\end{proof}

\begin{definition}{\bf (Max-degree condition)}\label{def:maxdegreetoedges}
    Let $\{G_{n}\}_n$ be a sequence of graphs where $G_{n}$ has $n$ nodes and $m$ edges. Let the maximum degree of $G_{n}$ be denoted by $d_{\max, n}$. We say the $\{G_{n}\}_n$ satisfies the max-degree condition if there exists $c > 0$ 
    % and $N_0 \in \mathbb{N}$ such that
    % \[  \frac{d_{\max, n}}{m} \geq c  > 0\,  
    % \]
    % for all $n > N_0$. Alternatively, 
    \[ \liminf_{n, m \to \infty} \frac{d_{\max, n}}{m} = c > 0 \, . 
    \]
    Let $S_x$ denote the set of graph sequences satisfying the max-degree condition. 
\end{definition}

We show that for graph sequences with converging line graphs, the square-degree property is equivalent to the max-degree condition

\begin{restatable}{lemma}{lemmasquareandmaxequivlence}\label{lemma:squareandmaxequivlence}
    Let $\{G_n\}_n$ be a graph sequence with $H_m = L(G_n)$. Suppose $\{H_m\}_m$ converges to $U$.  Then
    \[  \{G_n\}_n \in S_q \equiv \{G_n\}_n \in S_x \, . 
    \]
\end{restatable}

\subsection{Related work}

% \cheng{We don't seem to have anywhere where we discuss other ways to estimate hubs. Reviewers will expect related work, so let's give them some. It will help set the frame with which reviewers will read our paper. Could move some material from the introduction to here.}

Sparse graphs including scale-free and power-law graphs have been studied in depth. For scale-free graphs described by the Barabási-Albert model  \citep{barabasi1999emergence} the proportion $P(d)$ of vertices with degree $d$ obeys a power law $P(d) \propto d^{-\gamma}$, where $\gamma \approx 3$. \citet{bollobas2001degree} showed for these graphs the maximum degree is $\Theta(\sqrt{n})$, where $n$ is the number of the nodes in the graph. %For these graphs the probability $\Pi(i)$ that a new node connects to node $i$, which has degree $d_i$ is given by $\Pi(i) = \frac{d_i^{\alpha}}{\sum_{d_i^\alpha}}$ where $\alpha = 1$. When $\alpha > 1$, the graphs are described by a different regime called the super-linear preferential attachment. For these graphs, the maximum degree $d_{\max} \to n$ exhibiting a winner takes it all strategy while other degrees remain bounded \citep{Sethuraman2019}. %Thus, for $\alpha = 1$ the maximum degree  $d_{\max} \in \Theta(\sqrt{n})$ and for $\alpha > 1$ , $d_{\max} \to n$. %, and this model does not describe an in-between   

\cite{leskovec2010kronecker} proposed a graph generative model called Kronecker graphs. They modeled graph evolution using the Kronecker product and had two versions (deterministic and stochastic) of graph generators. They showed they can generate graphs with heavy-tailed degree distributions and other such properties.

%In contrast, c
Classical graphons were used to describe dense graphs. Over time the original definitions, constructions and convergences were extended to incorporate sparse graphs \citep{Borgs2021}. %One such construction is graphexes. Graphexes describe a triple $(I, S, W)$, where $I$ describes isolated edges, $S$ infinite stars and $W$ a graphon defined on $\mathbb{R}_+^2$. This particular graphon $W$ can generate dense or sparse graphs. For dense graphs, the graphon is defined on a compact $S \subset \mathbb{R}_+^2$, while for sparse graphs  $W$ is defined on  $\mathbb{R}_+^2$. Thus, extending the support of the graphon from a finite subset to the whole positive quadrant enabled sparsity in this instance. 
\cite{caron2017sparse} and \cite{Caron_Panero_Rousseau_2023} modeled graphs as exchangeable point processes where the graphon  $W$ is defined on  $\mathbb{R}_+^2$.  These were also known as graphex processes. They used Kallenberg exchangeability in their theoretical framework and sparsity was achieved by extending the support of the graphon from a finite subset to the whole positive quadrant.

Our approach is different to the above as we use inverse line graphs of disjoint clique graphs to generate the sparse component. As both graphons in our mixture are defined on the unit square, it makes it easier to estimate these graphons in practice, especially the sparse component. By varying mixture properties we can generate both dense and sparse graphs. 

%% file: 4_graph_mixtures.tex
\section{Mixtures of sparse and dense graphs}
% ---------------------------------------------------------------------------

\begin{figure*}
    \includegraphics[width=0.95\linewidth,trim= 0 3cm 0cm 2cm ,clip]{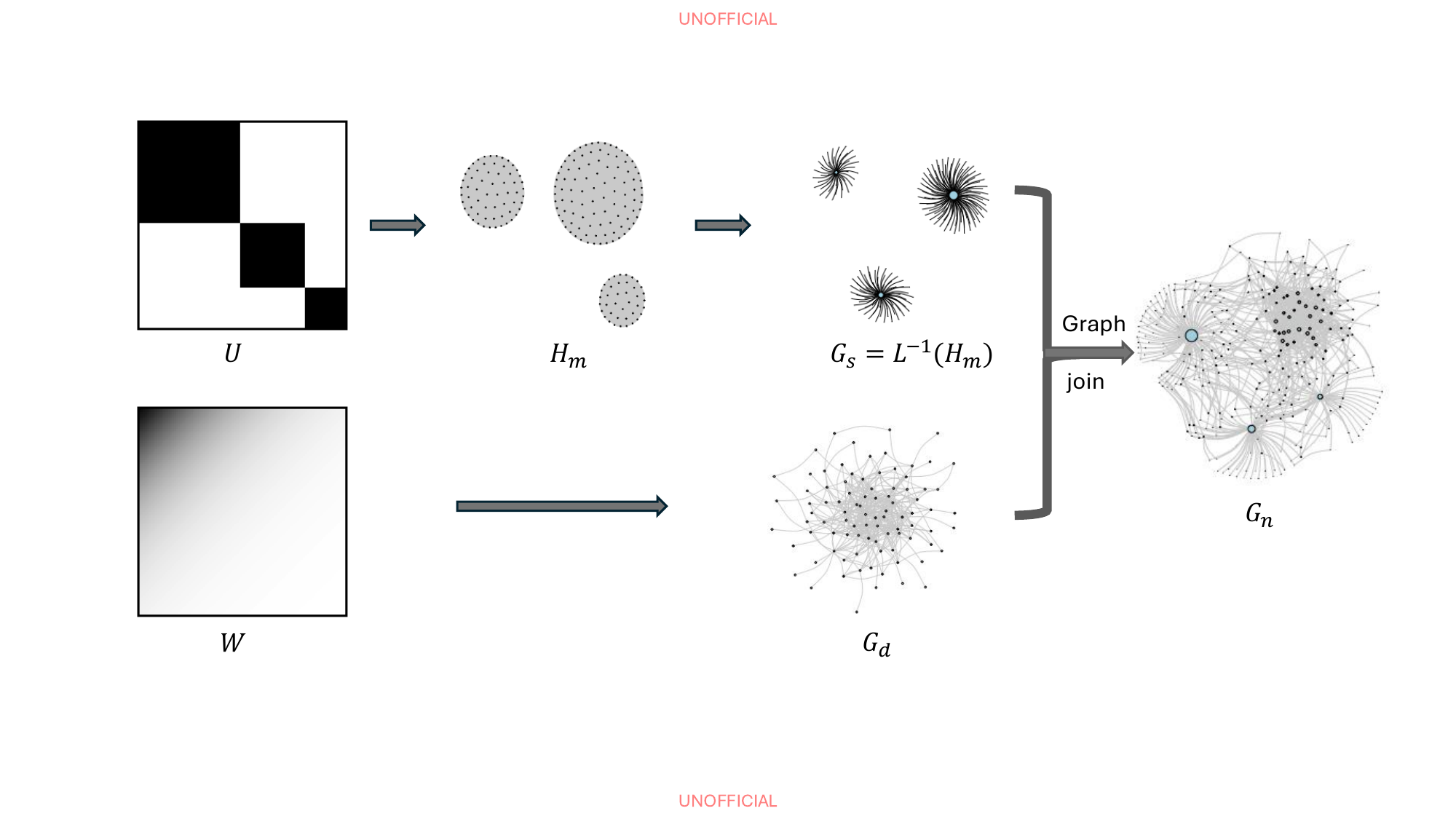}
    \caption{The overview of $(U,W)$-mixture graphs. A disjoint clique graph $H_m$ is sampled from graphon $U$. The inverse line graph $G_s = L^{-1}(H_m)$ is a disjoint set of stars. A graph $G_d$ is sampled from graphon $W$. Then, $G_s$ and $G_d$ are joined according to graph joining rules resulting in the mixture graph $G_n$. }
    \label{fig:infographic}
\end{figure*}

%\cheng{If we can, we should try to choose only one of $n_{s_i}$ or $m_{s_i}$ to reduce cognitive load on the reader.}

We generate a graph mixture by picking a graphon $W$ and a disjoint clique graphon $U$ (see Figure \ref{fig:infographic}).  %Graphon $U$ lives in the line graph space and is the key to generating sparse graphs.  By taking the inverse line graph of graphs generated by $U$, we get sparse graphs. Thus, $U$ models sparsity by the inverse line graph operation. 
First we generate two graph sequences using $U$ and $W$, noting that both graph sequences are dense, but the graph sequence generated by $U$ lives in the line graph space. The sparse graph sequence is obtained by taking the inverse line graphs of the graphs generated by $U$. Then we join the $i$th  graph generated by $W$ with the $i$th inverse line graph generated by $U$ according to a set of simple joining rules and obtain the graph mixture. The results in Sections \ref{sec:topk}, \ref{sec:estimatingUfinite} and \ref{sec:EstimatingUinfinitePartitions} do not have any terms pertaining to the join. They hold if we consider the disjoint union of the two graph sequences. However, the real graphs in Table 2 are not disjoint unions. As such, we consider a simple way to join nodes randomly without adding a lot of extra edges.  We ensure the main structural parts are contributed by $U$ and $W$, but not the join. 

A key component in the mixing process is the ratio of the number of nodes in the sparse graph to that of the dense graph. If this ratio goes to infinity along the sequence, then the mixture sequence is sparse. If the ratio is bounded, then the mixture sequence is dense. Thus, for a given $W$ and $U$ by changing the mixture ratio function, we can generate graph sequences ranging from dense to sparse.  

%We consider mixtures of sparse and dense graphs as both hubs and communities are commonly observed. Hubs can be represented by inverse line graphs of disjoint clique graphs, and a dense community structure can be represented by sampling from a classical graphon $W$.  We propose a graph mixture model that generates a dense and a sparse graph sequence and joins them according to some simple rules.

% We do not assume the structure of $W$ other than in a weak sense, which is detailed 

\subsection{The graph mixture model}\label{sec:mixturemodel}

A disjoint clique graphon $U$ is defined by a sequence of proportions (the mass-partition) $\bm{p} = \{p_i\}_{i=1}^{\infty}$. \cite{Janson2016321} considered $\sum_i p_i \leq 1$. In our work we consider a rescaled version of the original mass-partition and let the elements of the mass-partition add up to 1. We explore the inverse line graphs of graphs generated by $U$ when $\sum_i p_i < 1$ in Appendix \ref{sec:whenpileq1}. 
\begin{definition}\label{def:masspartition}(\textbf{Mass-partition})
    We define a mass-partition to be a sequence $\bm{p} = \{p_i\}_{i = 1}^{\infty}$ of non-negative real numbers such that
    \[ p_1 \geq p_2 \geq \cdots  \geq p_i \geq \cdots  \quad \text{and} \quad \sum_{i = 1}^\infty p_i = 1 \, . 
    \]
     We refer to the number of non-zero elements in $\bm{p}$ as partitions. Thus $U$ corresponding to $\bm{p}$ can have finite or infinite partitions. 
\end{definition}

\begin{definition} \label{def:WURandomMixtureGraphs}
    Given a graphon $W$, a disjoint clique graphon $U$, and positive, integer valued sequences $\{n_{d_i}\}_i$ and $\{m_{s_i}\}_i$ increasing with $i$, a \textbf{$\bm{(U,W)}$-mixture graph sequence} $G_{n_i} \sim \mathbb{G}\left(U,W, n_{d_i}, m_{s_i}\right)$ is constructed as follows: 
    \begin{enumerate}
        % \item Let $W$ be a graphon and $U$ be a disjoint clique graphon. %and $f:\mathbb{R}^+ \rightarrow \mathbb{R}^+$ is a monotone function. 
        % Let $\{n_{d_i}\}_i$ and $\{m_{s_i}\}_i$ be two positive, increasing, integer valued sequences with both $n_{d_i}$ and $m_{s_i}$ going to infinity. % satisfying $n_{d_i}/m_{s_i} =f(i)$. 
        \item Let $G_{d_i} \sim \mathbb{G}(n_{d_i}, W)$ and $H_{s_i} \sim \mathbb{G}(m_{s_i}, U)$, i.e., $G_{d_i}$ is a $W$-random graph (Definition \ref{def:wrandomgraphs}) and $H_{s_i}$ is a $U$-random graph. Sequences $\{G_{d_i}\}_i$ and $\{H_{s_i}\}_i$ consist of dense graphs. 
     \item As $U$ is a disjoint clique graphon, the inverse line graph of $H_{s_i}$ exists and is a disjoint union of stars. Suppose $G_{s_i} = L^{-1}(H_{s_i})$ has $n_{s_i}$ nodes where $L^{-1}$ denotes the inverse line graph operation. Thus, $\{G_{s_i}\}_i$ is a sparse graph sequence. % while $\{G_{d_i}\}_i$ is dense.
     %Thus, we have two graph sequences: $\{G_{d_i}\}_i$, a dense graph sequence and $\{G_{s_i}\}_i$, a sparse graph sequence. 
     \item Join $\{G_{d_i}\}_i$ and $\{G_{s_i}\}_i$ satisfying graph joining rules in Definition \ref{def:joiningrules}. Let $\{G_{n_i}\}_i$ denote the resulting graph sequence. % Let $G_n = \Join(G_d, G_s)$ denote the merged graph of $G_{d}$ and $G_{s}$  (Definition \ref{def:graphmerge}) where $G_n$ has $n$ nodes with $n = \max(n_d, n_s)$. 
    \end{enumerate}
     We say  $G_{n_i} \sim \mathbb{G}\left(U,W, n_{d_i}, m_{s_i}\right)$ is a  $(U,W)$-mixture graph. We call $G_{d_i}$ the \textbf{dense part} and $G_{s_i}$ the \textbf{sparse part} of $G_{n_i}$.  
     %Furthermore, we define the \textbf{mixing function} \cheng{Shall we define mixing functions closer to when we need them?} as $f(n_{d_i}, n_{s_i}) = (f_1(i), f_2(i))$ where $f_1(i) = n_{s_i}/n_{d_i}$ and $f_2(i) = n_{s_i}/n_{d_i}^2$. 
\end{definition}

Given two graph sequences -- one dense and one sparse --  we consider a simple way to randomly join a dense graph and a sparse graph with a given number of edges, making sure the main structural parts are contributed by the two graphs but not the join.   We formally define the joining rules below. 

\begin{definition}\label{def:joiningrules}\textbf{\textbf{(Graph Joining Rules)}}
    Let $\{G_{s_i}\}_i$ and $\{G_{d_i}\}_i$ be two graph sequences. We are interested in joining $G_{s_i}$ and $G_{d_i}$ for every $i$. Let $G_{n_i}$ denote the joined graph. We consider graph joins that satisfy the following conditions. 
    \begin{enumerate}
        \item \textbf{No new nodes}: no new nodes are added as part of the joining process. %This condition ensures that each node in $G_{n_i}$ comes either from $G_{s_i}$ or $G_{d_i}$.
         \item \textbf{No deletion}: No nodes or edges are deleted as part of the joining process.     
     \item \textbf{Random edges}\label{cond:randomedges}: if edges are added, they are added randomly in the sense that nodes within a graph are equally likely to be selected to form edges. %That is, if $G_{s_i}$ and $G_{d_i}$ have $n_{s_i}$ and $n_{d_i}$ nodes respectively, the probability that an edge is added to a given node in $G_{s_i}$, denoted by $p_{edge, G_{s_i}}$ satisfies $p_{edge, G_{s_i}} \leq \frac{1}{n_{s_i}}$ and similarly $p_{edge, G_{d_i}} \leq \frac{1}{n_{d_i}}$.
        % This condition facilitates randomness in the joining process and prevents creating additional structure in the graph. 
     \item \textbf{New edges}: 
        if edges are added, the number of new edges added  $m_{new_i}$ satisfies $m_{new_i} = c m_{d_i}$, for small $c \in \mathbb{R}$, where $ m_{d_i}$ denotes the number of edges in $G_{d_i}$. %As different values of $c$ can be chosen it gives flexibility to the joining process, this condition stipulates an upper bound on the edges, limiting the influence of the joining process.
    
        % \item \textbf{Rank preserving condition}: with high probability, the joining process preserves the ranks of the top $k$ degrees in $G_{s_i}$ for some $k \in \mathbb{N}$. That is, with high probability, the highest degree node in $G_{s_i}$ becomes the highest degree node in $G_{n_i}$ after the joining process and this holds for the top $k$ node degrees. This condition facilitates the recoverability of $G_{s_i}$ from $G_{n_i}$. 
    \end{enumerate}
\end{definition}
In defining these rules our motivation is to keep the joining process as simple as possible. As such, we have employed a random join. %Furthermore, the joining process only adds edges; new nodes are not added. 
These rules can be further explored and modified to suit different needs that result in more targeted mixtures.

\subsection{Expectations and edge density}

In graphon mixtures, randomness arise in two ways: (1) when the graphs are generated from $U$ and $W$ and (2) when the graphs are joined according to graph joining rules. We compute expectations of hub nodes and nodes in the dense part with respect to both sources of randomness in Lemmas \ref{lemma:expnewedges} and \ref{lemma:WUrandomgraphsaboutU2}. 
Depending on the mixing ratio $\frac{n_{s_i}}{n_{d_i}}$ the $(U,W)$ mixture graphs can be sparse or dense as shown in Lemma \ref{lemma:WUrandomgraphs1}.

\begin{restatable}{lemma}{lemmaWUrandomgraphsOne}\label{lemma:WUrandomgraphs1} Let $\{G_{n_i}\}_i$ be a sequence of $(U,W)$-mixture graphs (Definition \ref{def:WURandomMixtureGraphs}) with $G_{n_i} \sim \mathbb{G}\left(U,W, n_{d_i}, m_{s_i} \right)$. Let  $G_{d_i}$ be the dense part of $G_{n_i}$ and  let $G_{s_i}$ be the sparse part. Let $n_{d_i}$ and $n_{s_i}$ be the number of nodes in $G_{d_i}$ and  $G_{s_i}$ respectively. Then 
\begin{enumerate}
    \item If there exists $c \in \mathbb{R}^+$ such that $\limsup_{i \to \infty} \frac{n_{s_i}}{n_{d_i}} = c$ \, % for all $i$ we have $0 \leq \frac{n_{s_i}}{n_{d_i}} \leq c$ 
     then $\{G_{n_i}\}_i$ is dense. 
    \item 
     If $\lim_{i \to \infty} \frac{n_{s_i}}{n_{d_i}} = \infty$, then $\{G_{n_i}\}_i$ is sparse. 
\end{enumerate}
\end{restatable}

Lemma \ref{lemma:WUrandomgraphs1} shows that $(U,W)$-mixture graphs are a general construction that allows us to generate sparse or dense graphs. While our focus is on sparse graphs, dense graphs that cannot be fully explained by a single graphon $W$ can be generated by a $(U,W)$ mixture (see Appendix \ref{sec:Wcantfullyexplain}). Thus, $(U,W)$-mixture graphs can model a richer set of dense graphs compared to a single graphon $W$.  We  give examples of $(U,W)$-mixture graphs in Appendix \ref{app:mixtures} and show that by changing mixture properties we get  graph sequences ranging from dense to sparse using the same $(U,W)$ combination. %Furthermore depending on the ratios $\frac{n_{s_i}}{n_{d_i}}$ and $\frac{n_{s_i}}{n_{d_i}^2}$ the empirical graphons (Definition \ref{def:empiricalgraphon}) of $G_{n_i}$ and $L(G_{n_i})$ where $L$ denotes the line graph converge to different graphons as discussed in Lemma \ref{lemma:WUrandomgraphs2}.  %Specific examples of dense and sparse graphs in Appendix \ref{sec:densetosparse}. 

%% file: 5_special_case_mixtures.tex
\section{The top-$k$ degrees of sparse $(U,W)$ mixture graphs}\label{sec:topk}

In a network the highest-degree nodes (the hubs) often represent the most influential or central entities of the network. In social networks they maybe key individuals, in transportation networks they can represent critical infrastructure and in a neurological context hub overload and failure can explain neurological disorders \citep{stam2014modern}. 
%metabolic or biological networks these nodes can represent essential proteins or reactions. 
Thus modeling high-degree nodes is important. 
By observing a mixture graph without knowing $U$, $W$ or the mixture ratio function, we  predict the high-degrees of an unseen graph, by knowing only the number of  nodes in that graph.

For a sparse $(U,W)$ mixture sequence, we show that the high-degree nodes are contributed by $U$ (Lemma \ref{lemma:WUrandomgraphs6})). 
Recall that graphon $U$ is in the line graph space, and by Theorem~\ref{thm:Jansonmain} is formed by disjoint cliques. Hence the inverse line graphs are disjoint stars.  
As $U$ is described by a mass-partition $\bm{p} = (p_1, p_2, \ldots)$, if $p_j > p_k$ the degree of the star corresponding to $p_j$ is larger than that of $p_k$ even after joining the sparse and dense parts (Lemma \ref{lemma:WUrandomgraphs5}). Furthermore, the joining process does not change the order of the highest degrees as stated below.

% Lemma \ref{lemma:WUrandomgraphs5} shows that if $p_j > p_k$ then with high probability $q_{j,i} > q_{k, i}$. That is, for a given a mass-partition $\bm{p}$, with high probability, a star vertex corresponding to a higher $p_j$ will have larger degree than a star vertex corresponding to a lower $p_k$.

% Lemmas \ref{lemma:WUrandomgraphs6} and \ref{lemma:WUrandomgraphs5} show that as $i$ goes to infinity, the top $k$ degrees in $G_{n_i}$ are contributed by the sparse part, provided the mass-partition $\bm{p}$ has $k$ non-zero elements. Furthermore, the joining process does not change the order of the highest degrees. We state this in the proposition below.
% \cheng{Do you have a preference of the word ``order'' vs ``rank''? They both have other meanings, which is really confusing, and we need to choose one. I guess ``rank'' is the least work since it is already there.}
\begin{restatable}{proposition}{proprankpreserving}\label{prop:rankpreserving}(\textbf{Order Preserving Property})
    Let $\{G_{n_i}\}_i$ be a sequence of sparse $(U,W)$-mixture graphs (Definition \ref{def:WURandomMixtureGraphs}) with dense and sparse parts $G_{d_i}$ and $G_{s_i}$ respectively. Let $\bm{p} = (p_1, p_2, \ldots )$ be the mass-partition (Definition \ref{def:masspartition}) associated with $U$ which has at least $k$ partitions. %non-zero elements $p_1 > \cdots > p_k \neq 0$. 
    Let $\tilde{q}_{j,i}$ be the degree of the star in $G_{s_i}$ corresponding to $p_j \neq 0$. Let $q_{j,i}$ denote the degree of the corresponding vertex in $G_{n_i}$.  Let $\deg_{G_{n_i}} v_{(r)}$ denote the $r$th highest degree in $G_{n_i}$.  
    Then 
% \[  \lim_{i \to \infty} P\left( \bigcap_{j = 1}^k \left( q_{j, i} = \deg_{G_{n_i}} v_{(j)} \right) \right) = 1 \, . 
% \]
% Then use it
\ifthenelse{\boolean{twocolumn}}{
  % code for two column
   \begin{align*}
      P & \left( \bigcap_{j = 1}^k \left( q_{j, i} = \deg_{G_{n_i}} v_{(j)} \right) \right)  \geq \left( 1 - \frac{c_1}{m_{s_i}}\right)^{k} \times \\
      & \left( 1 - \exp\left(-c_2 \frac{m_{s_i}^2}{n_{d_i}^2} \right)  - \exp\left( -c_3m_{s_i}  \right) \right)
  \end{align*}
}{
  % code for single column  
   \[
 P\left( \bigcap_{j = 1}^k \left( q_{j, i} = \deg_{G_{n_i}} v_{(j)} \right) \right) \geq \left( 1 - \frac{c_1}{m_{s_i}}\right)^{k}\left( 1 - \exp\left(-c_2 \frac{m_{s_i}^2}{n_{d_i}^2} \right)  - \exp\left( -c_3m_{s_i}  \right) \right) 
 \]
}

% \if@twocolumn
%   % two column code
%   \begin{align*}
%       P\left( \bigcap_{j = 1}^k \left( q_{j, i} = \deg_{G_{n_i}} v_{(j)} \right) \right) & \geq \left( 1 - \frac{c_1}{m_{s_i}}\right)^{k} \times \\
%       & \left( 1 - \exp\left(-c_2 \frac{m_{s_i}^2}{n_{d_i}^2} \right)  - \exp\left( -c_3m_{s_i}  \right) \right)
%   \end{align*}
% \else
%   % single column code
%   \[
%  P\left( \bigcap_{j = 1}^k \left( q_{j, i} = \deg_{G_{n_i}} v_{(j)} \right) \right) \geq \left( 1 - \frac{c_1}{m_{s_i}}\right)^{k}\left( 1 - \exp\left(-c_2 \frac{m_{s_i}^2}{n_{d_i}^2} \right)  - \exp\left( -c_3m_{s_i}  \right) \right)
% \]
% \fi

That is, with high probability the order of the stars in the sparse part are preserved by joining.
\end{restatable}

The Order Preserving Property helps us to estimate the top-$k$ degrees of graphs when mass-partitions have at least $k$ non-zero entries. % Suppose we observe a graph. Treating this as a training graph, we can estimate the top-$k$ degrees of a test graph, given the total number of nodes in the test graph. We state this in the Lemma below.  

\begin{restatable}{lemma}{lemmaDegreePredict}\label{lemma:DegreePredict}
  Suppose $G_{n_i}$ and $G_{n_j}$ are two graphs from a sparse $(U,W)$-mixture graph sequence (Definition \ref{def:WURandomMixtureGraphs}). We treat $G_{n_i}$ as the training graph and $G_{n_j}$ as the test graph. Suppose $G_{n_i}$ and $G_{n_j}$ have $n_i$ and $n_j$ nodes respectively. % and suppose $W$ has a continuous degree function (Definition \ref{def:degreefunction}). 
  Let $\bm{p} = (p_1, p_2, \ldots )$ be the mass-partition (Definition \ref{def:masspartition}) associated with $U$, which has at least $k$ partitions.  Let $\deg_{G_{n_i}} v_{(\ell)}$ denote the $\ell$th largest degree in $G_{n_i}$  where $\ell \leq k$.  Then we estimate the $\ell$th largest degree in $G_{n_j}$ as
\begin{equation}\label{eq:degreepred}
    \deg_{G_{n_j}} \hat{v}_{(\ell)} = \deg_{G_{n_i}} v_{(\ell)} \times \frac{n_j}{n_i} \, , 
\end{equation}
 which satisfies
 % \[
 %     \lim_{i, j \to \infty}  \left( \frac{ \left| \deg_{G_{n_j}} \hat{v}_{(\ell)} -  \mathbb{E}( q_{\ell,j})  \right|}{m_{s_j}} \right) = 0 \, , 
 %    \]
    \begin{equation}\label{eq:degprederror}
        \frac{ \left| \deg_{G_{n_j}} \hat{v}_{(\ell)} -  \mathbb{E}( q_{\ell,j})  \right|}{m_{s_j}} \leq c p_{\ell} \left \vert 1 - \frac{n_jm_{s_i}}{ n_i m_{s_j}}  \right\vert  %\text{proof needs updating for this}
    \end{equation}
    
    with high probability, where ${q}_{\ell,j}$ denotes the degree of the hub vertex  corresponding to $p_{\ell} \neq 0$ in $G_{n_j}$ and $m_{s_i}$ and $m_{s_j}$ denote the number of edges in the sparse parts $G_{s_i}$ and $G_{s_j}$.
\end{restatable}

We use equation \eqref{eq:degreepred} to predict degrees and thus model highest-degree nodes. Equation \eqref{eq:degprederror} gives an error bound. 

%% file: 6_estimating_U.tex
\section{Estimating $U$ for finite number of partitions}\label{sec:estimatingUfinite}

Graphon $U$ describes the relative strength (or vulnerability)  of the hubs in a mixture. In a social media network, hubs act as superspreaders of information and reach; $U$ shows the relative power of the big players in the network. In an energy/telecommunications network, it can show the vulnerability/resilience in the event of an attack. Consider two energy networks $A$ and $B$ with $U_A = (0.5, 0.3, 0.2)$ and $U_B = (0.1, \ldots, 0
.1)$. Then mixture $A$ is far more vulnerable from a resilience viewpoint compared to mixture $B$ if attacked. By attacking one node 
 (the hub corresponding to 0.5), approximately 50\% of the network can be crippled in an energy/telecommunications network scenario. %Assuming back up connections in the network, a significant amount of network traffic (energy usage) needs to be channeled through the other hubs in such an instance. Attacking a hub node in mixture  has less catastrophic consequences.

%Understanding sources of sparsity can help uncover the graph generation mechanism. 
%For mixture graphs, sparsity comes from $U$, which %lives in the line graph space and 
Graphon $U$ can have finite or infinite partitions. By observing mixture graphs, we  estimate the mass-partition $\bm{p}$ (Definition \ref{def:masspartition}), both when $\bm{p}$ has finite and infinite partitions. In this section we consider finite partitions. For both cases, we consider a class of well-behaved $W$ in our mixture model. We do not need $W$ to be well-behaved to estimate the topmost degrees of $G_{n_i}$. However, we need some conditions on $W$ so that  we can estimate $U$.

% Every disjoint clique graphon $U$ corresponds to a mass partition $\bm{p}$ (Definition \ref{def:masspartition}). A mass-partition $\bm{p} = (p_1, p_2, \ldots)$ can have either finite or infinite non-zero elements adding to one. We consider these two cases separately. In this section we consider finite partitions. For both cases, we consider a class of well-behaved $W$ in our mixture model. 
% We do not need $W$ to be well-behaved to estimate the topmost degrees of $G_{n_j}$. However, we need some conditions on $W$ so that  we can estimate $U$. %This is because  we want to exclude $W$ from having certain erratic behavior. 

% \cheng{Why do we need continuity? What breaks when we don't?}
\begin{definition}\label{def:degreefunction}(\textbf{Degree Function})
    As in \cite{delmas2021asymptotic} we define the degree function 
    \[
    D(x) = \int_0^1 W(x,y) \, dy \, , 
    \]
    where $D(x)$ is defined on $x \in [0,1]$. 
\end{definition}

\begin{assumption}\label{assump:1}
    We assume $W$ to have a continuous degree function $D(x)$ (Definition \ref{def:degreefunction}).
\end{assumption}

When $W$ satisfies Assumption \ref{assump:1}, the degrees in the dense part $G_{d_i}$ are closely packed, i.e., given a degree value, there is another degree value quite close in $G_{d_i}$ (see Lemma \ref{lemma:krelated2}). %We make this notion precise in Lemma \ref{lemma:krelated2}, which we use to estimate $k$, the number of partitions. 
Using Assumption \ref{assump:1} we estimate $k$, the number of partitions.

\begin{restatable}{proposition}{PropestimatingKOne}\label{prop:estimatingK1}
     Let $\{G_{n_i}\}_i$ be a sequence of sparse $(U,W)$-mixture graphs (Definition \ref{def:WURandomMixtureGraphs}) with dense and sparse parts $G_{d_i}$ and $G_{s_i}$ having nodes $n_{d_i}$ and $n_{s_i}$ respectively.  Suppose $W$ satisfies Assumption \ref{assump:1} and $W \neq 1$.
     % the degree function $D(x)$ of $W$ (Definition \ref{def:degreefunction}) is continuous.  
     Let $\bm{p} = (p_1, p_2, \ldots )$ be the mass-partition (Definition \ref{def:masspartition}) associated with $U$ which has only $k$ partitions. %with $\sum_{j = 1}^k p_j = 1$.  
     Let $\deg_{G_{n_i}} v_{(\ell)}$ denote the $\ell$th largest degree in $G_{n_i}$. Then there exists $I_0$ such that  for $i > I_0$ we have 
     \begin{equation}\label{eq:khatfinite}
          k = \max_\ell \left( \log  \deg_{G_{n_i}} v_{(\ell)} - \log \deg_{G_{n_i}} v_{(\ell + 1)} \right)\, , 
     \end{equation}
where we exclude small degrees.
% i.e., $\deg_{G_{n_i}} v_{(\ell)} >c$ for some small $c$. 
That is, in the log scale the difference between successive top $k$ degrees is largest at $k$. 
\end{restatable}

Once we have found the partitions $k$, we can estimate the mass-partition $\bm{p} = (p_1, \ldots, p_k, 0 , 0, \ldots )$.
\begin{restatable}{proposition}{proppj}\label{prop:pj}
    Let $\{G_{n_i}\}_i$ be a sequence of sparse $(U,W)$-mixture graphs (Definition \ref{def:WURandomMixtureGraphs}) with dense and sparse parts $G_{d_i}$ and $G_{s_i}$ respectively. Let $\bm{p} = (p_1, p_2, \ldots )$ be the mass-partition (Definition \ref{def:masspartition}) associated with $U$ with only $k$ partitions. %Let $\tilde{q}_{j,i}$ be the degree of the star in $G_{s_i}$ corresponding to $p_j \neq 0$. 
    Let $q_{j,i}$ denote the degree of the hub vertex corresponding to $p_j \neq 0$ in $G_{n_i}$.   Then using the second order approximation of the Taylor expansion for the expectation and the first order approximation for the variance, and for some constants $c_1$ and $c_2$ we have
    \ifthenelse{\boolean{twocolumn}}{
   \[
    % for two columns
 \left\vert \mathbb{E}\left( \frac{q_{j, i}}{\sum_{\ell=1}^k q_{\ell, i} } \right) - p_j  \right\vert \leq \frac{c_1}{m_{s_i}} \, ,  \Var\left( \frac{q_{j, i}}{\sum_{\ell=1}^k q_{\ell, i} } \right) \leq \frac{c_2}{m_{s_i}} \, . 
    \]
}{
  % code for single column  
    \[
 \left\vert \mathbb{E}\left( \frac{q_{j, i}}{\sum_{\ell=1}^k q_{\ell, i} } \right) - p_j  \right\vert \leq \frac{c_1}{m_{s_i}} \quad \text{and} \quad  \Var\left( \frac{q_{j, i}}{\sum_{\ell=1}^k q_{\ell, i} } \right) \leq \frac{c_2}{m_{s_i}} \, . 
    \]
}
 \end{restatable}

Therefore, when $U$ has finite partitions, first we estimate $k$ using equation \eqref{eq:khatfinite}. 
% \begin{equation}\label{eq:khatfinite}
%     \hat{k} = \max_\ell \left( \log  \deg_{G_{n_i}} v_{(\ell)} - \log \deg_{G_{n_i}} v_{(\ell + 1)} \right)\, .
% \end{equation}
Then we estimate the mass-partition $\bm{p} = (p_1, p_2, \ldots, p_k, 0, \ldots )$ by
\begin{equation}\label{eq:pjhat}
    \hat{p}_{j} = \frac{ \deg_{G_{n_i}} v_{(j)} }{ \sum_{\ell = 1}^{\hat{k}} \deg_{G_{n_i}} v_{\ell)} } \, , 
\end{equation}
where $\deg_{G_{n_i}} v_{(\ell)}$ denotes the observed $\ell$th highest degree in ${G_{n_i}}$.  From the Order Preserving Property (Proposition \ref{prop:rankpreserving}) we know that for large $i$, $\deg_{G_{n_i}} v_{(\ell)} = q_{\ell, i}$, and thus the bounds in Proposition \ref{prop:pj} are satisfied.

%% file: 7_estimating_infinite_U.tex
\section{Estimating $U$ for infinite number of partitions}\label{sec:EstimatingUinfinitePartitions}

\begin{figure}[t]
    \centering
    \includegraphics[width=0.8\linewidth]{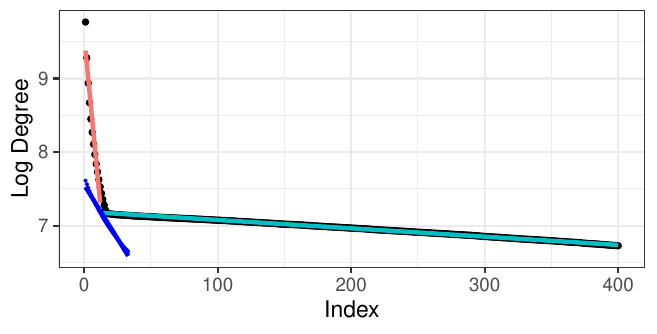}
    \caption{Unique log degree values of a $(U,W)$ mixture graph $G_{n_i}$. The line fitted to points $\{(j, \log( \deg_{G_{n_i}} v_{(j)}))\}_{j = 1}^{k_i}$ is shown in red and the line fitted to degrees generated from $W$ is shown in green. The line fitted to points $\{(j, \log \frac{m_{s_i}}{j}) \}_{j = 1}^{N}$ is shown in blue. }
    \label{fig:logComparison}
\end{figure}

When $U$ has infinite partitions, Proposition \ref{prop:estimatingK1} does not hold. In the finite partition case, there is a widening gap between the hub degrees generated by $U$ and the degrees generated by $W$. But in the infinite partition case, there are always smaller values $p_i$ that fill this gap.
% This is because smaller values $p_\ell$ in the mass partition $\bm{p}$ have degrees $m_{s_i} p_{\ell}$ that are indistinguishable from the degrees of $W$ in graph $G_{n_i}$, but as $i$ increases, $m_{s_i} p_{\ell}$ gets larger than the degrees of $W$ and thus, fills this gap.  
% 
However, we show that the degrees generated by $U$ are different to those generated by $W$. Let $k_i$ denote the number of hub nodes in $G_{n_i}$ generated by $U$ with degrees higher than those generated by $W$. Our goal is to estimate $k_i$. Thus, we focus on the larger degrees of $G_{n_i}$ and consider the unique degree values of $G_{n_i}$ greater than some percentile $C$. % and compute their logarithm in order to make the range of values smaller. 
We consider the ordered, unique log degree values $ \log( \deg_{G_{n_i}} v_{(j)})$, where $\deg_{ G_{n_i}} v_{(j)}$ is the $j$th  highest degree in $G_{n_i}$. %As we are considering unique values, we get rid of repeated values. 
To the points $(j, \log( \deg_{G_{n_i}} v_{(j)}))$, we fit two lines using the OLS estimator. We show that the line fitted to the first $k_i$ points $\{(j, \log( \deg_{G_{n_i}} v_{(j)}))\}_{j = 1}^{k_i}$ has a steeper slope compared to the line fitted to the rest of the points, which are generated by $W$.  In Figure \ref{fig:logComparison} these two lines are shown in red and green respectively. We do this by comparing with the line fitted to points on the curve $y = \log(m_{s_i}/x)$, which is shown in blue in Figure \ref{fig:logComparison}. 

We do this in two parts. First we show that the red line has a steeper slope than the blue line. % i.e., the line fitted to the points $\{(j, \log( \deg_{G_{n_i}} v_{(j)})\}_{j = 1}^{k_i}$ has a steeper slope than the line fitted to points $\{(j, \log \frac{m_{s_i}}{j}) \}_{j = 1}^{k_i}$. 
The key observation is that the series $\sum_{j} \frac{1}{j}$ diverges to infinity, and as such cannot represent a mass-partition, implying that a mass-partition $\{p_i\}_{i=1}^{\infty}$ has to converge to zero more steeply (or faster) than  $\{1/j\}_{j = 1}^{\infty}$. This is shown in Lemma \ref{lemma:sparsepartsteeperline}.
% The key observation is that the series $\sum_{j} \frac{1}{j}$ diverges to infinity, whereas the series $\sum_{j} \frac{1}{j^{1 + \alpha}}$ for $\alpha >0 $ converges to a finite value. Thus, the mass-partition $\bm{p}$ cannot be like $\sum_{j} \frac{c}{j}$, but rather, it has to resemble a series such as $\sum_{j} \frac{c}{j^{1 + \alpha}}$. Using this observation, we show that the red line has a steeper slope than the blue line.  
Next we show that the green line is much less steep than the blue line. As $W$ has a continuous degree function $D(x)$ (Assumption \ref{assump:1}) the sorted, unique degree values are close to each other, and for large $i$, they are mostly consecutive integers.  Using this we show that the green line has a less steep slope than the blue line. This is shown in Lemma \ref{lemma:densepartflatterline}.

% These results are shown in Lemmas \ref{lemma:sparsepartsteeperline} and  \ref{lemma:densepartflatterline}.
\begin{procedure}\label{proc:fit2lines}
Let $G_{n_i}$ be a graph from a sparse $(U,W)$-mixture graph sequence with $U$ having infinite partitions. Let $k_i$ denote the number of hub nodes in $G_{n_i}$ generated by $U$ with degrees higher than those generated by $W$. We consider the points $(j, \log(\deg_{G_{n_i}} v_{(j)})) \in \mathbb{R}^2$ for unique degree values greater than some percentile $C$. Suppose there are $N$ such points. For different cutoff points $r$ we fit two lines -- one for $ \{ (j, \log(\deg_{G_{n_i}} v_{(j)}))\}_{j=1}^r$ and another for $ \{ (j, \log(\deg_{G_{n_i}} v_{(j)}))\}_{j=r+1}^N$ using OLS regression. Suppose the $L_2$ loss functions for the first and second fitted lines are given by $\mathcal{L}_1(r)$ and  $\mathcal{L}_2(N-r)$ respectively. Then we estimate $\hat{k}_i$ as 
\begin{equation}\label{eq:khati}
    \hat{k}_i = \argmin_{r} \left( \mathcal{L}_1(r) + \mathcal{L}_2(N-r) \right) \, . 
\end{equation}
That is, $\hat{k}_i$ is the number of points  on the first fitted line segment when the sum of the loss functions is minimized. 
\end{procedure}

Lemma \ref{lemma:khatiIncreases} shows that the sequence $\{ \hat{k}_i\}_i$ tends to infinity with high probability.  This allows us to show that given $\hat{k}_i$ the ratio 
\begin{equation}\label{eq:hatpjforinfinite}
    \hat{p}_{j, i} = \frac{\deg_{G_{n_i}} v_{(j)} }{ \sum_{\ell = 1}^{\hat{k}_i} \deg_{G_{n_i}} v_{(j)} }
\end{equation}

converges to $p_j$ in expectation. The proof is similar to that of Proposition \ref{prop:pj}. %In Proposition \ref{prop:pj} we had finite $k$, which is changed to $\hat{k}_i$. 

\begin{restatable}{proposition}{proppjInfiniteK}\label{prop:pjInfiniteK}
    Let $\{G_{n_i}\}_i$ be a sequence of sparse $(U,W)$-mixture graphs (Definition \ref{def:WURandomMixtureGraphs}) with dense and sparse parts $G_{d_i}$ and $G_{s_i}$ respectively. Suppose $U$ has infinite partitions and let $\bm{p} = (p_1, p_2, \ldots )$ be the associated mass-partition (Definition \ref{def:masspartition}). Let $q_{j,i}$ denote the degree of the hub vertex corresponding to $p_j$ in $G_{n_i}$ and $\hat{k}_i$ be estimated using Procedure \ref{proc:fit2lines}.  Then using the second order approximation of the Taylor expansion for the expectation and the first order approximation for the variance, and for some constants $c_1$ and $c_2$ we have
    \ifthenelse{\boolean{twocolumn}}{
      % code for two column
      \[
   \left\vert \mathbb{E}\left( \frac{q_{j, i}}{\sum_{\ell=1}^{\hat{k}_i} q_{\ell, i} }   \right) - p_j \right\vert \leq \frac{c_1}{m_{s_i}}  \, ,  \Var\left( \frac{q_{j, i}}{\sum_{\ell=1}^{\hat{k}_i} q_{\ell, i} }\right) \leq \frac{c_2}{m_{s_i}}
    \]
    }{
      % code for single column  
       \[
   \left\vert \mathbb{E}\left( \frac{q_{j, i}}{\sum_{\ell=1}^{\hat{k}_i} q_{\ell, i} }   \right) - p_j \right\vert \leq \frac{c_1}{m_{s_i}}  \quad \text{and} \quad \Var\left( \frac{q_{j, i}}{\sum_{\ell=1}^{\hat{k}_i} q_{\ell, i} }\right) \leq \frac{c_2}{m_{s_i}}
    \]
    }
   
 \end{restatable}

We use equation \eqref{eq:khati} to compute $\hat{k}_i$ and equation \eqref{eq:hatpjforinfinite} to compute $\hat{p}_{j}$ for $j \in \{1, \ldots, \hat{k}_i\}$, which estimates the first $\hat{k}_i$ elements of the mass-partition $\bm{p}$ describing $U$. Additionally, if we know the rate at which the elements in the mass-partition $\{p_i\}_i$ go to zero, then we can estimate $\sum_{j = {\hat{k}_i}}^{\infty} p_j$, the sum of the unestimated elements in the partition. This is shown in Lemma \ref{lemma:errorbounds}. 

%% file: 8_generalization_and_experiments.tex
\section{Experiments}\label{sec:Exp}

\begin{table*}[t]
    \centering
        \caption{Proposed method is an order of magnitude better than baseline on synthetic data. MAPE (smaller is better) reported with standard deviation in parenthesis.}
    \begin{tabular}{cccccc} % cccccc p{2cm}p{2cm}p{2cm}p{2cm}p{2cm}p{2cm}
    \toprule
    Task & Description    &  Experiment 1 & Experiment 2 & Experiment 3 & Experiment 4  \\
    \midrule
    \multirow{2}{*}{Top-$k$ deg.  }  & Proposed & 0.386 (0.337) &   0.225 (0.178) &  0.394 (0.288) &  0.341 (0.391) \\ 
         & Baseline &   9.078 (0.436) &  9.098 (0.238) & 9.185 (0.336) & 9.074 (0.488) \\
    \hline
    \multirow{2}{*}{Finite $U$} & Proposed &  1.800 (1.571)  & 0.602  (0.428) &   0.845   (0.644)  &   1.621  (1.448) \\
    & Baseline & 98.709 (0.030)  &   98.702 (0.012) &    97.952 (0.024)  &   97.766 (0.025) \\
     \hline                   
     \multirow{2}{*}{Infinite $U$} 
                                            & Proposed & 10.917 (0.098) &  1.559 (0.069) &  6.351 (0.051) &     0.339  (0.171) \\
                                            & Baseline & 97.238 (0.002) &  95.233 (0.003 ) & 95.070   (0.003)  &  98.245  (0.003) \\
    % \hline \hline                                                    
    % \multirow{2}{*}{Infinite $U$} 
    %                                          & $\hat{k}_i$ & 30 & 23 & 30 & 4 \\
    %                                          & $\sum_{j = 1}^{\hat{k}_i} p_j$ & 0.902 & 0.985 & 0.941 & 0.998 \\
    \bottomrule                                         
    \end{tabular}
    \label{tab:results}
\end{table*}

\begin{table*}[t]
    \centering
     \caption{Comparison of proposed method with \cite{bollobas2001degree}, \cite{caron2017sparse} and 3 versions of Kronecker graphs \citep{leskovec2010kronecker} on 6 datasets. Average MAPE reported with standard deviation in parenthesis. Best results in bold. Sig. denotes statistical significance with $\alpha = 0.1$.}
     \resizebox{\textwidth}{!}{
    \begin{tabular}{cccccccc}
    \toprule
  %   \multirow{2}{*}{Dataset} & \multirow{2}{*}{Bollobas} & \multirow{2}{c}{Caron-Fox} &  \multirow{2}{*}{Kronfit Default} &  \multirow{2}{c}{Kronfit Deterministic-3} &
  % &   \multirow{2}{c}{Kronfit Stochastic-3}  & \multirow{2}{c}{UW-Mixture} & Significance \\
  Dataset & Bollobas & Caron-Fox &  Kronfit1 &  Kronfit2  
  &   Kronfit3  & UW-Mixture & Sig. \\
    \midrule    
FB  &  0.130 (0.041) &   0.175 (0.103) &  0.289 (0.068)  & 111.0 (14.9) & 0.518 (0.047) & \textbf{0.101} (0.044)  & \cmark\\
HEP-PH & 0.177 (0.105) & 0.435 (0.066) & 0.351 (0.165) & 20.20 (4.01) & 0.294 (0.115) & \textbf{0.061} (0.046)  & \cmark\\
MOOC &   0.015 (0.012) & 0.854 (0.014) & 0.772 (0.068) & 0.797 (0.022) & 0.835 (0.031) & \textbf{0.005} (0.005) & \cmark\\
SMS  &   0.051 (0.033) & 0.426 (0.051) & 0.702 (0.032) &  --           & 0.747 (0.024) & \textbf{0.029} (0.020) & \cmark\\
UCI  &   0.201 (0.106) & 0.757 (0.044) & 0.404 (0.098) &  2.19 (0.409) & 0.908 (0.011) & \textbf{0.185} (0.090) & \xmark \\
Yahoo &  0.030 (0.037)  & 0.446 (0.046) & 0.356 (0.080) &  --            & 0.834 (0.014) & \textbf{0.021} (0.011) & \cmark\\
\bottomrule      
    \end{tabular}}
    \label{tab:realworld}
\end{table*}

We conduct experiments with synthetic and real data on a standard laptop. %With synthetic data, we conduct a range of experiments, changing the number of non-zero elements in the mass-partition in $U$. 
% As a baseline comparison method for degree prediction, we use scale-free graph properties discussed in \cite{bollobas2001degree}. Another simple comparison method detailed in Appendix \ref{app:experiments} is used for estimating $U$. 
As the evaluation metric we use the Mean Absolute Percentage Error (MAPE) given by $100\times \frac{1}{N}\sum_{i=1}^N \frac{| \hat{y}_i - y_i|}{y_i}$.  For degree prediction, $y_i$ denotes the top-$k$ degree values and for estimating $U$, $y_i$ denotes the  mass-partition elements $p_i$.
More details of the experiments are given in Appendix \ref{app:experiments}.

\subsection{Illustration with synthetic data}

With synthetic data experiments we focus on 3 tasks: (1)  predict the top-$k$ degrees, (2)  estimate $U$ when %the mass-partition has finite non-zero elements 
$\bm{p}$ has finite partitions and (3)  estimate $U$ when %the mass-partition has infinite non-zero elements. 
$\bm{p}$ has infinite partitions.
To predict the top-$k$ degrees we use training graphs with $n_i$ nodes and test graphs with $n_j$ nodes. Using equation \eqref{eq:degreepred} we predict the top-$k$ degrees for graphs in  the test set. To estimate $U$ we only use training graphs.  When $U$ has  finite partitions we use equation \eqref{eq:pjhat} and for infinite partitions we use equation \eqref{eq:hatpjforinfinite}.  For each task we conduct 4 experiments. For degree prediction we consider $W_1 = \exp(-(x+y))$, $W_2 = 0.1$  and $U_1$ with mass-partition $\bm{p}_1 \propto \{1/j^{1.2}\}_{j=2}^{50}$ and $U_2$ with mass-partition $\bm{p}_2 \propto \{1/1.2^{j}\}_{j=2}^{50}$.  The mass-partitions are normalized so that they add up to 1. 
A baseline comparison is conducted using scale-free graph properties \citep{bollobas2001degree}.  
% \cheng{We should be able to squeeze these sentences in. ``We consider $W_1 = \exp(-(x+y))$, $W_2 = 0.1$  and $U_1$ with mass-partitions $\bm{p}_1 \propto \{1/j^{1.2}\}_{j=2}^{50}$ and $U_2$ with mass-partition $\bm{p}_2 \propto \{1/1.2^{j}\}_{j=2}^{50}$.  The mass-partitions are normalized so that they add up to 1. "}
Table \ref{tab:results} gives the results of the experiments.

\subsection{Real-world networks}
% \cheng{We seem to have three estimation methods: degree, finite U, infinite U. Specify which method we use, and specify what is MAPE calculated on.}
% For observed real world graphs, we do not have a ground truth model for $U$, hence we are only able to compute MAPE for predicted degrees.
% Using Hep-PH Physics citations dataset \citep{leskovec2005graphs, gehrke2003overview} and Facebook (FB) links dataset \citep{viswanath2009activity} we predict degrees. In Hep-PH dataset, each paper is a node and if paper $i$ cites paper $j$, then there is an edge between nodes $i$ and $j$. We focus on monthly graphs, where the graph is updated with new papers and  new citations. %, i.e., it is a growing graph. 
% We consider graphs for months  80, 85, 90, 95 and 100. Using each graph at month $\ell$, we predict the top 10 degrees of the graph at months $\ell+6$ up to $\ell+10$. Similarly, for the FB dataset we consider graphs on days  520, 540, 560, 580 and 600 with prediction steps 60 to 100, stepping by 10. Table \ref{tab:realworldPhysics} gives the results averaged over the prediction step and the graph index. 

For observed real world graphs, we do not have a ground truth model for $U$, hence we only compute MAPE for predicted degrees.
We use Facebook (FB) links \citep{viswanath2009activity}, Hep-PH Physics citations  \citep{leskovec2005graphs, gehrke2003overview}, MOOC interactions \citep{kumar2019predicting}, SMS \citep{wu2010evidence}, UCI Messages \citep{panzarasa2009patterns} and Yahoo messages \citep{yahoo} datasets. Each dataset is given as an edgelist with timestamped edges. We use daily graphs for  Facebook links, MOOC interactions,  SMS,  and UCI messages datasets. For Yahoo messages we use 2-hourly graphs and for Hep-PH citations monthly graphs. Each dataset consists of a sequence of growing graphs $\{G_{n_i}\}_i$ and we consider $G_{n_i}$ for $i \in \{20, \ldots , 24 \}$ as training graphs. We select the test graphs $G_{n_j}$ such that $j = \min_{k} \{ |G_{n_k}|: |G_{n_k}| - |G_{n_i}| \geq 500  \}$, i.e., $G_{n_j}$ is the first graph that has more than 500 nodes compared to $G_{n_i}$. We predict the top-10 degrees of $G_{n_j}$ using \cite{bollobas2001degree}, \cite{caron2017sparse} and 3 variations of Kronecker graphs \citep{leskovec2010kronecker}. Furthermore, we test for statistical significance between the top 2 methods. Table \ref{tab:realworld} gives the results with blanks denoting timed out instances. The $(U,W)$-mixture estimate is significantly better for 5 out of 6 datasets. 

%% file: 9_conclusions.tex
\section{Conclusions\label{sec:conclusions}}

We present graphon mixtures, an approach that explicitly models sparse and dense components, and generates graph mixtures using two graphons. The graphon $W$ is used to generate dense graphs and the disjoint clique graphon $U$ is used to generate sparse graphs via the inverse line graph operation. Then the sparse and the dense sequences are joined resulting in a mixture sequence. 

Our focus has been on $U$, the new piece in the puzzle. We have modeled the highest degrees of sparse $(U,W)$ mixtures and estimated $U$.  We can estimate $U$ with high accuracy when $U$ has finite disjoint cliques. When $U$ has infinitely many disjoint cliques, the accuracy depends on the rate at which the clique proportions go to zero. We have focused on sparse $(U,W)$ mixtures. % Dense $(U,W)$ mixtures supplement graphs generated by a single graphon $W$. 
While we briefly touch on dense $(U,W)$ mixtures in Appendix \ref{app:mixtures}, this topic can be explored further.  Exploring sparsity resulting from other types of graphs such as paths and rings is an avenue for future work. 
%Many methods have been developed to estimate the classical graphon. Once the effect of $U$ is removed from the mixture, $W$ can be estimated from the remaining graph. This is an avenue for future work. 

% The $(U,W)$-mixture graphs can range from dense to sparse depending on the mixture properties. Furthermore, 
% Without knowing $U$ or $W$, by observing a graph in a sparse $(U,W)$-mixture sequence we can predict the highest degrees of unseen mixture graphs. We can estimate $U$ with high accuracy when $U$ has finite disjoint cliques. When $U$ has infinitely many disjoint cliques, the accuracy depends on the rate at which the clique proportions go to zero. 

%% file: 10_A_Appendix_Background.tex
%\section*{Supplement to Graphon Mixtures}
\section{Background: Limits of sparse graphs and line graphs}

First we define line graphs. 
\begin{definition}\label{def:linegraph}(\textbf{Line graph}) Let $G$ denote a graph with at least one edge. Then its \textbf{line graph} denoted by $L(G)$ is the graph whose vertices are the edges of $G$, with two vertices being adjacent if the corresponding edges are adjacent in $G$ \citep{beineke2021line}.
\end{definition}

The \textbf{empirical graphon} maps a graph to a function defined on a unit square. 
\begin{definition}\label{def:empiricalgraphon}(\textbf{empirical graphon})
    Given a graph $G$ with $n$ vertices labeled $\{1, \ldots, n\}$, we define its \textbf{empirical graphon} $W_G: [0, 1]^2 \rightarrow [0, 1]$ as follows: We split the interval $[0, 1]$ into $n$ equal intervals $\{J_1, J_2, \ldots, J_n \}$ (first one closed, all others half open) and for $x \in J_i, y \in J_j$ define
    \[ W_G(x,y) = \begin{cases}
        1 & \, \text{if} \quad  ij \in E(G) \, \\
        0 & \, \text{otherwise} \, ,
    \end{cases}
    \]
where $E(G)$ denotes the edges of $G$. The empirical graphon replaces the the adjacency matrix with a unit square and the $(i,j)$th entry of the adjacency matrix is replaced with a square of size $(1/n) \times (1/n)$.
\end{definition}

Graph convergence is defined using the \textbf{cut-norm} and the \textbf{cut-metric}. 

\begin{definition}\label{def:cut1}The \textbf{cut norm} of graphon $W$   \citep{frieze1999quick, borgs2008convergent} is defined as
\begin{equation}
     \lVert W \rVert_\square = \sup_{A, B} \left\vert\int_{A\times B} W(x,y) \, dx dy \right\vert  \, , 
\end{equation}
where the supremum is taken over all measurable sets $A$ and $B$ of $[0 ,1]$.  \end{definition}

\begin{definition}\label{def:cut2}(\textbf{cut metric}) Given two graphons $W_1$ and $W_2$ the \textbf{cut metric}  \citep{borgs2008convergent} is defined as 
\[ \delta_{\square}(W_1, W_2) = \inf_\varphi \left\lVert W_1 - W_2^\varphi \right\rVert_\square \, , 
\]
where the infimum is taken over all measure preserving bijections $\varphi:[0,1] \rightarrow [0,1]$.  
\end{definition}

% The role of $\varphi$ is similar to that of graph isomorphisms. That is if $W_1$ and $W_2$ are two empirical graphons, then 

The cut metric is a pseudo-metric because $\delta_{\square}(W_1, W_2)  = 0 $ does not imply $W_1 = W_2$, i.e., $ \delta_{\square}(W_1, W_2)  \geq 0 $ for $W_1 \neq W_2$.

% Let $\mathcal{W}$ denote the space of graphons, i.e., $\mathcal{W} = \{ W \in \mathcal{W} \}$. Then, the cut metric is a pseudo-metric in $\mathcal{W}$ because $\delta_{\square}(W_1, W_2)  = 0 $ does not imply $W_1 = W_2$, i.e., $ \delta_{\square}(W_1, W_2)  \geq 0 $ for $W_1 \neq W_2$. However the cut metric $\delta_{\square} $  is a metric on the quotient space $\tilde{\mathcal{W}} = \mathcal{W}/ \sim  $ where $f \sim g$  if $f(x, y) = g(\sigma x,  \sigma y)$ for some measure preserving $\sigma$.

Next we give some definitions from \citet{Janson2016321}.
\begin{definition}\label{def:masspartitionJanson}(\textbf{Janson's mass-partition})
    We define a mass-partition to be a sequence $\bm{p} = \{p_i\}_{i = 1}^{\infty}$ of non-negative real numbers such that
    \[ p_1 \geq p_2 \geq \cdots \geq 0 \quad \text{and} \quad \sum_{i = 1}^\infty p_i \leq 1 \, . 
    \]
    Let $\mathcal{M}$ be the set of all mass-partitions. 
\end{definition}
For every mass-partition we can define a graphon as follows:
\begin{definition}\label{def:masspartitiongraphon}
Given a mass-partition $\bm{p}$ (Definition \ref{def:masspartitionJanson}) we define a graphon $W_{\bm{p}}^\mathcal{M}$ by taking disjoint subsets $\{A_i\}_{i = 1}^\infty$ of a probability space $(S, \mu)$ such that $\mu(A_i) = p_i$ and defining $W_{\bm{p}}^\mathcal{M} = \sum_{i = 1}^{\infty} \bm{1}_{A_i \times A_i}$. 
\end{definition}

We can define a mass-partition for any graph. To do that we use component sizes. 
\begin{definition}\label{def:compsize}
    For any graph $G$, we denote the component sizes by $C_1(G) \geq C_2(G) \geq \cdots$, ordered such that when $k$ is larger than the number of components $C_k(G) = 0$.  
\end{definition}
Each graph $G$ defines a mass-partition $\{C_i(G)/|G|\}_{i=1}^\infty$ where $|G|$ denotes the number of nodes in $|G|$.

The next theorem links mass partitions, graphons and disjoint clique graphs. 
\begin{theorem}\label{thm:Janson1} {\bf(\cite{Janson2016321} Thm 7.5) } If $\{G_n\}_n$ is a sequence of disjoint clique graphs with $|G_n| \to \infty$ and $W$ is a graphon, then $G_n \to W$ if an only if $W = W_{\bm{p}}^\mathcal{M}$  for some $\bm{p} \in \mathcal{M}$ and $\{ \left(C_i(G_n)/|G_n| \right)_i \} \to \bm{p}$.
\end{theorem}

Theorem \ref{thm:Janson1} tells us that a mass-partition $\bm{p}$ can be used to describe a disjoint clique graphon, i.e., there is a one-to-one correspondence between disjoint clique graphons and mass-partitions.  %We find Theorem 8.3 in \cite{Janson2016321} relevant and useful for describing sparse graphs that have non-zero line graph limits. 

Theorem \ref{thm:Jansonmain} and Theorem \ref{thm:Janson1}  \citep{Janson2016321} tells us that graphons of line graphs are disjoint clique graphons and as such they can be described by mass-partitions.

\subsection{Square-degree property and max-degree condition}\label{sec:appendix1}

\begin{definition}[\bf Square-degree property]\label{def:square} Let $\{G_n\}_n$ denote a sequence of graphs.  Then $\{G_n\}_n$ exhibits the square-degree property if there exists some $c_1  > 0 $ and $N_0 \in \mathbb{N}$ such that for all $n \geq N_0$ 
\[ \sum \deg v_{i, n}^2 \geq c_1 \left( \sum \deg v_{i, n} \right)^2  \, . 
\]  
% Alternatively, 
% \[ \liminf_{n \to \infty} \frac{\sum \deg v_{i, n}^2 }{ \left(\sum \deg v_{i, n} \right)^2} = c_1 > 0 \, . 
% \]
The set of graph sequences satisfying the square-degree property is denoted by $S_q$. %, i.e. if $\{G_n\}_n$ satisfies $Sq$ then $\{G_n\}_n \in S_q$.
\end{definition}

\lemmasquareandmaxequivlence*
\begin{proof}
    See Lemmas  \ref{lemma:maxdegimpliesSquaredeg} and \ref{lemma:SquaredegimpliesMaxDeg}.
\end{proof}

\begin{lemma}
\label{lemma:maxdegimpliesSquaredeg}
    Let $\{G_n\}_n$ be a graph sequence with $H_m = L(G_n)$. Suppose $\{H_m\}_m$ converges to $U$.  Then
    \[  \{G_n\}_n \in S_x \implies \{G_n\}_n \in S_q \, , 
    \]
    where $S_x$ denotes the set of graph sequences satisfying the max-degree condition (Definition \ref{def:maxdegreetoedges}) and $S_q$ denotes the set of graph sequences satisfying the square-degree property (Definition \ref{def:square}).
\end{lemma}
\begin{proof}
We compute the ratio of the sum of degree squares to the number of edges squared as this quantity determines the square-degree property (Definition \ref{def:square}).  For $ \{G_n\}_n \in S_x$ 
    \begin{equation}
        \frac{ \sum \deg v_{i, n}^2}{m^2}  \geq \frac{d_{\max, n}^2}{m^2} \geq c^2 > 0 \, , 
    \end{equation}
    we obtain the result.
\end{proof}

\begin{lemma}
\label{lemma:SquaredegimpliesMaxDeg}
    Let $\{G_n\}_n$ be a graph sequence with $H_m = L(G_n)$. Suppose $\{H_m\}_m$ converges to $U$.  Then
    \[  \{G_n\}_n \in S_q \implies \{G_n\}_n \in S_x \, , 
    \]
       where $S_x$ denotes the set of graph sequences satisfying the max-degree condition (Definition \ref{def:maxdegreetoedges}) and $S_q$ denotes the set of graph sequences satisfying the square-degree property (Definition \ref{def:square}).
\end{lemma}
\begin{proof} For $ \{G_n\}_n \in S_q$ as $H_m$ converges to $U$
\begin{align}
    \liminf_{n \to \infty}  \frac{ \sum \deg v_{i, n}^2}{m^2} =  \lim_{n \to \infty}  \frac{ \sum \deg v_{i, n}^2}{m^2} & = c > 0
\end{align}
Let us denote the ordered degrees by $d_{(1)}, d_{(2)}$, and so on where  $d_{max, n} = d_{(1)}$. Then for any $n$
\[ \frac{d_{(1)}}{m} \geq \frac{d_{(2)}}{m} \geq \cdots \geq  \frac{d_{(n)}}{m} \,  ,  
\]
and 
\[ \sum_i \deg v_{i, n}^2 = d_{(1)}^2 + d_{(2)}^2 + \cdots +  d_{(n)}^2 \, .
\]
Noting the series $\frac{ \sum \deg v_{i, n}^2}{m^2}$ is absolutely convergent and the individual terms are positive  $\left(  \frac{d_{(i)}^2}{m^2} >0 \right) $ we have 
\[  \lim_{n \to \infty}  \frac{ \sum \deg v_{i, n}^2}{m^2}  =  \lim_{n \to \infty}  \frac{d_{(1)}^2 }{m^2}  + \lim_{n \to \infty}  \frac{d_{(2)}^2 }{m^2}   + \cdots + \lim_{n \to \infty}  \frac{d_{(n)}^2 }{m^2}   = c \, . 
\]
If each individual limit on the right hand side of the above equation were zero, then we would get a contradiction. As such, at least some limits need to be positive. As $d_{(1)}/m$ is the largest, there exists a constant $c_2 > 0$ such that
\[ \frac{d_{(1)}^2 }{m^2} = \frac{d_{\max, n}^2 }{m^2} = c_2 \, , 
\]
which completes the proof.
\end{proof}

Lemmas \ref{lemma:maxdegimpliesSquaredeg} and \ref{lemma:SquaredegimpliesMaxDeg} show that if we consider sparse graph sequences that have converging line graphs, then the square-degree property is equivalent to the max-degree property. 

%% file: 10_B_Appendix_Mixtures.tex
% ---------------------------------------------------------------------------
\section{Mixtures of sparse and dense graphs}\label{app:mixtures}
% ---------------------------------------------------------------------------
% -----------------------------------------------
\subsection{$(U,W)$ mixture examples}
We give two examples of $(U,W)$-mixture graphs in this section. 

\subsubsection{Example 1}
We consider $U$ and $W$ shown in Figure \ref{fig:UandWEx1} where $U$ is given by the mass-partition $\bm{p}= (0.5, 0.3, 0.2)$ and $W$ is given by a stochastic block model graphon.  Graphon $W$ has 4 equally sized communities where the edge probability of two nodes $i$ and $j$ within the same community is given by $p_{ij} = 0.6$ and the edge probability of two nodes $i$ and $k$ in two different communities is given by $p_{ik} =0.025$. %Figure \ref{fig:UandWEx1} shows the graphons $U$ and $W$ for this example.
 
\begin{figure}[t]
    \centering
    \includegraphics[width=0.6\linewidth]{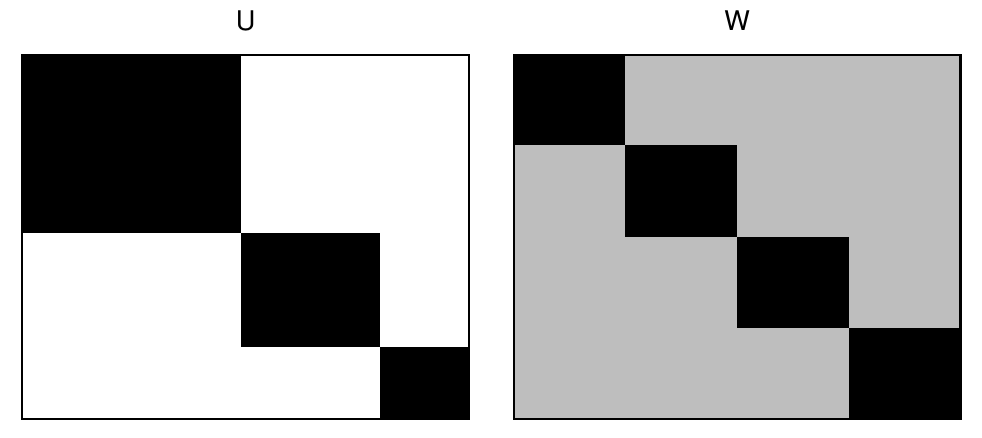}
    \caption{Graphons $U$ and $W$ used in Example 1.}
    \label{fig:UandWEx1}
\end{figure}

Figures \ref{fig:TwoGraphMixturesForEx11} and \ref{fig:TwoGraphMixturesForEx12} show examples of mixture graphs from this $U$ and $W$. In Figure \ref{fig:TwoGraphMixturesForEx11} the dense part  $G_{d_i}$ has 100 nodes and the sparse part $G_{s_i}$ has 600 nodes. As $G_{s_i}$ has a larger number of nodes, the resulting mixture graph $G_{n_i}$ has large hubs, resembling a graph from a sparse sequence. In contrast, in Figure \ref{fig:TwoGraphMixturesForEx12} the dense part $G_{d_i}$ has 100 nodes and the sparse part $G_{s_i}$ has 50 nodes making the mixture graph $G_{n_i}$ resemble a graph from a dense sequence. 

From this $U$ and $W$ we construct a sparse $(U,W)$-mixture sequence with $n_{s_i} = \lceil n_{d_i} \times \sqrt{5i} \rceil$ for $i \in \{1, \ldots 20\}$. For each mixture graph $G_{n_i}$, we compute the global clustering coefficient, the mean distance of the graph and its spectral radius. Table \ref{tab:graphPropEx1} gives the results. We see that the clustering coefficient does not go to zero in this example. 

\begin{table}[!ht]
    \centering
    \caption{The global clustering coefficient, mean distance and spectral radius of the sparse $(U,W)$ mixture graphs sequence discussed in Example 1}
    \begin{tabular}{ccccc}
    \toprule
    $n_{d_i}$ & $n_{s_i}$ &  Clustering Coefficient & Mean Distance & Spectral Radius \\
    \midrule
   100  & 224        &      0.32    &      4.06     &      17.08 \\
   200  & 633        &      0.31    &      3.83     &      33.60 \\ 
   300  & 1162       &       0.33   &       3.67    &       50.50 \\
   400  & 1789       &       0.33   &       3.75    &       67.46 \\
   500  & 2500       &       0.33   &       3.66    &       84.18 \\
   600  & 3287       &       0.33   &       3.57    &      101.17 \\
   700  & 4142       &       0.33   &       3.54    &      117.99 \\
   800  & 5060       &       0.33   &       3.42    &      134.99 \\
   900  & 6038       &       0.33   &       3.48    &      151.80 \\
 1000  & 7072        &      0.33    &      3.48     &     168.85 \\
 1100  & 8158        &      0.33    &      3.47     &     186.21 \\
 1200  & 9296        &      0.33    &      3.41     &     203.01 \\
 1300 & 10481        &      0.33    &      3.46     &     219.61 \\
 1400 & 11714        &      0.33    &      3.37     &     236.22 \\
 1500 & 12991        &      0.33    &      3.32     &     252.36 \\
 1600 & 14311        &      0.33    &      3.28     &     270.54 \\
 1700 & 15674        &      0.33    &     3.27      &    287.00 \\
 1800 & 17077        &      0.33    &      3.25     &     303.76 \\
 1900 & 18519        &      0.33    &      3.29     &     320.70 \\
 2000 & 20000        &      0.33    &      3.26     &     337.55 \\
\bottomrule
    \end{tabular}
    
    \label{tab:graphPropEx1}
\end{table}
%    nd    ns clust_coef_global mean_distance spectral_radius
 %   100   224              0.32          4.06           17.08
 %   200   633              0.31          3.83           33.60
 %   300  1162              0.33          3.67           50.50
 %   400  1789              0.33          3.75           67.46
 %   500  2500              0.33          3.66           84.18
 %   600  3287              0.33          3.57          101.17
 %   700  4142              0.33          3.54          117.99
 %   800  5060              0.33          3.42          134.99
 %   900  6038              0.33          3.48          151.80
 % 1000  7072              0.33          3.48          168.85
 % 1100  8158              0.33          3.47          186.21
 % 1200  9296              0.33          3.41          203.01
 % 1300 10481              0.33          3.46          219.61
 % 1400 11714              0.33          3.37          236.22
 % 1500 12991              0.33          3.32          252.36
 % 1600 14311              0.33          3.28          270.54
 % 1700 15674              0.33          3.27          287.00
 % 1800 17077              0.33          3.25          303.76
 % 1900 18519              0.33          3.29          320.70
 % 2000 20000              0.33          3.26          337.55

\begin{figure}[!ht]
    \centering
    \includegraphics[width=0.95\linewidth]{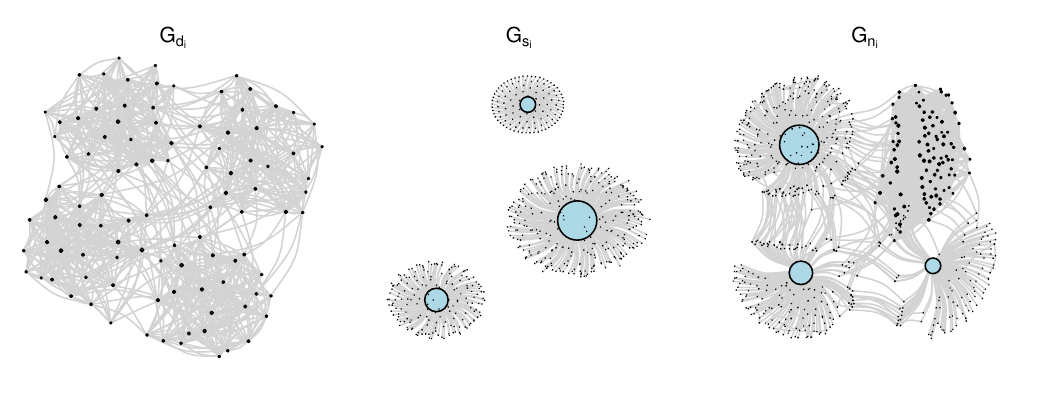}
     \caption{A $(U,W)$-mixture graph with $n_{d_i} = 100$ and $n_{s_i} = 600$ where $U$ and $W$ are shown in Figure \ref{fig:UandWEx1}. The mixture graph $G_{n_i}$ resembles a graph from a sparse sequence. }
      \label{fig:TwoGraphMixturesForEx11}
    \includegraphics[width=0.95\linewidth]{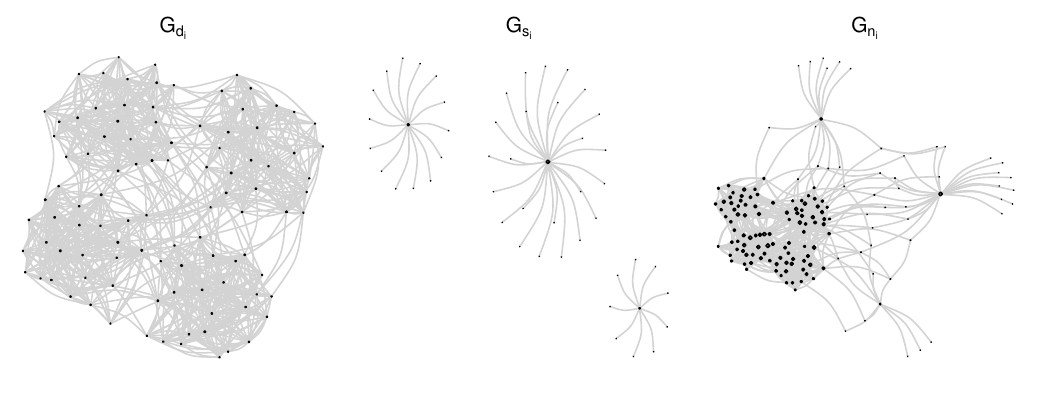}
    \caption{A $(U,W)$-mixture graph with $n_{d_i} = 100$ and $n_{s_i} = 50$ where $U$ and $W$ are shown in Figure \ref{fig:UandWEx1}. The mixture graph $G_{n_i}$ resembles a graph from a dense sequence.}
    \label{fig:TwoGraphMixturesForEx12}
\end{figure}

\subsubsection{Example 2}
Next we consider $U$  having the mass-partition $\bm{p} \propto \left(\frac{1}{2}, \frac{1}{3}, \ldots, \frac{1}{10} \right)$ such that $\sum_i p_i = 1$. Let $W$ be given by 
\[
W = 0.9 \exp\left( \frac{(-y^2 - (x-1)^2)}{0.05^2}  \right) + 0.9 \exp\left( \frac{(-(y-1)^2 - x^2)}{0.05^2}  \right) + 0.9 \exp \left( - \frac{\left( \sin \left(3\pi/4\right)x + \cos \left(3\pi/4\right)x \right)}{0.05^2} \right) \, , 
\]
as in \cite{xia2023implicit}. Graphon $W$ produces dense ring-type graphs. The two graphons are shown in Figure \ref{fig:UandWEx2}.  Figure \ref{fig:TwoGraphMixturesForEx21} shows the dense part $G_{d_i}$ sampled from $W$ with $n_{d_i} = 100$, the sparse part $G_{s_i}$ with $n_{s_i} = 100$ and the mixture graph $G_{n_i}$.  Figure \ref{fig:TwoGraphMixturesForEx22} shows the dense part $G_{d_i}$ with $n_{d_i} = 100$, the sparse part $G_{s_i}$ with $n_{s_i} = 600$ and the mixture graph $G_{n_i}$. We see that the hub structure is stronger in Figure \ref{fig:TwoGraphMixturesForEx22} as $n_{s_i} > n_{d_i}$.   

Again we construct a sparse mixture sequence from this $U$ and $W$. Similar to Example 1, we let $n_{s_i} = \lceil n_{d_i} \times \sqrt{5i} \rceil$ for $i \in \{1, \ldots 20\}$. For each mixture graph $G_{n_i}$, we compute the global clustering coefficient, the mean distance of the graph and its spectral radius. Table \ref{tab:graphPropEx2} gives the results. Similar to Example 1, we see a robust clustering coefficient as $n_{s_i}$ increases. 

\begin{table}[!ht]
    \centering
    \caption{The global clustering coefficient, mean distance and spectral radius of the sparse $(U,W)$ mixture graphs sequence discussed in Example 2}
    \begin{tabular}{ccccc}
    \toprule
    $n_{d_i}$ & $n_{s_i}$ &  Clustering Coefficient & Mean Distance & Spectral Radius \\
    \midrule
   100  &  224      &        0.27    &      6.06    &       10.62 \\
   200  & 633       &       0.29     &     4.51     &      21.84 \\
   300  & 1162      &        0.31    &      4.28    &       33.22 \\
   400  & 1789      &        0.31    &      4.23    &       44.12 \\
   500  & 2500      &        0.32    &      4.06    &       55.80 \\
   600  & 3287      &        0.32    &      4.14    &       66.76 \\ 
   700  & 4142      &        0.31    &      4.08    &       77.61 \\
   800  & 5060      &        0.32    &      4.09    &       89.42 \\
   900  & 6038      &        0.32    &      3.98    &      100.76 \\
 1000  & 7072       &       0.32     &     4.03     &     111.96 \\
 1100  & 8158       &       0.32     &     4.04     &     123.18 \\
 1200  & 9296       &       0.32     &     4.03     &     133.78 \\
 1300 & 10481       &       0.32     &     4.02     &     145.67 \\
 1400 & 11714       &       0.32     &     3.88     &     156.22 \\ 
 1500 & 12991       &       0.32     &     3.97     &     167.71 \\
 1600 & 14311       &       0.32     &     3.96     &     178.63 \\ 
 1700 & 15674       &       0.32     &     3.96     &     190.49 \\
 1800 & 17077       &       0.32     &     3.92     &     201.24 \\
 1900 & 18519       &       0.32     &     3.86     &     212.29 \\
 2000 & 20000       &       0.32     &     3.91     &     223.61  \\
\bottomrule
    \end{tabular}
    \label{tab:graphPropEx2}
\end{table}

\begin{figure}[!ht]
    \centering
    \includegraphics[width=0.6\linewidth]{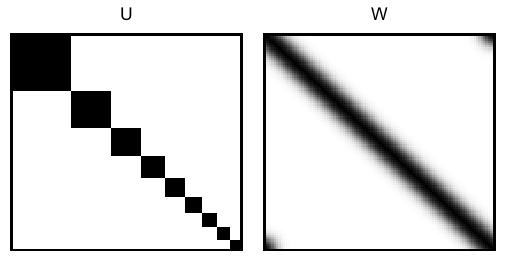}
    \caption{Graphons $U$ and $W$ used in Example 2.}
    \label{fig:UandWEx2}
\end{figure}

\begin{figure}[!ht]
    \centering
    \includegraphics[width=0.95\linewidth]{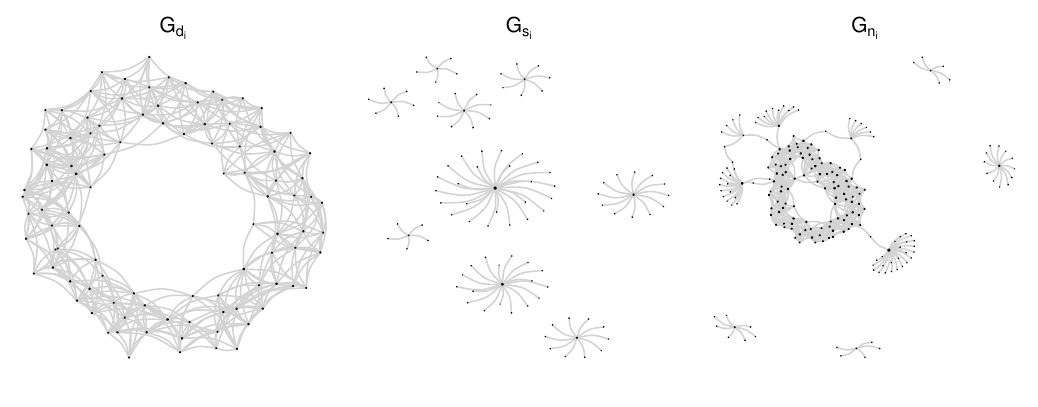}
    \caption{A $(U,W)$-mixture graph with $n_{d_i} = 100$ and $n_{s_i} = 100$ where $U$ and $W$ are shown in Figure \ref{fig:UandWEx2}.}
    \label{fig:TwoGraphMixturesForEx21}
    \includegraphics[width=0.95\linewidth]{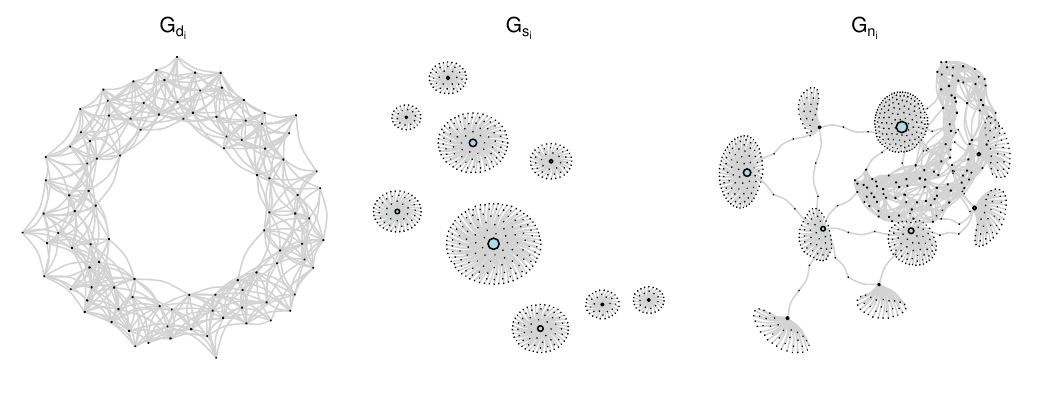}
    \caption{A $(U,W)$-mixture graph with $n_{d_i} = 100$ and $n_{s_i} = 600$ where $U$ and $W$ are shown in Figure \ref{fig:UandWEx2}}
    \label{fig:TwoGraphMixturesForEx22}
\end{figure}

We see that the mixture graphs $G_{n_i}$ shown in Figures \ref{fig:TwoGraphMixturesForEx11} and \ref{fig:TwoGraphMixturesForEx12} are quite different from those in Figures \ref{fig:TwoGraphMixturesForEx21} and \ref{fig:TwoGraphMixturesForEx22}.  The main reason is that graphon $W$ stipulates the dense part of the mixture and as such, different $W$ produces different mixture graphs. 

\subsection{An example where the cut metric cannot explain the anomalous node}\label{sec:Wcantfullyexplain}
%We provide an example motivating graphon mixtures. 
Consider the following example: Let $\{G_{n}\}_{n}$ be a sequence of graphs such that originally $G_n \sim G(n, p)$ where $G(n, p)$ denotes a graph generated from the Erdős–Rényi model where $G_n$ has $n$ nodes and the edge probability is $p$. Then we modify each graph in the sequence such that a randomly selected node in $G_n$ is connected to all other nodes in $G_n$. Figure \ref{fig:GnpWithAStar} shows  the empirical graphons for $n \in \{50, 100, 200, 400\}$ with $p = 0.1$ for this construction.

\begin{figure}[!ht]
    \centering
    \includegraphics[width=0.9\linewidth]{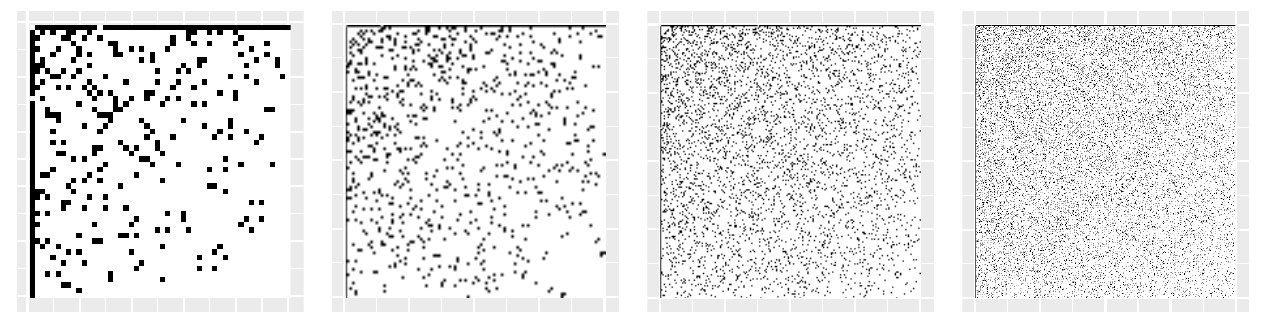}
    \caption{Empirical graphons of $G_n \sim G(n,p)$  with one node connected to all other nodes for $n \in \{50, 100, 200, 400\}$. }
    \label{fig:GnpWithAStar}
\end{figure}

The empirical graphons in this example converge to $W = 0.1$ in the cut metric (Definition \ref{def:cut2}). However if we inspect the degree distribution, we see a persistent anomalous node with high degree as shown in  Figure \ref{fig:GnpandStarHistogram}.

\begin{figure}[!ht]
    \centering
    \includegraphics[width=0.9\linewidth]{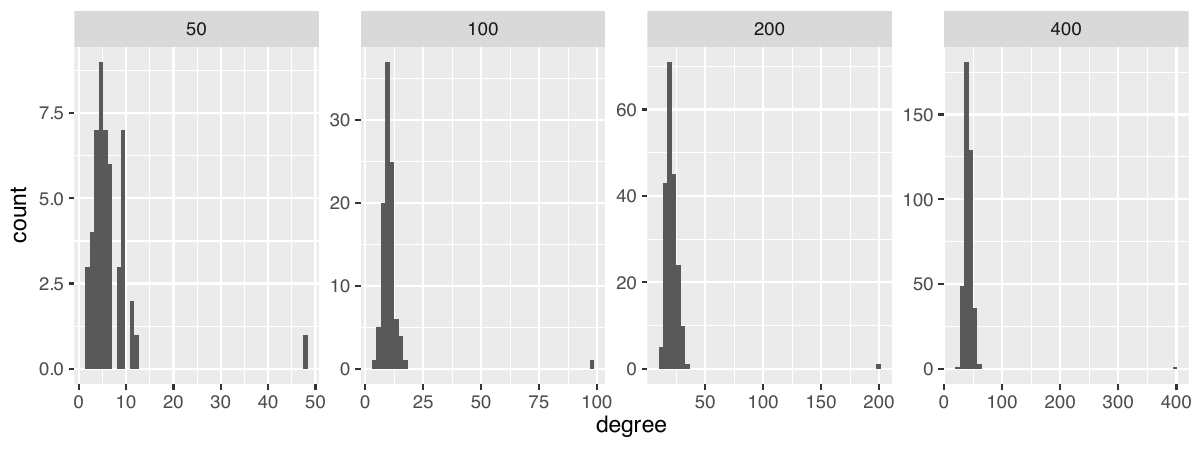}
    \caption{Degree histograms of the graphs corresponding to empirical graphons in Figure \ref{fig:GnpWithAStar}. }
    \label{fig:GnpandStarHistogram}
\end{figure}

% This is an example of a dense ($U,W)$-mixture sequence, which converges to $W = 0.1$ in the cut-metric (Definition \ref{def:cut2}). 
The empirical graphons (Definition \ref{def:empiricalgraphon}) converge to $W = 0.1$ in the cut-metric (Definition \ref{def:cut2}). While the graph sequence is dense, it has a sparse component that a single graphon $W$ does not fully  explain. That is, graphs generated from $W = 0.1$, would not have the persistent anomalous node shown in Figure \ref{fig:GnpandStarHistogram}. This shows that $W = 0.1$ does not fully capture the observed graphs. Modeling these graphs as a mixture gives the flexibility to account for the persisting high degree node.

\subsection{Generating dense and sparse graph sequences with a $(U,W)$ mixture }\label{app:example}
In this section we consider $(U,W)$-mixture graph sequences $\{ G_{n_i} \}_i$ generated with $W = \exp(-(x+y))$ and $U$ with mass-partition $\left(\frac{2}{3}, \frac{1}{3} \right)$. The graphons $U$ and $W$ are shown in Figure \ref{fig:WUExample}. We consider graph generation for different rates of evolution of $\frac{n_{s_i}}{n_{d_i}}$.

\begin{figure}[!p]
    \centering
    \includegraphics[width=0.7\linewidth]{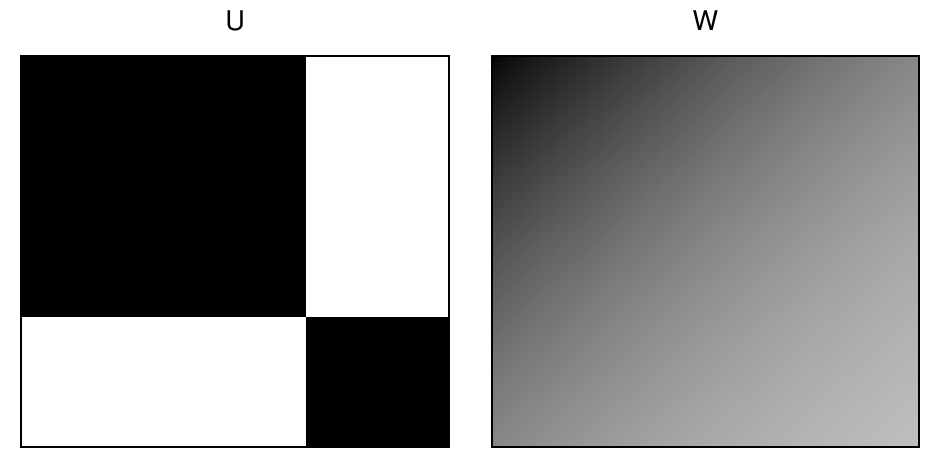}
    \caption{Graphons $U$ and $W$ in the $(U,W)$ mixture with $W = \exp(-(x+y))$ and $U$ having the mass-partition $(2/3, 1/3)$. Graphon $U$ lives in the line graph space and is used to generate sparse graphs via the inverse line graph operation. Graphon $W$ is the limit of the dense part.}
    \label{fig:WUExample}
    \includegraphics[width=0.9\linewidth]{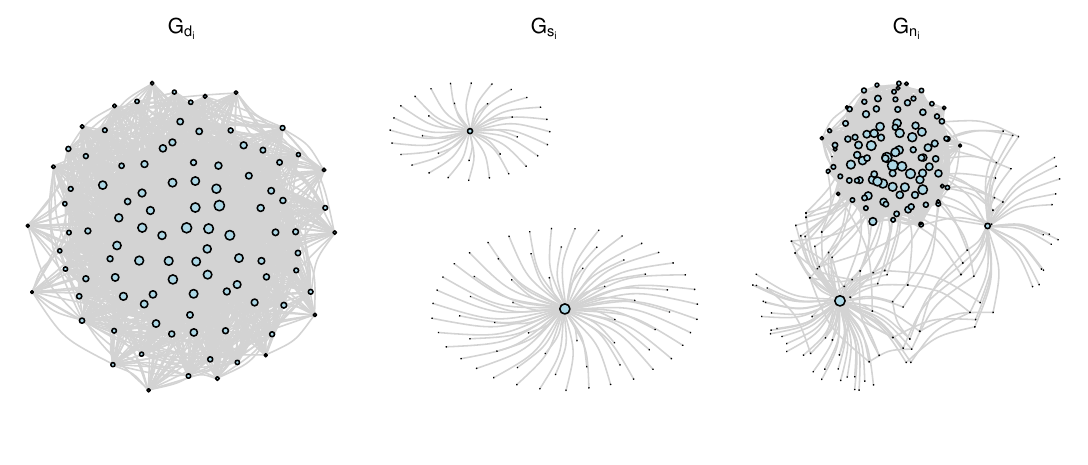}
    \caption{The dense part $G_{d_i}$, sparse part $G_{s_i}$ and the mixture graph $G_{n_i}$ Both dense part $G_{d_i}$ and sparse part $G_{s_i}$ have 100 nodes each. }
    \label{fig:densemixgraphratio1}
    \includegraphics[width=0.9\linewidth ]{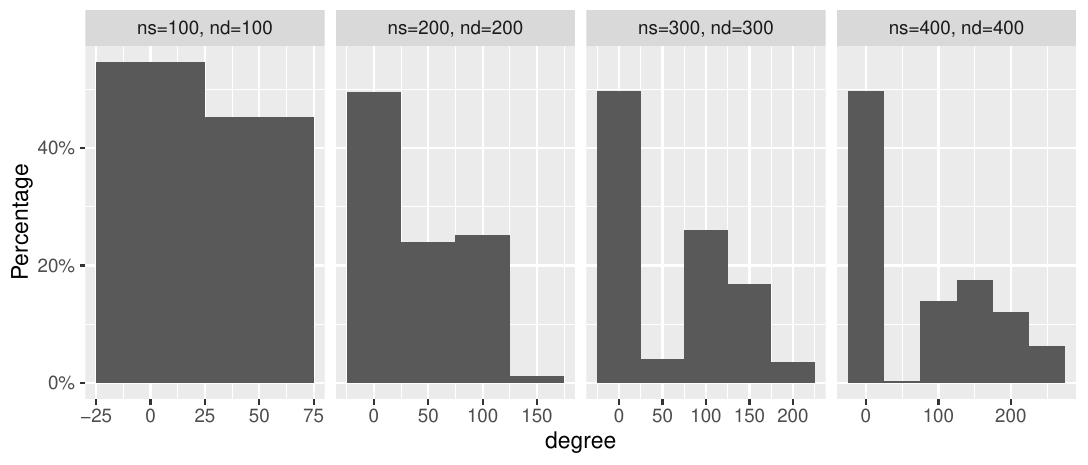}
    \caption{Degree distribution when $\frac{n_{s_i}}{n_{d_i}} = 1$}
    \label{fig:estimatedWUandDegreeDistEg11}
\end{figure}

\paragraph{Dense $\{G_{n_i}\}_i$, when $\frac{n_{s_i}}{n_{d_i}} \to c \in \mathbb{R}$:} 
When $\frac{n_{s_i}}{n_{d_i}} \to c \in \mathbb{R}$ the graph sequence $\{G_{n_i}\}_i$ is dense (Lemma \ref{lemma:WUrandomgraphs1}), i.e., the edge density of $G_{n_i}$ does not go to zero. Figure \ref{fig:densemixgraphratio1} shows an example of the dense part $G_{d_i}$, the sparse part $G_{s_i}$ and the mixture graph $G_{n_i}$ for $n_{d_i} = n_{s_i} = 100$.  Figure \ref{fig:estimatedWUandDegreeDistEg11} shows the degree distributions for different values of $n_{d_i}$ and $n_{s_i}$ when the ratio $\frac{n_{s_i}}{n_{d_i}} = 1$. When the ratio $\frac{n_{s_i}}{n_{d_i}}$ converges to $c \leq 1$, depending on $W$ and $U$, the largest degree may be contributed by the dense part. In Figure \ref{fig:estimatedWUandDegreeDistEg11} the largest degree is actually produced by $U$, however, it is so close to the degrees generated by $W$ that we cannot distinguish the effect of $U$ in this example. 

Figure  \ref{fig:densemixgraphratio3} shows the dense part $G_{d_i}$ and the sparse part $G_{s_i}$ when $n_{d_i} = 100$ and $n_{s_i} = 300$. Figure \ref{fig:estimatedWUandDegreeDistEg12} shows the degree distribution  when $\frac{n_{s_i}}{n_{d_i}} = 3$.  When the ratio $\frac{n_{s_i}}{n_{d_i}}$ converges to $c > 1$, for large enough $i$ the sparse part contributes to highest-degree nodes, which are anomalous. In Figure \ref{fig:estimatedWUandDegreeDistEg12} two anomalous nodes are contributed by $U$ for each pair of $n_{s_i}$ and $n_{d_i}$.

 \begin{figure}[!p]
    \centering
     \includegraphics[height=0.21\textheight]{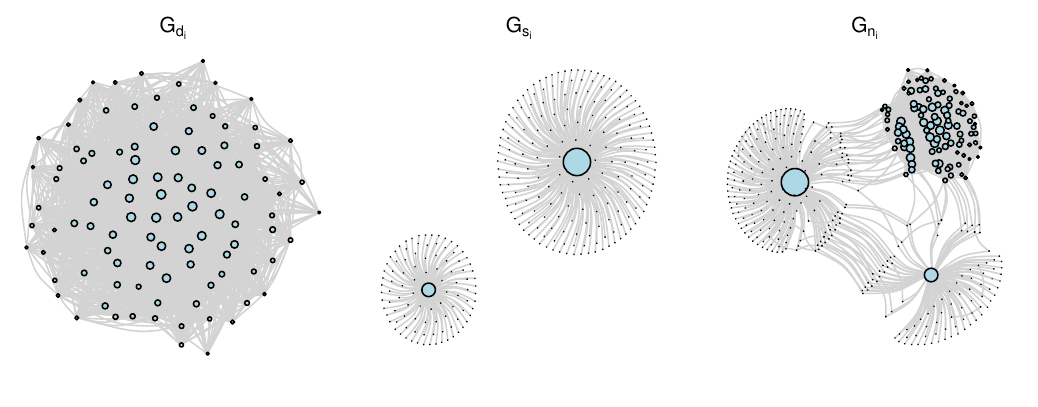}
    \caption{The dense part $G_{d_i}$, sparse part $G_{s_i}$ and the mixture graph $G_{n_i}$ with $n_{d_i} = 100$ and $n_{s_i} = 300$. }
    \label{fig:densemixgraphratio3}
    \includegraphics[height=0.21\textheight ]{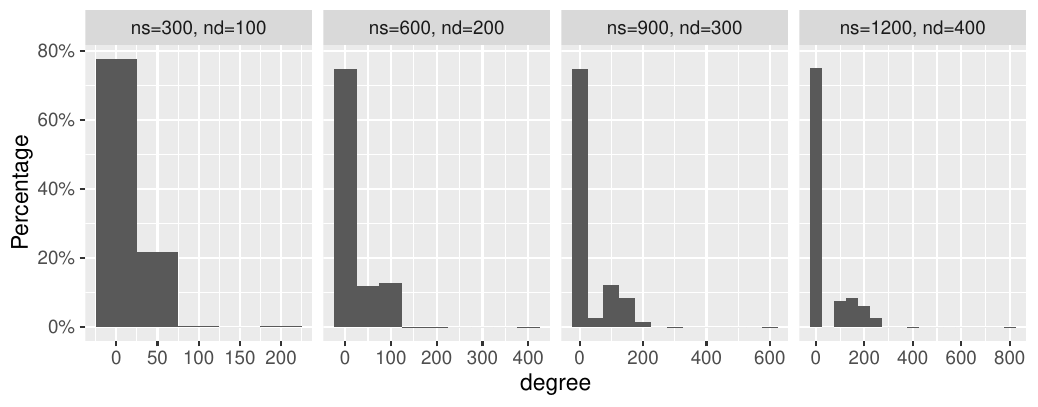}
    \caption{When $\frac{n_{s_i}}{n_{d_i}} = 3$}
    \label{fig:estimatedWUandDegreeDistEg12}

    \includegraphics[height=0.21\textheight]{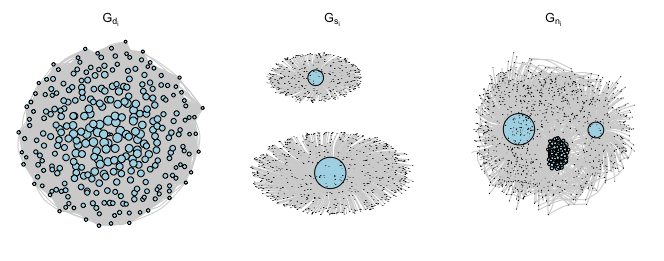}
    \caption{The dense part $G_{d_i}$, sparse part $G_{s_i}$ and the mixture graph $G_{n_i}$ with $n_{d_i} = 300$ and $n_{s_i} = 1162$. }
    \label{fig:sparsemixgraph1}

       \includegraphics[height=0.21\textheight]{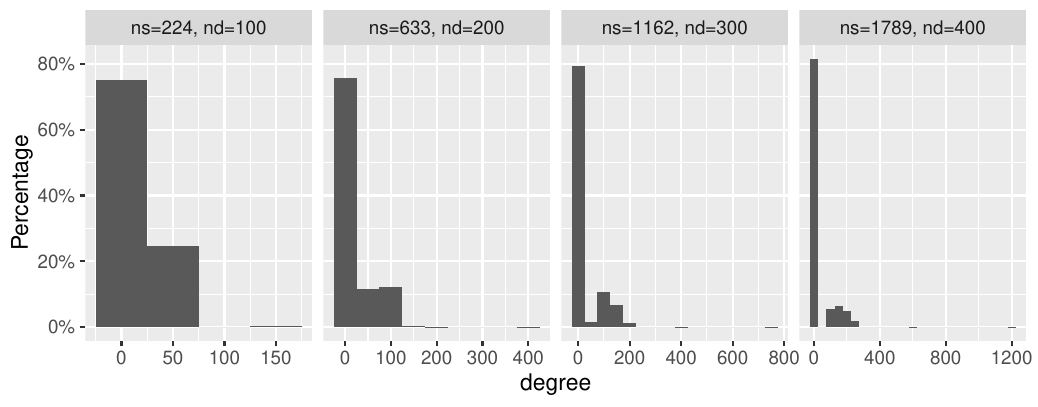}
    \caption{When $\frac{n_{s_i}}{n_{d_i}} \to \infty$  }
    \label{fig:estimatedWUandDegreeDistEg21}

   \end{figure}

% \begin{figure}[h]
%     \centering
%     \includegraphics[width=0.9\linewidth]{Graphics/TASK_04_Dense_Example.pdf}
%     \includegraphics[height=0.21\textheight ]{Graphics/Appendix_Task_03_Eg1_Degree_Dist.pdf}
%     \caption{Caption}
%     \label{fig:placeholder}
% \end{figure}
 \begin{figure}[h]
    \centering
      \includegraphics[height=0.21\textheight]{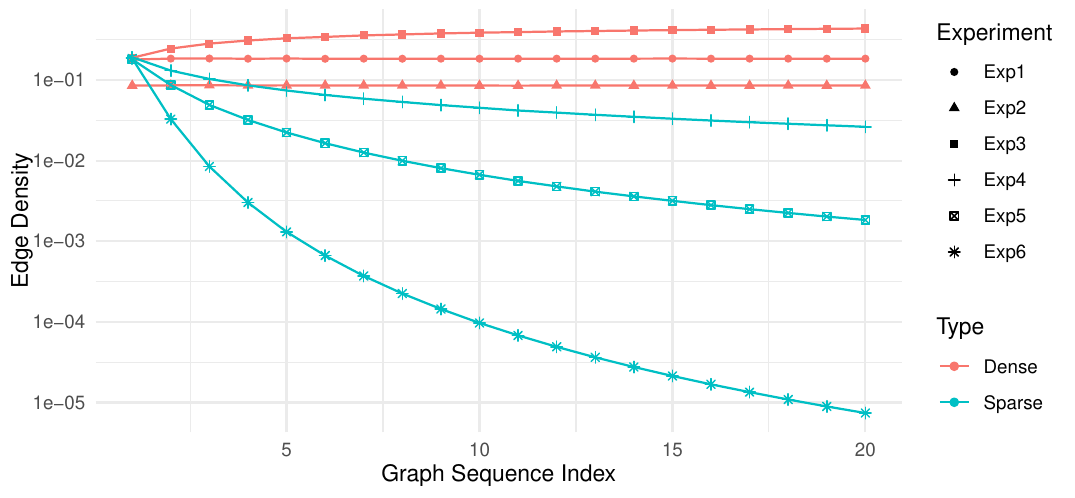}
    \caption{Sparse and dense $(U,W)$ sequences for different functions of $n_{s_i}/n_{d_i}$.}
    \label{fig:GeneratingDifferentSequences}
\end{figure}

\paragraph{Sparse $\{G_{n_i}\}_i$, when $\frac{n_{s_i}}{n_{d_i}} \to \infty$:} 
If $\{G_{n_i}\}_i$ is sparse, then $n_{s_i}$ grows faster than $n_{d_i}$ making $\frac{n_{s_i}}{n_{d_i}}$ go to infinity (Lemma \ref{lemma:WUrandomgraphs1}). We used $n_{s_i} = \lceil \sqrt{5i} \times n_{d_i} \rceil $ for this experiment and  Figure \ref{fig:estimatedWUandDegreeDistEg21} shows the degree distributions. Two anomalous nodes contributed by $U$ can be seen in 3 of the 4 subplots. Figure \ref{fig:sparsemixgraph1} shows the dense part $G_{d_i}$, the sparse part $G_{s_i}$ and the mixture graph $G_{n_i}$ when $n_{d_i} = 300$ and $n_{s_i} = 1162$. We see that the effect of the stars is higher in this mixture. 

\paragraph{Different $n_{s_i}/n_{d_i}$ ratios:} Figure \ref{fig:GeneratingDifferentSequences} shows the edge density along the graph sequence for six functions of $n_{s_i}/n_{d_i}$. The first experiment (Exp 1 in Figure \ref{fig:GeneratingDifferentSequences}) considers $n_{s_i}/n_{d_i} = 1$ and experiment 2 considers $n_{s_i}/n_{d_i} = 2$. For experiment 3 we consider $n_{s_i}/n_{d_i} \approx 1/\sqrt{i}$. As $n_{s_i}/n_{d_i}$ are bounded from above for the first 3 experiments we get dense graphs.  For experiment 4, we considered $n_{s_i}/n_{d_i} \approx \sqrt{i}$ making the mixture sequence sparse. For experiments 5 and 6 we considered $n_{s_i}/n_{d_i} = i$ and $n_{s_i}/n_{d_i} = i^2$ respectively. Experiments 4, 5 and 6 produced sparse sequences. For sparse sequences $\{n_{s_i}/n_{d_i}\}_i$ determines the rate at which the mixture becomes sparse.

\subsection{Expectations}

\begin{lemma}\label{lemma:expnewedges}
     Let $\{G_{n_i}\}_i$ be a sequence of $(U,W)$-mixture graphs (Definition \ref{def:WURandomMixtureGraphs}) with $G_{n_i} \sim \mathbb{G}\left(U,W, n_{d_i}, m_{s_i} \right)$ with dense and sparse parts $G_{d_i}$ and $G_{s_i}$ respectively. Let $n_{d_i}$ and $n_{s_i}$ be the number of nodes in $G_{d_i}$ and  $G_{s_i}$ respectively.  Suppose $m_{new_i}$ edges are added as part of the joining process. Suppose $v_\ell$ is originally a vertex in the dense part $G_{d_i}$. We denote its degree in $G_{d_i}$ by $\deg_{G_{d_i}} v_\ell$ and its degree in $G_{n_i}$ by $\deg_{G_{n_i}} v_\ell$. Then taking expectations over the joining process we have
     \begin{align}
         \mathbb{E}\left( \deg_{G_{n_i}} v_\ell \right)  & = \deg_{G_{d_i}} v_\ell  + c_1 \frac{m_{new_i}}{n_{d_i}} \, \\
         \Var(\deg_{G_{n_i}} v_\ell) & = m_{new_i}\frac{c_1}{n_{d_i}}\left(1 -  \frac{c_1}{n_{d_i}}\right) \, , 
     \end{align}
     % or
     % \begin{align}
     %     \mathbb{E}\left( \deg_{G_{n_i}} v_\ell \right)  & = \deg_{G_{d_i}} v_\ell  + c_3\frac{m_{new_i}}{n_{d_i} + n_{s_i}} \, , \\
     %    \Var(\deg_{G_{n_i}} v_\ell) & = m_{new_i} \left( \frac{c_3}{n_{d_i} + n_{s_i}} \right)\left(1 -   \frac{c_3}{n_{d_i} + n_{s_i}}\right) \, , 
     % \end{align}
     where $c_1>0$ depend on the joining process. % and expectation and variance are calculated with respect to the randomisation associated with the joining process.
     Similarly, suppose $v_j$ is originally a vertex in the sparse part $G_{s_i}$. We denote its degree in $G_{s_i}$ by $\deg_{G_{s_i}} v_j$ and its degree in $G_{n_i}$ by $\deg_{G_{n_i}} v_j$. Then, taking expectations over the joining process, 
     \begin{align}
         \mathbb{E}\left( \deg_{G_{n_i}} v_j \right)  & = \deg_{G_{s_i}} v_j  + c_2 \frac{m_{new_i}}{n_{s_i}} \, , \\
          \Var(\deg_{G_{n_i}} v_j) & = m_{new_i}\frac{c_2}{n_{s_i}}\left(1 -  \frac{c_2}{s_{d_i}}\right) \, , 
     \end{align}
      % or
      % \begin{align}
      %   \mathbb{E}\left( \deg_{G_{n_i}} v_j \right)  & = \deg_{G_{s_i}} v_j  + c_3\frac{m_{new_i}}{n_{d_i} + n_{s_i}} \, , \\
      %   \Var(\deg_{G_{n_i}} v_j) & = m_{new_i} \left( \frac{c_3}{n_{d_i} + n_{s_i}} \right)\left(1 -   \frac{c_3}{n_{d_i} + n_{s_i}}\right) \, , 
      % \end{align}
     where $c_2 >0$. 
\end{lemma}
\begin{proof}
The \textit{Random edges condition} (Definition \ref{def:joiningrules} Condition \ref{cond:randomedges}) stipulates that nodes within a graph are equally likely to be selected.  That is, the probability of selecting a node is inversely proportional to the number of nodes in the graph. For a new edge, the probability of selecting a node in $G_{d_i}$ is $\frac{c_1}{n_{d_i}}$ and the probability of selecting a node in $G_{s_i}$ is $\frac{c_2}{n_{s_i}}$. This gives rise to a binomial distribution. Thus, the expected number of new edges for node $v_{\ell}$ in $G_{d_i}$ is $c_1\frac{m_{new_i}}{n_{d_i}}$ and its variance is $m_{new_i}\frac{c_1}{n_{d_i}}\left(1 -  \frac{c_1}{n_{d_i}}\right)$. 
For node $v_{j}$ in $G_{s_i}$ the expected number of new edges is $c_2\frac{m_{new_i}}{n_{s_i}}$ and its variance is $m_{new_i}\frac{c_2}{n_{s_i}}\left(1 -  \frac{c_2}{s_{d_i}}\right)$.

%This can happen in two ways: 
% \begin{enumerate}
%     \item[A.]  For a new edge, the probability of selecting a node in $G_{d_i}$ is $\frac{c_1}{n_{d_i}}$ and the probability of selecting a node in $G_{s_i}$ is $\frac{c_2}{n_{s_i}}$.
%     \item[B.] For a new edge, the probability of selecting a node in $G_{d_i}$ or $G_{s_i}$ is $\frac{1}{n_{s_i} + n _{d_i}}$.
% \end{enumerate}
% Both cases give rise to binomial distributions. For case A, the expected number of new edges for node $v_{\ell}$ in $G_{d_i}$ is $c_1\frac{m_{new_i}}{n_{d_i}}$ and its variance is $m_{new_i}\frac{c_1}{n_{d_i}}\left(1 -  \frac{c_1}{n_{d_i}}\right)$. 
%     For node $v_{j}$ in $G_{s_i}$ the expected number of new edges is $c_2\frac{m_{new_i}}{n_{s_i}}$ and its variance is $m_{new_i}\frac{c_2}{n_{s_i}}\left(1 -  \frac{c_2}{s_{d_i}}\right)$.

%     For case B, probability of a new edge is the same for any vertex regardless of the graph of origin. Thus the expected number of new edges is $c_3\frac{m_{new_i}}{n_{s_i} + n_{d_i}}$ for any vertex and its variance is $m_{new_i} \left( \frac{c_3}{n_{d_i} + n_{s_i}} \right)\left(1 -   \frac{c_3}{n_{d_i} + n_{s_i}}\right)$.
\end{proof}

 \begin{lemma}\label{lemma:WUrandomgraphsaboutU2}
Let $\{G_{n_i}\}_i$ be a sequence of $(U,W)$-mixture graphs (Definition \ref{def:WURandomMixtureGraphs}) with $G_{n_i} \sim \mathbb{G}\left(U,W, n_{d_i}, m_{s_i} \right)$ with sparse part $G_s$. Let $\bm{p} = (p_1, p_2, \ldots)$ be the mass-partition (Definition \ref{def:masspartition}) corresponding to $U$.  Let $\tilde{q}_{j,i}$ be the degree of the star in $G_{s_i}$ corresponding to $p_j \neq 0$ and let $q_{j,i}$ denote the degree of the corresponding node in $G_{n_i}$. Then taking expectations over graph generation and the joining process
\begin{align}
    \mathbb{E}({q}_{j, i}) & = m_{s_i} p_j + \frac{c m_{new_i}}{n_{s_i}} \, , \\
    \Var({q}_{j, i}) &= m_{s_i} p_j(1 - p_j) + m_{new_i} \frac{c} {n_{s_i}} \left( 1 - \frac{c} {n_{s_i}} \right)   \, . 
\end{align}
 \end{lemma}  
%\lemmaWUrandomgraphsaboutUTwo*
\begin{proof}
    The \textit{Random edges condition} in Definition \ref{def:joiningrules} specifies that edges are added randomly. From Lemma \ref{lemma:WUrandomgraphsaboutU} we know that 
     \[ \mathbb{E}(\tilde{q}_{j, i}) = m_{s_i} p_j  \quad \text{and} \quad  \Var(\tilde{q}_{j, i})  = m_{s_i} p_j(1 - p_j) \, . 
\]
Combining with Lemma \ref{lemma:expnewedges} we obtain
\begin{align}
    \mathbb{E}(q_{j, i}) & = \mathbb{E}(\tilde{q}_{j, i}) + \frac{c m_{new_i}}{n_{s_i}} \, \\
     \mathbb{E}(q_{j, i}) & = m_{s_i}p_j + \frac{c m_{new_i}}{n_{s_i}} \, . 
\end{align}
Similarly the variance satisfies
\begin{align}
     \Var(q_{j, i}) & = \Var(\tilde{q}_{j, i}) + \Var(\text{new edges for this vertex}) \, , \\
    \Var(q_{j, i}) & = m_{s_i} p_j(1 - p_j)  + m_{new_i}\frac{c}{n_{s_i}}\left( 1 - \frac{c}{n_{s_i}} \right) \, . 
\end{align}
\end{proof}

\subsection{Edge density of $(U,W)$-mixture graphs }

\lemmaWUrandomgraphsOne*
\begin{proof}
    Suppose $G_{d_i}$, $G_{s_i}$ and $G_{n_i}$ have $n_{d_i}$, $n_{s_i}$ and $n_{i}$ nodes and $m_{d_i}$, $m_{s_i}$ and $m_i$ edges respectively. Then
     $m_{d_i} + m_{s_i} \leq  m_i$ and $n_i = n_{d_i} + n_{s_i}$.  %  \leq m_{d_i} + m_{s_i} + c_1m_{d_i}
    \begin{enumerate}
        \item Then the edge density of $G_{n_i}$ satisfies
        \begin{align}
        \density(G_{n_i}) & \geq \frac{2(m_{d_i} + m_{s_i}) }{ (n_{d_i} + n_{s_i})^2} \, , \\
         & = \frac{2(m_{d_i} + m_{s_i}) }{(n_{d_i}^2 + 2 n_{d_i} n_{s_i}  +  n_{s_i}^2)} \, , \\
        & = 2\frac{\frac{m_{d_i}}{n_{d_i}^2} + \frac{m_{s_i}}{n_{d_i}^2}}{ 1 + 2 \frac{n_{s_i}}{n_{d_i}} + \frac{n_{s_i}^2}{n_{d_i}^2}} \, ,  \\
        & = 2\frac{\frac{m_{d_i}}{n_{d_i}^2} + \frac{m_{s_i}}{n_{s_i}^2} \frac{n_{s_i}}{n_{d_i}^2} }{ 1 + 2 \frac{n_{s_i}}{n_{d_i}} + \frac{n_{s_i}^2}{n_{d_i}^2}} \, , \\
        & \geq \frac{ \rho_{d_i}}{1 + 2c + c^2} \, , 
    \end{align}    

    where $\rho_{d_i}$ is the edge density of the dense part $G_{d_i}$. As $G_{d_i}$ is sampled from $W$, the edge density $\rho_{d_i}$ is bounded away from zero as $i$ goes to infinity. This makes the edge density of $G_{n_i}$ strictly positive  as $i$ goes to infinity, making $\{G_{n_i}\}_i$  dense.  
    \item Suppose %there does not exist a $c \in \mathbb{R}^+$ such that   for all $i$ we have $0 \leq \frac{n_{s_i}}{n_{d_i}} \leq c$. That means that 
    $\lim_{i \to \infty} \frac{n_{s_i}}{n_{d_i}} = \infty$. Equivalently, $\lim_{i \to \infty} \frac{n_{d_i}}{n_{s_i}} = 0$. As $m_i = m_{d_i} + m_{s_i} + c'm_{d_i}$ where the new edges $m_{new_i} = c'm_{d_i}$ depend on the joining process, the edge density of $G_{n_i}$ is given by
    \begin{align}
         \density(G_{n_i}) & = \frac{2(m_{d_i} + m_{s_i} + c'm_{d_i} ) }{ (n_{d_i} + n_{s_i})^2} \, , \\
          & = \frac{2((1+ c')m_{d_i} + m_{s_i}) }{ n_{d_i}^2 + 2 n_{d_i} n_{s_i} + n_{s_i}^2  } \, , \\
          & = 2 \frac{  (1+ c')\frac{m_{d_i}}{n_{s_i}^2}  + \frac{m_{s_i}}{n_{s_i}^2}  }{  \frac{n_{d_i}^2}{n_{s_i}^2} + 2 \frac{n_{d_i}}{n_{s_i}}  + 1 } \\
        & = 2 \frac{  (1+ c')\frac{m_{d_i}}{n_{d_i}^2}\frac{n_{d_i}^2}{{n_{s_i}^2}} + \frac{m_{s_i}}{n_{s_i}^2}  }{  \frac{n_{d_i}^2}{n_{s_i}^2} + 2 \frac{n_{d_i}}{n_{s_i}}  + 1 } \\
        &  = 2 \frac{  (1+ c')\rho_{d_i}\frac{n_{d_i}^2}{{n_{s_i}^2}} + \frac{m_{s_i}}{n_{s_i}^2}  }{  \frac{n_{d_i}^2}{n_{s_i}^2} + 2 \frac{n_{d_i}}{n_{s_i}}  + 1 } \\
   \lim_{i \to \infty}  \density(G_{n_i})     & = 0 \, , 
    \end{align}
 as $\{G_{s_i}\}_i$ is sparse and as $\lim_{i \to \infty} \frac{n_{d_i}}{n_{s_i}} = 0$.         
\end{enumerate}
\end{proof}

\begin{lemma}%{lemma:WUrandomgraphsThree}
    \label{lemma:WUrandomgraphs3}
    Let $\{G_{n_i}\}_i$ be a sequence of sparse $(U,W)$-mixture graphs (Definition \ref{def:WURandomMixtureGraphs}) with dense and sparse parts of $G_{n_i}$ being $G_{d_i}$ and $G_{s_i}$ respectively. Let $\bm{p} = (p_1, p_2, \ldots )$ be the mass-partition (Definition \ref{def:masspartition}) associated with $U$ and suppose $\bm{p}$ has $k$ non-zero elements. Recall $G_{s_i}$ is a union of disjoint stars and isolated edges.  Let $\tilde{q}_{j,i}$ be the degree of the star in $G_{s_i}$ corresponding to $p_j \neq 0$. % and let $q_{j,i}$ denote the degree of the corresponding node in $G_{n_i}$.  %Then %there exists $n_i > N_{0}(k)$, such that 
    % the top $k$ degrees of $G_{n_i}$ are contributed by $G_{s_i}$ with high probability. That is, 
    Then, if we consider $G_{n_i}$ to be the disjoint union of $G_{d_i}$ and $G_{s_i}$ % there exists  $n_i > N_{0}(k)$ such that 
    the top $k$ degrees denoted by $\deg v_{(1\ldots k)}$ satisfies  
    \[ \lim_{i \to \infty } P\left( \bigcup_{j=1}^k \tilde{q}_{j, i} \subseteq \deg v_{(1\ldots k)} \right) = 1 \, . 
    \]
\end{lemma}
%\lemma:WUrandomgraphsThree*
\begin{proof}
From Lemma \ref{lemma:WUrandomgraphs1} we know if $\{G_{n_i}\}_i$ is sparse, then $n_{s_i}/n_{d_i} \to \infty$ as $i$ tends to infinity. Furthermore, $n_i = n_{d_i} + n_{s_i}.$ %in this case $n_i = \max(n_{s_i}, n_{d_i}) = n_{s_i}$. 
In the dense part $G_{d_i}$, vertices have degree less than or equal to $n_{d_i}$. 

From Lemma \ref{lemma:WUrandomgraphsaboutU} the degree of the star in $G_{s_i}$ corresponding to $p_j$ satisfies $\mathbb{E}(\tilde{q}_j) = m_{s_i}p_j$.  Then using Chernoff bounds
\[ P\left( \tilde{q}_j \leq (1 - \epsilon) m_{s_i}p_j  \right) \leq \exp\left( - \frac{m_{s_i}p_j \epsilon^2}{2} \right) \, . 
\]
The sparse part $G_{s_i}$ has $n_{s_i}$ nodes and $m_{s_i}$ edges and is the inverse line graph of  $ H_{s_i} \sim \mathbb{G}(U, m_{s_i})$. As the inverse line graph of a clique is a star and the inverse line graph of an isolated node is an isolated edge, we have $m_{s_i} \in \Theta(n_{s_i})$. This gives
\[ m_{s_i} p_j \geq c_1p_j n_{s_i} > n_{d_i} \, . 
\]
for large enough $i$ as $n_{s_i}/n_{d_i} \to \infty$. Thus, by letting $n_{d_i} \leq (1 - \epsilon) m_{s_i} p_j$ we get 
\begin{equation}\label{eq:chernoffbounds1}
    P\left( \tilde{q}_j \leq n_{d_i}  \right) \leq \exp\left( - \frac{m_{s_i}p_j}{2} \left(1 - \frac{n_{d_i}}{m_{s_i}p_j} \right)^2 \right) \, ,
\end{equation}
which goes to zero fast as $n_{d_i}/(m_{s_i}p_j)$ goes to zero and as $m_{s_i}$ goes to infinity. For $j \in \{1, \ldots, k\}$ if all $\tilde{q}_{j,i} > n_{d_i}$ then $ \bigcup_{j=1}^k \tilde{q}_{j, i} \subseteq \deg v_{(1\ldots k)}$ for large $i$. Then
\ifthenelse{\boolean{twocolumn}}{
  % code for two column
  \begin{align}
    P& \left( \bigcup_{j=1}^k \tilde{q}_{j, i} \subseteq \deg v_{(1\ldots k)} \right)  =  \\
    & P\left( (\tilde{q}_{1, i} > n_{d_i}) \cap (\tilde{q}_{2, i} > n_{d_i}) \cdots (\tilde{q}_{k, i} > n_{d_i}) \right) \, , \\
    & \geq 1 - \sum_{j = 1}^k P\left( \tilde{q}_{j, i} \leq n_{d_i} \right) \, , \\
    & \geq 1 - \sum_{j = 1}^k \exp\left( - \frac{m_{s_i}p_j}{2} \left(1 - \frac{n_{d_i}}{m_{s_i}p_j} \right)^2 \right) \, , 
\end{align}
giving the result. 
}{
  % code for single column  
  \begin{align}
    P\left( \bigcup_{j=1}^k \tilde{q}_{j, i} \subseteq \deg v_{(1\ldots k)} \right) & = P\left( (\tilde{q}_{1, i} > n_{d_i}) \cap (\tilde{q}_{2, i} > n_{d_i}) \cdots (\tilde{q}_{k, i} > n_{d_i}) \right) \, , \\
    & \geq 1 - \sum_{j = 1}^k P\left( \tilde{q}_{j, i} \leq n_{d_i} \right) \, , \\
    & \geq 1 - \sum_{j = 1}^k \exp\left( - \frac{m_{s_i}p_j}{2} \left(1 - \frac{n_{d_i}}{m_{s_i}p_j} \right)^2 \right) \, , 
\end{align}
giving the result. 
}

\end{proof}

\subsection{When $\sum_i p_i < 1$}\label{sec:whenpileq1}
If we focus on $U$ in the $(U,W)$-mixture graphs, we see that $U \rightarrow H_m \rightarrow G_s = L^{-1}(H_m)$. The graphon $U$ is used to generate a disjoint clique graph $H_m$, of which the inverse line graph is computed to obtain $G_s$, the sparse part. Suppose $\bm{p}$ is the mass-partition corresponding to $U$. If $\sum_j p_j = 1$, then $H_m$ consists of  disjoint cliques and $G_s$ consists of disjoint stars. If $\sum_j p_j < 1$, then $H_m$ consists of disjoint cliques including isolated vertices, which are technically cliques of size 1. The inverse line graph of isolated vertices are isolated edges, which are technically stars of 2 vertices, $K_{1,1}$. This is illustrated in Figure \ref{fig:linegraphonUandStarAndIsolatedEdges}. 
 
\begin{figure}[!ht]
    \centering
    \includegraphics[width=0.3\linewidth]{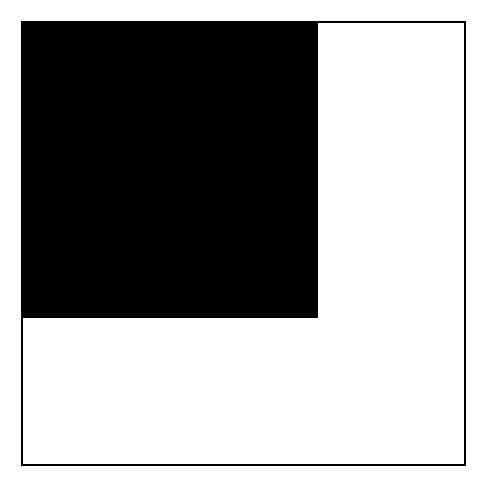}
    \includegraphics[width=0.6\linewidth]{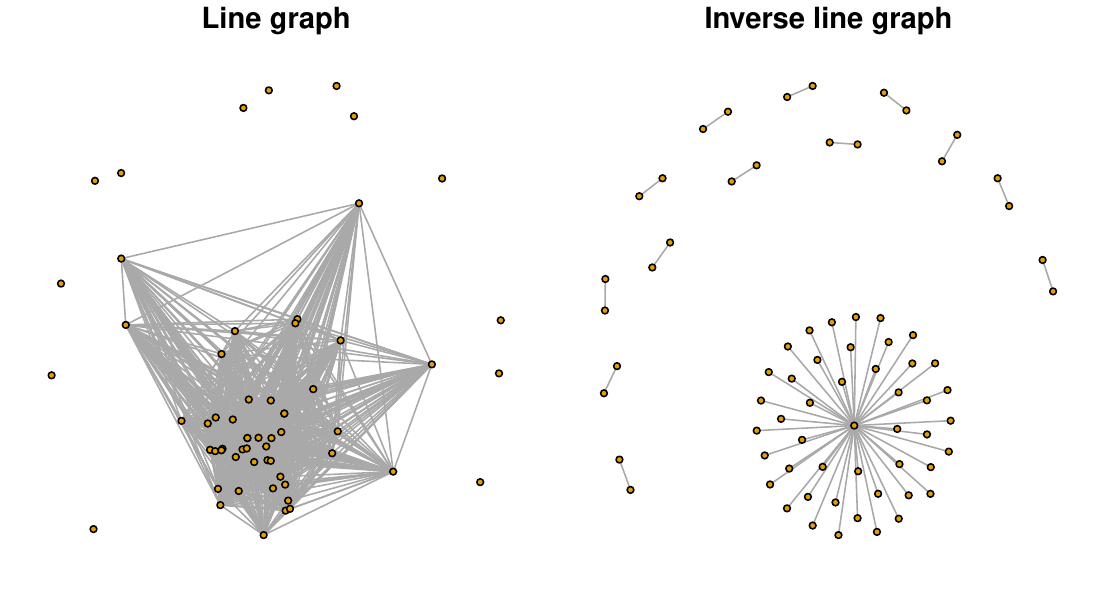}    
    \caption{A disjoint clique graphon $U$ on the left, graph $H_m \sim \mathbb{G}(U, m)$ in the middle and $G_n = L^{-1}(H_m)$ on the right. Graph $H_m$ has a large clique and some isolated nodes. Graph $G_n$ is the inverse line graph of $H_m$ and has a star and isolated edges. The star is the inverse line graph of the large clique and isolated edges are the inverse of isolated vertices. }
    \label{fig:linegraphonUandStarAndIsolatedEdges}
\end{figure}

\begin{lemma}\label{lemma:WUrandomgraphsaboutU}
 Let $\{G_{n_i}\}_i$ be a sequence of $(U,W)$-mixture graphs (Definition \ref{def:WURandomMixtureGraphs}) with $G_{n_i} \sim \mathbb{G}\left(U,W, n_{d_i}, m_{s_i} \right)$ with sparse part $G_s$. Let $\bm{p} = (p_1, p_2, \ldots)$ be the mass-partition (Definition \ref{def:masspartition}) corresponding to $U$.  Then every non-zero $p_j$ results in a star $K_{1, \tilde{q}_j}$ in $G_{s}$ with degree of the hub-node $\tilde{q}_{j, i}$ satisfying
 \[ \mathbb{E}(\tilde{q}_{j, i}) = m_{s_i} p_j  \quad \text{and} \quad  \Var(\tilde{q}_{j, i})  = m_{s_i} p_j(1 - p_j) \, . 
\]
If $\sum_j p_j < 1$ then $G_{s}$ has isolated edges in addition to disjoint stars. Let $\tilde{I}_i$ denote the number of isolated edges in $G_{s}$. Then
\ifthenelse{\boolean{twocolumn}}{
  % code for two column
  \begin{align*}
      \mathbb{E}(\tilde{I}_i) & = m_{s_i} \left(1 - \sum_j p_j \right)   \text{and} \, , \\
      \Var(\tilde{I}_i) & = m_{s_i} \sum_j p_j \left(1 - \sum_j p_j \right) \, .
  \end{align*}
  }{
  % code for single column  
   \[ \mathbb{E}(\tilde{I}_i) = m_{s_i} \left(1 - \sum_j p_j \right)  \quad \text{and} \quad \Var(\tilde{I}_i) = m_{s_i} \sum_j p_j \left(1 - \sum_j p_j \right) \, . 
     \]
}

 \end{lemma}
\begin{proof}
    For a $(U,W)$-mixture graph $G_{n_i}$, the sparse part is given by $G_{s_i}$. For the sparse part, we draw $m_s$ points $x_1, x_2, \ldots, x_{m_s}$ uniformly from $[0, 1]$ (see Definition \ref{def:WURandomMixtureGraphs}). The vertices $k$ and $\ell$ are connected with probability $U(x_k, x_\ell)$. This graph is called $H_{s}$. As $U$ is a disjoint clique graphon it can be represented by a mass-partition $\bm{p} = (p_1, p_2, \ldots )$. Thus, $H_s$ is a collection of disjoint cliques. As in Definition \ref{def:compsize} let $C_j$ denote the component size or the number of nodes in the $j$th largest component.  As $x_1, x_2, \ldots, x_{m_s}$ are uniformly drawn, for non-zero $p_j$, the component size $C_j(H_s)$ has a binomial distribution with parameters $m_s$ and $p_j$. Thus for non-zero $p_j$
    \[ \mathbb{E}(C_j(H_s)) = m_{s_i} p_j \, , 
    \]
    and 
    \[ \Var(C_j(H_s))  = m_{s_i} p_j(1 - p_j)
    \]
    where expectation is computed with respect to different randomizations. 
    When we compute the inverse line graph of $L^{-1}(H_s)$, each clique corresponding to non-zero $p_j$ is converted to a star $K_{1, \tilde{q}_j}$ such that the degree of the hub node is equal to the number of nodes in the clique, i.e., $\tilde{q}_{j,i} = C_j(H_s)$,  giving us
    \[ \mathbb{E}(\tilde{q}_{j, i} ) = m_{s_i} p_j   \quad \text{and} \quad  \Var(\tilde{q}_{j, i})  = m_{s_i} p_j(1 - p_j) \, .
    \]
    \item     When $\sum_j p_j < 1$, this results in isolated nodes in $H_s$. The number of isolated nodes $H_s$ is equal to $\tilde{I}_i$, the number of isolated edges in $G_s$ as $G_s$ is the inverse line graph of $H_s$. As $x_1, x_2, \ldots, x_{m_s}$ are uniformly drawn from $[0,1]$, $\tilde{I}_{i}$ has a binomial distribution with parameters $m_{s_i}$ and $p_0 = 1 - \sum_j p_j$ making
    \ifthenelse{\boolean{twocolumn}}{
      % code for two column
      \begin{align*}
          \mathbb{E}(\tilde{I}_{i} ) & = m_{s_i}p_0 = m_{s_i} \left(1 - \sum_j p_j \right)  \quad \text{and}  \\
          \Var(\tilde{I}_{i}) & =  m_{s_i}p_0(1 - p_0) =m_{s_i}  \sum_j p_j \left(1 - \sum_j p_j \right)
      \end{align*}
    }{
      % code for single column  
       \[ \mathbb{E}(\tilde{I}_{i} ) = m_{s_i}p_0 = m_{s_i} \left(1 - \sum_j p_j \right)  \quad \text{and} \quad \Var(\tilde{I}_{i}) =  m_{s_i}p_0(1 - p_0) =m_{s_i}  \sum_j p_j \left(1 - \sum_j p_j \right) \, . 
    \]
    }

\end{proof}

%% file: 10_C_Appendix_Top_k_Degrees.tex
\section{Predicting top-$k$ degrees}

\begin{lemma}\label{lemma:WUrandomgraphs5}
    Let $\{G_{n_i}\}_i$ be a sequence of sparse $(U,W)$-mixture graphs (Definition \ref{def:WURandomMixtureGraphs}) with dense and sparse parts $G_{d_i}$ and $G_{s_i}$ respectively. Let $\bm{p} = (p_1, p_2, \ldots )$ be the mass-partition (Definition \ref{def:masspartition}) associated with $U$ and suppose $p_j > p_k \neq 0$. Let $\tilde{q}_{j, i}$ and  $\tilde{q}_{k, i}$ denote the degree of star vertices in $G_{s_i}$ corresponding to $p_j$ and $p_k$ respectively. Let  ${q}_{j, i}$ and  ${q}_{k, i}$ denote their degrees in $G_{n_i}$. Then
% \[  \lim_{i \to \infty} P\left( q_{k, i} >  q_{j,i} \right) = 0 \, . 
% \]
\[
 P\left( q_{k, i} >  q_{j,i} \right) \leq \frac{C}{m_{s_i}}
\]
    That is, with high probability the graph mixture does not change the order of the degree of star vertices.
\end{lemma}

% \lemmaWUrandomgraphsFive*
\begin{proof}
Let $G_{d_i}, G_{s_i}$ and $G_{n_i}$ have $n_i$, $n_{d_i}$ and $n_{s_i}$ number of nodes and $m_i$, $m_{d_i}$ and $m_{s_i}$ number of edges respectively. Let $m_{new_i}$ denote the number of edges added as part of the joining process. 
From Lemma \ref{lemma:WUrandomgraphsaboutU2} we know that 
\begin{align}
    \mathbb{E}({q}_{j, i}) & = m_{s_i} p_j + \frac{c m_{new_i}}{n_{s_i}} \, , \\
    \Var({q}_{j, i}) &= m_{s_i} p_j(1 - p_j) + m_{new_i} \frac{c} {n_{s_i}} \left( 1 - \frac{c} {n_{s_i}} \right)   \, . 
\end{align}

We define $X_i = q_{k,i} - q_{j,i}$, where $q_{k,i}$ and $q_{j,i}$ are the degrees of the hubs in $G_{n_i}$ corresponding to $p_k$ and $p_j$ respectively.  Then 
\ifthenelse{\boolean{twocolumn}}{
  % code for two column
  \begin{align}
     \mathbb{E}(X_i) & = m_{s_i}( p_k - p_j)  \, ,  \\
     \Var(X_i) & = m_{s_i} \left( p_j(1 - p_j) +  p_k(1 - p_k) \right)   \\
     & \quad  +  2m_{new_i} \frac{c} {n_{s_i}} \left( 1 - \frac{c} {n_{s_i}} \right)  \, , \\
     & = m_{s_i} \left( p_j(1 - p_j) + p_k(1 - p_k) \right) \\
     & \quad + 2c'm_{d_i} \frac{c} {n_{s_i}} \left( 1 - \frac{c} {n_{s_i}} \right) \, . 
\end{align}
}{
  % code for single column  
  \begin{align}
     \mathbb{E}(X_i) & = m_{s_i}( p_k - p_j)  \, ,  \\
     \Var(X_i) & = m_{s_i} \left( p_j(1 - p_j) + p_k(1 - p_k) \right) + 2m_{new_i} \frac{c} {n_{s_i}} \left( 1 - \frac{c} {n_{s_i}} \right)  \, , \\
     & = m_{s_i} \left( p_j(1 - p_j) + p_k(1 - p_k) \right) + 2c'm_{d_i} \frac{c} {n_{s_i}} \left( 1 - \frac{c} {n_{s_i}} \right) \, . 
\end{align}
}

As $m_{s_i} \in \Theta(n_{s_i})$,  $m_{d_i} \in \Theta(n_{d_i}^2)$ and $n_{d_i}/n_{s_i} \to 0$ for sparse graphs we obtain
\begin{align}\label{eq:1LemmaWUrandomgraphs5}
    \frac{\Var X_i}{ \mathbb{E}(X_i)^2} & \leq \frac{C}{m_{s_i}} \, ,  \\
\text{making} \quad   \lim_{i \to \infty} \frac{\Var X_i}{ \mathbb{E}(X_i)^2} & = 0 \, ,  \notag
\end{align}
As $p_j > p_k$, we have $\mathbb{E}(X_i) < 0$. 
Furthermore, the probability 
\ifthenelse{\boolean{twocolumn}}{
  % code for two column
  \begin{align*}
    P(X_i >0  ) & = P(X_i - \mathbb{E}(X_i) \geq - \mathbb{E}(X_i)) \, , \\
    & = P(X_i - \mathbb{E}(X_i) \geq |\mathbb{E}(X_i)|) \, ,  \\
     & \qquad \text{as } -\mathbb{E}(X_i) = | \mathbb{E}(X_i)| \, ,  \\
    & \leq P(|X_i - \mathbb{E}(X_i)| \geq |\mathbb{E}(X_i)|)  \, ,  \\
    & \leq \frac{\Var(X_i)}{\mathbb{E}(X_i)^2}
    \end{align*}
}{
  % code for single column  
  \begin{align*}
    P(X_i >0  ) & = P(X_i - \mathbb{E}(X_i) \geq - \mathbb{E}(X_i)) \, , \\
    & = P(X_i - \mathbb{E}(X_i) \geq |\mathbb{E}(X_i)|) \quad \text{as } -\mathbb{E}(X_i) = | \mathbb{E}(X_i)| \, ,  \\
    & \leq P(|X_i - \mathbb{E}(X_i)| \geq |\mathbb{E}(X_i)|)  \, ,  \\
    & \leq \frac{\Var(X_i)}{\mathbb{E}(X_i)^2}
    \end{align*}
}

where we have used Chebyshev’s inequality in the last step and the fact that $P(|X| \geq a) \geq P(X \geq a)$ as the set $X \geq a$ is a subset of the set $|X| \geq a$. Combining with equation \eqref{eq:1LemmaWUrandomgraphs5} we get
\[
  P(q_{k,i} > q_{j, i}) \leq \frac{C}{m_{s_i}} \, , 
\]
making
\[ \lim_{i \to \infty} P(q_{k,i} > q_{j, i}) = 0 \, . 
\]
\end{proof}

\begin{lemma}\label{lemma:WUrandomgraphs4}
    Let $\{G_{n_i}\}_i$ be a sequence of sparse $(U,W)$-mixture graphs (Definition \ref{def:WURandomMixtureGraphs}) with dense and sparse parts $G_{d_i}$ and $G_{s_i}$ respectively. Let $\bm{p} = (p_1, p_2, \ldots )$ be the mass-partition (Definition \ref{def:masspartition}) associated with $U$ and suppose $p_k \neq 0$.  Let $v_\ell$ be a vertex originally from the dense part $G_{d_i}$ and let us denote the degree of $v_\ell$ in $G_{d_i}$ by  $\deg_{G_{d_i}} v_\ell$  and its degree in $G_{n_i}$ by $\deg_{G_{n_i}} v_\ell$. Then 
% \[  \lim_{i \to \infty} P\left( \deg_{G_{n_i}} v_\ell \geq (1 - \epsilon) m_{s_i} p_k \right) = 0 \, . 
% \]
\[
 P\left( \deg_{G_{n_i}} v_\ell \geq (1 - \epsilon) m_{s_i} p_k \right) \leq \exp\left(-C \frac{m_{s_i}^2}{n_{d_i}^2} \right)
\]
    where $\epsilon >0$ and $m_{s_i}$ denotes the number of edges in the sparse part.  That is, the probability of a node in the dense part having a larger degree than the expected degree of a star node $G_{n_i}$ goes to zero. 
\end{lemma}
\begin{proof}

Recall the joining process (Definition \ref{def:joiningrules}) does not delete or add vertices. Hence each vertex in $G_{n_i}$ comes either from the dense part or from the sparse part. Suppose vertex $v_\ell$ comes from the dense part $G_{d_i}$. Let us  denote the degree of $v_\ell$ in $G_{d_i}$ by  $\deg_{G_{d_i}} v_\ell$  and denote its degree in $G_{n_i}$ by $\deg_{G_{n_i}} v_\ell$. Then from Lemma \ref{lemma:expnewedges} we know that 
\[ \mathbb{E}(\deg_{G_{n_i}} v_\ell) =  \deg_{G_{d_i}} v_\ell + c \frac{m_{new_i}}{n_{d_i}}  \leq \deg_{G_{d_i}} v_\ell + c' \frac{m_{d_i}}{n_{d_i}}  \, . 
\]
Rewriting $\mu_{\ell,i} = \mathbb{E}(\deg_{G_{n_i}} v_\ell) $ for $\delta > 0$ we have the following Chernoff bound
\[ P\left( \deg_{G_{n_i}} v_\ell \geq (1 + \delta)\mu_{\ell, i} \right) \leq \exp\left( - \frac{ \mu_{\ell, i} \delta^2}{3} \right) \, . 
\]
From Lemma \ref{lemma:WUrandomgraphsaboutU} we know $\mathbb{E}(\tilde{q}_{k, i}) = m_{s_i} p_k$ where $\tilde{q}_{k, i}$ denotes the degree of the hub vertex corresponding to $p_k$ in $G_{s_i}$. Then as $n_{s_i}/n_{d_i}$ goes to infinity for sparse graphs, for a fixed $\epsilon >0$ we have $(1 - \epsilon) m_{s_i} p_k >> (1 + \delta) \mu_{\ell, i}$ for increasing $i$. Hence, the probability of $\deg_{G_{n_i}} v_\ell \geq (1 - \epsilon) m_{s_i} p_k $ decreases for a fixed $\epsilon >0$. Thus, for $\delta \leq (1 - \epsilon) m_{s_i} p_k/\mu_{\ell, i} - 1 $ we obtain
\ifthenelse{\boolean{twocolumn}}{
  % code for two column
  \begin{align*}
       P & \left( \deg_{G_{n_i}} v_\ell \geq  (1 - \epsilon)m_{s_i} p_k \right) \\
       & \leq \exp\left( - \frac{ \mu_{\ell, i}}{3}\left( \frac{(1 - \epsilon)m_{s_i} p_k}{ \mu_{\ell, i}} - 1 \right)^2\right)  
  \end{align*}
  \[ 
 \]
}{
  % code for single column  
  \[ P\left( \deg_{G_{n_i}} v_\ell \geq  (1 - \epsilon)m_{s_i} p_k \right) \leq \exp\left( - \frac{ \mu_{\ell, i}}{3}\left( \frac{(1 - \epsilon)m_{s_i} p_k}{ \mu_{\ell, i}} - 1 \right)^2\right)  
 \]
}
As $\mu_{\ell, i} \in \mathcal{O}(n_{d_i})$, $m_{s_i} \in \mathcal{\Theta}(n_{s_i})$ and  $n_{s_i}/n_{d_i}$ goes to infinity for sparse graphs the expression $ \left( \frac{(1 - \epsilon)m_{s_i} p_k}{ \mu_{\ell, i}} - 1 \right)^2 $ goes to infinity. Therefore, the probability $P\left( \deg_{G_{n_i}} v_\ell \geq  (1 - \epsilon)m_{s_i} p_k \right)$ goes to zero satisfying %Therefore, the probability that a node in the dense part has higher degree than the expected degree of a star node in $G_{n_i}$ goes to zero. 
\[ 
P\left( \deg_{G_{n_i}} v_\ell \geq  (1 - \epsilon)m_{s_i} p_k \right) \leq  \exp \left( -C \frac{m_{s_i}^2}{n_{d_i}^2}\right) \, . 
\]
\end{proof}

\begin{lemma}\label{lemma:WUrandomgraphs6}
    Let $\{G_{n_i}\}_i$ be a sequence of sparse $(U,W)$-mixture graphs (Definition \ref{def:WURandomMixtureGraphs}) with dense and sparse parts $G_{d_i}$ and $G_{s_i}$ respectively. Let $\bm{p} = (p_1, p_2, \ldots )$ be the mass-partition (Definition \ref{def:masspartition}) associated with $U$. Let $\tilde{q}_{j,i}$ be the degree of the star in $G_{s_i}$ corresponding to $p_j \neq 0$. Let $q_{j,i}$ denote the degree of the corresponding vertex in $G_{n_i}$. Let $v_\ell$ be a vertex originally from the dense part $G_{d_i}$ and let us denote the degree of $v_\ell$ in $G_{d_i}$ by  $\deg_{G_{d_i}} v_\ell$  and its degree in $G_{n_i}$ by $\deg_{G_{n_i}} v_\ell$.  Then 
% \[  \lim_{i \to \infty} P\left( \deg_{G_{n_i}} v_\ell \geq  q_{j, i} \right) = 0 \, . 
% \]
\[
 P\left( \deg_{G_{n_i}} v_\ell \geq  q_{j, i} \right) \leq \exp \left( -c \frac{m_{s_i}^2}{n_{d_i}^2} \right) + \exp (-c' m_{s_i})
\]
That is, the probability of a node originally in the dense part having a larger degree than that of a star node in $G_{n_i}$ goes to zero. 
\end{lemma}
\begin{proof}
    Let $G_{d_i}, G_{s_i}$ and $G_{n_i}$ have $n_i$, $n_{d_i}$ and $n_{s_i}$ number of nodes and $m_i$, $m_{d_i}$ and $m_{s_i}$ number of edges respectively. Then for $\epsilon >0 $
    \ifthenelse{\boolean{twocolumn}}{
      % code for two column
      \begin{align}\label{eq:lemmaWU6eq1}
         P& \left( \deg_{G_{n_i}} v_\ell \geq  q_{j, i} \right) \\
         & \quad  = P\left( \deg_{G_{n_i}} v_\ell \geq  q_{j, i} | q_{j, i} \geq (1 - \epsilon) m_{s_i} p_j \right) \times \\
         & \qquad P\left( q_{j, i} \geq (1 - \epsilon) m_{s_i} p_j  \right)  \notag \\
        & \quad + P\left( \deg_{G_{n_i}} v_\ell \geq  q_{j, i} | q_{j, i} \leq (1 - \epsilon) m_{s_i} p_j \right) \times \\
        & \qquad P\left( q_{j, i} \leq (1 - \epsilon) m_{s_i} p_j  \right)  \, , \notag \\
        & \leq P\left( \deg_{G_{n_i}} v_\ell \geq  (1 - \epsilon) m_{s_i} p_j  \right) \times \\
        & \qquad P\left( q_{j, i} \geq (1 - \epsilon) m_{s_i} p_j  \right)   \notag \\
        & \quad + P\left( \deg_{G_{n_i}} v_\ell \geq  q_{j, i} | q_{j, i} \leq (1 - \epsilon) m_{s_i} p_j \right) \times \\
        & \qquad P\left( q_{j, i} \leq (1 - \epsilon) m_{s_i} p_j  \right)  \, , \notag \\ 
        & \leq P\left( \deg_{G_{n_i}} v_\ell \geq  (1 - \epsilon) m_{s_i} p_j  \right) +  \\
        & \qquad P\left( q_{j, i} \leq (1 - \epsilon) m_{s_i} p_j  \right) 
    \end{align}
    }{
      % code for single column  
      \begin{align}\label{eq:lemmaWU6eq1}
         P\left( \deg_{G_{n_i}} v_\ell \geq  q_{j, i} \right)  & = P\left( \deg_{G_{n_i}} v_\ell \geq  q_{j, i} | q_{j, i} \geq (1 - \epsilon) m_{s_i} p_j \right) P\left( q_{j, i} \geq (1 - \epsilon) m_{s_i} p_j  \right)  \notag \\
        & \quad + P\left( \deg_{G_{n_i}} v_\ell \geq  q_{j, i} | q_{j, i} \leq (1 - \epsilon) m_{s_i} p_j \right) P\left( q_{j, i} \leq (1 - \epsilon) m_{s_i} p_j  \right)  \, , \notag \\
        & \leq P\left( \deg_{G_{n_i}} v_\ell \geq  (1 - \epsilon) m_{s_i} p_j  \right) P\left( q_{j, i} \geq (1 - \epsilon) m_{s_i} p_j  \right)   \notag \\
        & \quad + P\left( \deg_{G_{n_i}} v_\ell \geq  q_{j, i} | q_{j, i} \leq (1 - \epsilon) m_{s_i} p_j \right) P\left( q_{j, i} \leq (1 - \epsilon) m_{s_i} p_j  \right)  \, , \notag \\ 
        & \leq P\left( \deg_{G_{n_i}} v_\ell \geq  (1 - \epsilon) m_{s_i} p_j  \right) +  P\left( q_{j, i} \leq (1 - \epsilon) m_{s_i} p_j  \right) 
    \end{align}
    }
    
From Lemma \ref{lemma:WUrandomgraphs4} we know that 
\[
P\left( \deg_{G_{n_i}} v_\ell \geq  (1 - \epsilon) m_{s_i} p_j  \right) \leq \exp \left( -C \frac{m_{s_i}^2}{n_{d_i}^2}\right) \, .
\]

From Lemma \ref{lemma:WUrandomgraphsaboutU2} we know
\[  \mathbb{E}({q}_{j, i})  = m_{s_i} p_j + \frac{c m_{new_i}}{n_{s_i}} \, . 
\]
Using Chernoff bounds we get
\[ P\left( q_{j, i} \leq (1 - \epsilon) m_{s_i} p_j  \right) \leq \exp\left( -\frac{m_{s_i} p_j \epsilon^2}{2} \ \right) \, .
\]
Substituting in equation \eqref{eq:lemmaWU6eq1} we get
\[
P\left( \deg_{G_{n_i}} v_\ell \geq  q_{j, i} \right) \leq \exp \left( -C \frac{m_{s_i}^2}{n_{d_i}^2}\right) + \exp \left( -c' m_{s_i}\right) \, .
\]
% As $i$ goes to infinity $m_{s_i}$ goes to infinity making this probability converge to zero. Therefore both probabilities in equation \eqref{eq:lemmaWU6eq1} tend to zero making 
% \[  \lim_{i \to \infty} P\left( \deg_{G_{n_i}} v_\ell \geq  q_{j, i} \right) = 0 \, . 
% \]
\end{proof}

\proprankpreserving*
\begin{proof}
    Let us explore the condition $q_{j, i} = \deg_{G_{n_i}} v_{(j)}$ for $j \in [k]$. This can only happen when $q_{1,i} > q_{2, i} > \cdots > q_{k, i}$ and when all other nodes have degrees lower than $q_{k, i}$. Let us denote the maximum degree of the other nodes by $\deg_{G_{n_i}} v_u $.
    The probability
    \ifthenelse{\boolean{twocolumn}}{
      % code for two column
      \begin{align}\label{eq:bigprobability}
        P\left( \bigcap_{j = 1}^k q_{j, i} = \deg_{G_{n_i}} v_{(j)} \right) & = P\left(q_{1,i} > q_{2, i} > \cdots > q_{k, i} \right) \times  \notag  \\
         & \quad P\left( \deg_{G_{n_i}} v_u  < q_{k,i} \right) \, ,  \notag \\
        & = \prod_{j= 1}^{k-1} P\left(q_{j,i} > q_{j+1, i} \right) \times  \notag \\ 
        & \quad P\left( \deg_{G_{n_i}} v_u  < q_{k, i} \right) \, . 
    \end{align}
    }{
      % code for single column  
     \begin{align}\label{eq:bigprobability}
        P\left( \bigcap_{j = 1}^k q_{j, i} = \deg_{G_{n_i}} v_{(j)} \right) & = P\left(q_{1,i} > q_{2, i} > \cdots > q_{k, i} \right) \times P\left( \deg_{G_{n_i}} v_u  < q_{k,i} \right) \, ,  \notag \\
        & = \prod_{j= 1}^{k-1} P\left(q_{j,i} > q_{j+1, i} \right) \times P\left( \deg_{G_{n_i}} v_u  < q_{k, i} \right) \, . 
    \end{align}
    }

%as sampling graphs from $U$ and $W$ 
From Lemma \ref{lemma:WUrandomgraphs5} we know that 
\[
P\left(q_{j+1,i} > q_{j, i} \right) \leq \frac{C}{m_{s_i}} \, ,
\]
% \[  \lim_{i \to \infty} P\left(q_{j+1,i} > q_{j, i} \right) = 0
% \]
making 
\begin{equation}\label{eq:prop1eq1}
  P\left(q_{j,i} > q_{j+1, i} \right) \geq 1 - \frac{C}{m_{s_i}} \, . 
\end{equation}

% \[  \lim_{i \to \infty} P\left(q_{j,i} > q_{j+1, i} \right) = 1 \, . 
% \]   
Let $v_\ell$ be a vertex originally from the dense part $G_{d_i}$ and let us denote the degree of $v_\ell$ in $G_{d_i}$ by  $\deg_{G_{d_i}} v_\ell$  and its degree in $G_{n_i}$ by $\deg_{G_{n_i}} v_\ell$. From Lemma \ref{lemma:WUrandomgraphs6} we know that for vertices in the dense part
\begin{equation} %\label{eq:prop1eq1}
P\left( \deg_{G_{n_i}} v_\ell \geq  q_{j, i} \right) \leq \exp \left( -C \frac{m_{s_i}^2}{n_{d_i}^2}\right) + \exp \left( -c' m_{s_i}\right) \, .
% \lim_{i \to \infty} P\left( \deg_{G_{n_i}} v_\ell \geq  q_{j, i} \right) = 0 \, ,
\end{equation}
making
\begin{equation} \label{eq:prop1eq2}
P\left( \deg_{G_{n_i}} v_\ell <  q_{j, i} \right) \geq 1 - \exp \left( -C \frac{m_{s_i}^2}{n_{d_i}^2}\right) - \exp \left( -c' m_{s_i}\right) \, .
    % \lim_{i \to \infty} P\left( \deg_{G_{n_i}} v_\ell <  q_{j, i} \right) = 1 \, .
\end{equation}
Suppose $U$ has $k$ partitions. Then the maximum degree of the ``other'' nodes is obtained by a vertex in the dense part. 
In addition to the vertices in the dense part, there are degree-1 vertices in the sparse part. These are the leaf or non-hub nodes in every star. Their degrees are much smaller than the rest.  %Similar to the computation for a vertex in the dense part, the degree-1 vertices in the sparse part also satisfy equation \eqref{eq:prop1eq2}. 
Thus, equation \eqref{eq:prop1eq2} holds for all ``other'' vertices including the vertex which has the maximum degree of these vertices. Combining equations \eqref{eq:prop1eq1} and \eqref{eq:prop1eq2} with equation \eqref{eq:bigprobability} we obtain 
\begin{align}
 P\left( \bigcap_{j = 1}^k \left( q_{j, i} = \deg_{G_{n_i}} v_{(j)} \right) \right)  & \geq \left( 1 - \frac{c_1}{m_{s_i}}\right)^{k-1}\left( 1 - \exp\left(-c_2 \frac{m_{s_i}^2}{n_{d_i}^2} \right)  - \exp\left( -c_3m_{s_i}  \right) \right)   \, ,   \\
 & > \left( 1 - \frac{c_1}{m_{s_i}}\right)^{k}\left( 1 - \exp\left(-c_2 \frac{m_{s_i}^2}{n_{d_i}^2} \right)  - \exp\left( -c_3m_{s_i}  \right) \right) \, , 
\end{align}
where $c_1$ denotes the largest $C$ in equation \eqref{eq:prop1eq1} for different $j \leq k-1$. % when $U$ has $k$ partitions. 

Suppose $U$ has either infinite partitions or $\ell$ finite partitions with $\ell > k$. Then the probability  
 \begin{align}% \label{eq:bigprobability2}
        P\left( \bigcap_{j = 1}^k q_{j, i} = \deg_{G_{n_i}} v_{(j)} \right) & = P\left(q_{1,i} > q_{2, i} > \cdots > q_{k, i} \right) \times P\left(q_{k,i} > \max_{j \in \{k+1, \ldots  \}} q_{j,i} \right) \times P\left( \deg_{G_{n_i}} v_u  < q_{k,i} \right) \, ,  \notag 
        % & = \prod_{j= 1}^{k-1} P\left(q_{j,i} > q_{j+1, i} \right) \times P\left( \deg_{G_{n_i}} v_u  < q_{k, i} \right) \, . 
    \end{align}
where $\max_{j \in \{k+1, \ldots  \}} q_{j,i}$ denotes the maximum degree of other hub nodes contributed by $U$.  Similar to equation \eqref{eq:prop1eq1} the probability
\[
P\left(q_{k,i} > \max_{j \in \{k+1, \ldots  \}} q_{j,i} \right)  \geq 1 - \frac{C}{m_{s_i}}
\]
making 
\[
 P\left( \bigcap_{j = 1}^k \left( q_{j, i} = \deg_{G_{n_i}} v_{(j)} \right) \right)   \geq \left( 1 - \frac{c_1}{m_{s_i}}\right)^{k}\left( 1 - \exp\left(-c_2 \frac{m_{s_i}^2}{n_{d_i}^2} \right)  - \exp\left( -c_3m_{s_i}  \right) \right)  \, . 
\]
\end{proof}

\lemmaDegreePredict*
\begin{proof}
    For large $i$ and $j$ with high probability the expression
    \begin{align}
         \frac{ \left| \deg_{G_{n_j}} \hat{v}_{(\ell)} -  \mathbb{E}( q_{\ell,j})  \right|}{m_{s_j}}  & = \frac{1}{m_{s_j}} \left \vert \deg_{G_{n_i}} v_{(\ell)} \frac{n_j}{n_i}  - \left( m_{s_j} p_{\ell} + \frac{c m_{new_j}}{n_{s_j}} \right)  \right\vert  \, , \\
         & = \frac{1}{m_{s_j}} \left \vert q_{\ell, i} \frac{n_j}{n_i}  - \left( m_{s_j} p_{\ell} + \frac{c m_{new_j}}{n_{s_j}} \right)  \right\vert  \, , \\
    \end{align}
    where we have used equation \eqref{eq:degreepred} and Lemma \ref{lemma:WUrandomgraphsaboutU2} for $\mathbb{E}( q_{\ell,j})$ in the first line and the Order Preserving Property to say $\deg_{G_{n_i}} v_{(\ell)} = q_{\ell,i}$  in the second line for large $i$. 
    From Lemma \ref{lemma:WUrandomgraphsaboutU2} we can see that 
    \[
    \lim_{i \to \infty }\frac{\Var(q_{\ell,i})}{\mathbb{E}(q_{\ell,i})^2} = 0 \, , 
    \]
    implying that the observed values of $q_{\ell,i}$ lie close to its expected value for increasing $i$.  Therefore,
    % \begin{align}\label{eq:limitratio}
    %     \lim_{i, j \to \infty} \frac{ \left| \deg_{G_{n_j}} \hat{v}_{(\ell)} -  \mathbb{E}( q_{\ell,j})  \right|}{m_{s_j}} & = \lim_{i, j \to \infty} \frac{1}{m_{s_j}} \left \vert \left(m_{s_i} p_{\ell} + \frac{c m_{new_i}}{n_{s_i}} \right)  \frac{n_j}{n_i}  - \left( m_{s_j} p_{\ell} + \frac{c m_{new_j}}{n_{s_j}} \right)  \right\vert  \, , \notag \\
    %     & = \lim_{i, j \to \infty} \frac{n_j}{m_{s_j}} \left \vert \left( \frac{m_{s_i} p_{\ell}}{n_i} + \frac{c m_{new_i}}{n_{s_i} n_i} \right)    - \left( \frac{m_{s_j} p_{\ell}}{n_j} + \frac{c m_{new_j}}{n_{s_j} n_j} \right)  \right\vert  \, , \notag \\ 
    %     & = \lim_{i, j \to \infty} \frac{n_j}{m_{s_j}} \left \vert  \frac{m_{s_i} p_{\ell}}{n_i}     -  \frac{m_{s_j} p_{\ell}}{n_j}   \right\vert  \, , \notag \\ 
    %     & = \lim_{i, j \to \infty} \frac{n_j p_{\ell}}{m_{s_j}} \left \vert  \frac{m_{s_i} }{n_i}     -  \frac{m_{s_j}}{n_j}    \right\vert  \, , 
    % \end{align}

    \begin{align}\label{eq:limitratio}
         \frac{ \left| \deg_{G_{n_j}} \hat{v}_{(\ell)} -  \mathbb{E}( q_{\ell,j})  \right|}{m_{s_j}} & \approx \frac{1}{m_{s_j}} \left \vert \left(m_{s_i} p_{\ell} + \frac{c m_{new_i}}{n_{s_i}} \right)  \frac{n_j}{n_i}  - \left( m_{s_j} p_{\ell} + \frac{c m_{new_j}}{n_{s_j}} \right)  \right\vert  \, , \notag \\
        & =  \frac{n_j}{m_{s_j}} \left \vert \left( \frac{m_{s_i} p_{\ell}}{n_i} + \frac{c m_{new_i}}{n_{s_i} n_i} \right)    - \left( \frac{m_{s_j} p_{\ell}}{n_j} + \frac{c m_{new_j}}{n_{s_j} n_j} \right)  \right\vert  \, , \notag \\ 
        & \leq  c_1 \frac{n_j}{m_{s_j}} \left \vert  \frac{m_{s_i} p_{\ell}}{n_i}     -  \frac{m_{s_j} p_{\ell}}{n_j}   \right\vert  \, , \notag \\ 
        & = c_1 \frac{n_j p_{\ell}}{m_{s_j}} \left \vert  \frac{m_{s_i} }{n_i}     -  \frac{m_{s_j}}{n_j} \right\vert  \, ,  \\
        & = c_1p_{\ell} \left \vert 1 - \frac{n_jm_{s_i}}{ n_i m_{s_j}}  \right\vert  \, . 
    \end{align}

where we have used $\frac{c m_{new_i}}{n_{s_i} n_i}$ goes to zero. This is because $m_{new_i} \in \Theta(n_{d_i}^2)$ as a result of the joining process (Definition \ref{def:joiningrules}) and $n_{s_i}/n_{i}$ goes to 1  as $n_{d_i}/n_{s_i}$ goes to zero for sparse graphs. 

The number of nodes and edges in the sparse part are linked by 
\[
m_{s_i} + k_i = n_{s_i}
\]
where $k_i$ denotes the number of stars in the sparse part $G_{s_i}$ and $n_{s_i}$ denote the number of nodes. If $U$ has finite partitions, then $k_i$ is bounded for all large $i$. However, if $U$ has infinite partitions then the number of stars observed $k_i$ will increase with increasing $i$. Nevertheless, $k_i/n_{s_i}$ goes to zero. This is because the number of nodes in a collection of stars grows much faster than the number of stars when we uniformly sample from the disjoint clique graphon $U$. Thus we have
\begin{equation}\label{eq:msiovernsi}
    \lim_{i \to \infty} \frac{m_{s_i}}{n_{s_i}} = 1 \, . 
\end{equation}
Let $n_{d_i}$  denote the number of nodes in the dense part. Then
\begin{align}
    n_{s_i} + n_{d_i} & = n_i \, , \\
    \frac{n_{s_i}}{n_i} + \frac{n_{d_i}}{n_i} & = 1 \, . 
 \end{align}
As $n_{s_i}/n_{d_i}$ goes to infinity for sparse graphs we have $n_{d_i}/n_i$ going to zero making $n_{s_i}/n_i$ tending to one. Combining with \eqref{eq:msiovernsi} we get
\[
\lim_{i \to \infty} \frac{m_{s_i}}{n_i} = 1 \, . 
\]
Substituting this in equation \eqref{eq:limitratio} , we get 
\[
\lim_{i, j \to \infty} \frac{ \left| \deg_{G_{n_j}} \hat{v}_{(\ell)} -  \mathbb{E}( q_{\ell,j})  \right|}{m_{s_j}}  = 0 \, . 
\]
\end{proof}

%% file: 10_D_Finite_U.tex
\section{Estimating $U$ for finite partitions}

Using the Taylor series expansion, we derive the second order approximation of the expectation of a ratio of random variables and the first order approximation for the variance. 

\begin{lemma}\label{lemma:expetationofRVratio}(\cite{stuart2010kendall})
    Let $X$ and $Y$ be two random variables where $Y$ has no mass at zero.   Then the second order Taylor series approximation of  $\mathbb{E}(X/Y)$ is given by
    \[ \mathbb{E}\left( \frac{X}{Y} \right ) =   \frac{\mathbb{E}(X)}{\mathbb{E}(Y)} - \frac{\Cov(X,Y)}{\mathbb{E}(Y)^2} +  \frac{\mathbb{E}(X) \Var(Y)}{\mathbb{E}(Y)^3}    \, , 
    \]
    and the first order Taylor approximation of $\Var(X/Y)$ is given by
    \[
    \Var\left(\frac{X}{Y} \right) = \frac{ \mathbb{E}(X)^2}{\mathbb{E}(Y)^2} \left( \frac{\Var(X)}{ \mathbb{E}(X)^2 } - 2\frac{ \Cov(X,Y)}{ \mathbb{E}(X) \mathbb{E}(Y) } + \frac{\Var(Y)}{ \mathbb{E}(Y)^2} \right) \, . 
    \]
\end{lemma}
\begin{proof}
For any $f(x, y)$, the second order Taylor expansion about $\bm{a} = (a_x, a_y)$ is given by
    \begin{align}
        f(x,y) & = f(\bm{a}) + f'_x(\bm{a})(x - a_x) +  f'_y(\bm{a})(y - a_y) \\
        & + \frac{1}{2} f''_{xx}(\bm{a})(x - a_x)^2 + f''_{xy}(\bm{a})(x - a_x)(y - a_y) + \frac{1}{2} f''_{yy}(\bm{a})(y - a_y)^2  + R \, ,
    \end{align}
Let $\mathbb{E}(X) = \mu_x$,\,  $\mathbb{E}(Y) = \mu_y$, \,  $f(X,Y) = X/Y$ and $\bm{a} = (\mu_x, \mu_y)$. Then the expectation of the second order approximation is given by
\begin{align}\label{eq:secondorderexp}
    \mathbb{E}(f(X,Y)) & = \mathbb{E}\left( f(\bm{a}) + f'_x(\bm{a})(X - a_x) +  f'_y(\bm{a})(Y - a_y)  \right) \, ,  \notag \\
     & \quad + \mathbb{E}\left( \frac{1}{2} f''_{xx}(\bm{a})(X - a_x)^2 + f''_{xy}(\bm{a})(X - a_x)(Y - a_y) + \frac{1}{2} f''_{yy}(\bm{a})(Y - a_y)^2  \right) \, , \notag \\
     & = f(\bm{a}) + f'_x(\bm{a}) \mathbb{E}(X - \mu_x) +  f'_y(\bm{a})\mathbb{E}(Y - \mu_y)  \notag\\
     & \quad + \frac{1}{2} f''_{xx}\Var(X) + f''_{xy}(\bm{a})\Cov(X,Y) + \frac{1}{2} f''_{yy}\Var(Y) \, , \notag\\
     & = f(\bm{a}) + \frac{1}{2} f''_{xx}(\bm{a})\Var(X) + f''_{xy}(\bm{a})\Cov(X,Y) + \frac{1}{2} f''_{yy}(\bm{a})\Var(Y) 
\end{align}
as $\mathbb{E}(X - \mu_x) = \mathbb{E}(X - \mu_y) = 0$. As $f''_{xx} = 0$, $f''_{xy} = -\frac{1}{y^2}$ and $f''_{yy} = \frac{2x}{y^3}$ we obtain
    \[ \mathbb{E}\left( \frac{X}{Y} \right ) =   \frac{\mathbb{E}(X)}{\mathbb{E}(Y)} - \frac{\Cov(X,Y)}{\mathbb{E}(Y)^2} +  \frac{\mathbb{E}(X) \Var(Y)}{\mathbb{E}(Y)^3}    \, . 
    \]

  The variance is defined as
  \[
  \Var(f(X,Y)) = \mathbb{E}\left( \left[f(X,Y) - \mathbb{E} (f(X,Y))  \right]^2 \right)
  \]
  From equation \eqref{eq:secondorderexp} we let 
  \[
  \mathbb{E}(f(X,Y)) \approx f(\bm{a})
  \]
  and get
  \[
  \Var(f(X,Y)) \approx \mathbb{E}\left( \left[f(X,Y) - f(\bm{a})  \right]^2 \right) \, . 
  \]
  Then we use the first order Taylor expansion for $f(X,Y)$ around $\bm{a}$ inside the expectation term
  \begin{align}
       \Var(f(X,Y)) & \approx \mathbb{E}\left( \left[ f(\bm{a}) + f'_x(\bm{a})(x - a_x) +  f'_y(\bm{a})(y - a_y) - f(\bm{a})   \right]^2 \right) \, , \\
       & = \mathbb{E}\left( \left[  f'_x(\bm{a})(X - a_x) +  f'_y(\bm{a})(Y - a_y)    \right]^2 \right) \, , \\
       & = \mathbb{E}\left(  f'^2_x(\bm{a})(X - a_x)^2 + 2  f'_x(\bm{a})(X - a_x)f'_y(\bm{a})(Y - a_y) +     f'^2_y(\bm{a})(Y - a_y)^2      \right) \, , \\
       & = f'^2_x(\bm{a}) \Var(X) + 2  f'_x(\bm{a}) f'_y(\bm{a}) \Cov(X,Y) +  f'^2_y(\bm{a}) \Var(Y) \, . 
\end{align}
For $f(X,Y) = X/Y$ we have $f'(X) = 1/Y$, $f'(Y) = -\frac{X}{Y^2}$,  $f'^2_x(\bm{a}) = 1/\mu_y^2$,  $f'^2_y(\bm{a}) = \frac{\mu_x^2}{\mu_y^4}$ and $ f'_x(\bm{a}) f'_y(\bm{a})  = -\frac{\mu_x}{\mu_y^3}$. Substituting these we obtain
\begin{align}
    \Var(f(X,Y)) & \approx  \frac{1}{\mu_y^2} \Var(X) -2 \frac{\mu_x}{\mu_y^3}\Cov(X,Y) +  \frac{\mu_x^2}{\mu_y^4} \Var(Y) \, , \\
    & = \frac{\mu_x^2}{\mu_y^2} \left( \frac{\Var(X)}{ \mu_x^2} - 2\frac{\Cov(X,Y)}{\mu_x \mu_y} + \frac{\Var(Y)}{\mu_y^2}  \right) \, , \\
    & = \frac{\mathbb{E}(X)^2}{\mathbb{E}(Y)^2} \left( \frac{\Var(X)}{ \mathbb{E}(X)^2} - 2\frac{\Cov(X,Y)}{\mathbb{E}(X) \mathbb{E}(Y)} + \frac{\Var(Y)}{\mathbb{E}(Y)^2}  \right) \, .
\end{align}

\end{proof}

\subsection{Estimating $p_j$ }
\proppj*
\begin{proof}
    As $\tilde{q}_{j,i}$ is the degree of the star in $G_{s_i}$ corresponding to $p_j \neq 0$ and as $q_{j,i}$ denotes the degree of the corresponding vertex in $G_{n_i}$ we have 
    \[ q_{j, i} = \tilde{q}_{j,i} + m_{j, new_i} \,  
    \]
    where $m_{j, new_i}$ denotes the number of edges added to that vertex as part of the joining process. Then
    \[ \sum_{\ell = 1}^k q_{\ell, i} = \sum_{\ell = 1}^k \left( \tilde{q}_{\ell,i} + m_{\ell, new_i} \right) = m_{s_i} + \sum_{\ell = 1}^k m_{\ell, new_i} \, , 
    \]
    as the sum of the edges in the $k$ stars equal $m_{s_i}$. This is because of two reasons: First we consider a mass-partition $\bm{p}$ with $k$ non-zero entries with $\sum_{\ell=1}^k p_\ell = 1$. That is $G_{s_i}$ has only $k$ stars.  The second reason comes from the construction of  $(U,W)$-mixture graphs (Definition \ref{def:WURandomMixtureGraphs}). Recall we sample $m_{s_i}$ nodes from $U$ and construct a graph, which is a disjoint clique graph. Then we find its inverse line graph, which we call $G_{s_i}$. Consequently  $G_{s_i}$ has $m_{s_i}$ edges, which are spread across $k$ stars.  

    The number of edges $m_{s_i}$, $m_{new_i}$ and the number of nodes ${n_{s_i}}$ are not random variables. They are parameters of the $(U,W)$-mixture graph process. The degree of each node is a random variable as sampling from the graphons are involved in determining the degree.  Let $X_i = q_{j, i}$ and $Y_i = \sum_{\ell=1}^k q_{\ell, i}$.
    
    Then from Lemma \ref{lemma:WUrandomgraphsaboutU2}
    \begin{align}
        \mathbb{E}(X_i) & = m_{s_i}p_j + \frac{c m_{new_i}}{n_{s_i}} \, ,  \\
        \mathbb{E}(Y_i) & = m_{s_i} + \frac{k c m_{new_i}}{n_{s_i}} \, , \\
        \Var(X_i) & = m_{s_i} p_j(1 - p_j)  + m_{new_i}\frac{c}{n_{s_i}}\left( 1 - \frac{c}{n_{s_i}} \right) \, , \\
        \Var(Y_i) & = \Var \left( \sum_{\ell = 1}^k m_{j, new_i} \right) \, , \\
        & = \sum_{\ell = 1}^k \Var \left(m_{j, new_i} \right)  \, , \\
        & = k  m_{new_i}\frac{c}{n_{s_i}}\left( 1 - \frac{c}{n_{s_i}} \right) \, . 
    \end{align}

    The covariance
    \begin{align}
        \Cov(X_i, Y_i) & = \Cov \left(q_{j, i}, \sum_{\ell=1}^k q_{\ell, i} \right) \, , \\
        & = \Cov \left(q_{j, i}, q_{j, i} + \sum_{\substack{\ell = 1 \\ \ell \neq j}}^k q_{\ell, i} \right) \, , \\
        & = \Var( q_{j,i} ) \, , \\
        & = m_{s_i} p_j(1 - p_j)  +  m_{new_i}\frac{c}{n_{s_i}}\left( 1 - \frac{c}{n_{s_i}} \right) \, . 
    \end{align}
       % & = \Cov \left(\tilde{q}_{j,i} + m_{j, new_i}, m_{s_i} + \sum_{\ell = 1}^k m_{\ell, new_i} \right) \, , \\
        % & = \Cov\left(\tilde{q}_{j,i}, m_{s_i} \right)  + \Cov\left(\tilde{q}_{j,i}, \sum_{\ell = 1}^k m_{\ell, new_i}  \right) + \Cov\left(  m_{j, new_i},   m_{s_i}\right) + \Cov\left(  m_{j, new_i}, \sum_{\ell = 1}^k m_{\ell, new_i}  \right) \, ,  \\
        % & = \Cov\left(  m_{j, new_i},  m_{j, new_i} + \sum_{ \substack{\ell = 1 \\ \ell \neq j}}^k m_{\ell, new_i}\right) \, , \\
        % & = \Var\left(m_{j, new_i}, m_{j, new_i} \right) \, , \\
    
    From Lemma \ref{lemma:expetationofRVratio} using the second order Taylor series approximation we have
    \begin{equation}\label{eq:expectationofaratio}
        \mathbb{E}\left( \frac{X_i}{Y_i} \right ) =   \frac{\mathbb{E}(X_i)}{\mathbb{E}(Y_i)} - \frac{\Cov(X_i,Y_i)}{\mathbb{E}(Y_i)^2} +  \frac{\mathbb{E}(X_i) \Var(Y_i)}{\mathbb{E}(Y_i)^3}    \, .
    \end{equation}

    We compute each term separately and let $i$ go to infinity.  The first term
    \begin{align}\label{eq:firstterm}
        \frac{\mathbb{E}(X_i)}{\mathbb{E}(Y_i)}  & = \frac{ m_{s_i}p_j + \frac{c m_{new_i}}{n_{s_i}} }{ m_{s_i} + \frac{k c m_{new_i}}{n_{s_i}} } \, , \notag \\
        & = \frac{ p_j + \frac{c m_{new_i}}{n_{s_i} m_{s_i} } }{ 1 + \frac{k c m_{new_i}}{n_{s_i} m_{s_i} }  } \, ,  \notag\\
        \lim_{i \to \infty } \frac{\mathbb{E}(X_i)}{\mathbb{E}(Y_i)} & = p_j \, , 
        \end{align}
    as $m_{new_i} \in \mathcal{O}(n_{d_i}^2)$, $m_{s_i} \in \Theta(n_{s_i})$ and  $n_{d_i} \in o(n_{s_i})$ the ratio $m_{new_i}/(n_{s_i}m_{s_i}) \to 0$. 
    The second term 
    \begin{align}\label{eq:secondterm}
        \frac{\Cov(X_i,Y_i)}{\mathbb{E}(Y_i)^2} & = \frac{ m_{s_i} p_j(1 - p_j)  +  m_{new_i}\frac{c}{n_{s_i}}\left( 1 - \frac{c}{n_{s_i}} \right) }{ \left( m_{s_i} + \frac{k c m_{new_i}}{n_{s_i}} \right)^2 } \, ,  \notag\\
        & = \frac{ m_{s_i} p_j(1 - p_j)  +  m_{new_i}\frac{c}{n_{s_i}}\left( 1 - \frac{c}{n_{s_i}} \right) }{ m_{s_i}^2 + 2 m_{s_i} \frac{k c m_{new_i}}{n_{s_i}} + \left( \frac{k c m_{new_i}}{n_{s_i}} \right)^2  } \, , \notag \\
        & =  \frac{\frac{1}{m_{s_i}} p_j(1 - p_j)  +  \frac{m_{new_i}}{m_{s_i}^2 } \frac{c}{n_{s_i}}\left( 1 - \frac{c}{n_{s_i}} \right) }{ 1 + 2  \frac{k c m_{new_i}}{ m_{s_i} n_{s_i}} + \left( \frac{k c m_{new_i}}{n_{s_i} m_{s_i}} \right)^2} \, ,  \notag\\
        \lim_{i \to \infty}  \frac{\Cov(X_i,Y_i)}{\mathbb{E}(Y_i)^2} & = 0 \, , 
    \end{align}
    as both $m_{new_i}/(n_{s_i}m_{s_i}) \to 0$ and $m_{new_i}/m_{s_i}^2 \to 0$.

    The third term 
    \begin{align}\label{eq:thirdterm}
        \frac{\mathbb{E}(X_i) \Var(Y_i)}{\mathbb{E}(Y_i)^3} & = \frac{ \left( m_{s_i}p_j + \frac{c m_{new_i}}{n_{s_i}} \right)  \left(  k  m_{new_i}\frac{c}{n_{s_i}}\left( 1 - \frac{c}{n_{s_i}} \right) \right)   }{ \left( m_{s_i} + \frac{k c m_{new_i}}{n_{s_i}} \right)^3 } \, ,  \notag\\
        & = \frac{ \left( \frac{p_j}{m_{s_i}} + \frac{c m_{new_i}}{n_{s_i}m_{s_i}} \right)  \left(  k  \frac{m_{new_i}}{m_{s_i}^2} \frac{c}{n_{s_i}}\left( 1 - \frac{c}{n_{s_i}} \right) \right)   }{ \left( 1 + \frac{k c m_{new_i}}{n_{s_i} m_{s_i} } \right)^3 } \, , \notag \\
        \lim_{i\to \infty } \frac{\mathbb{E}(X_i) \Var(Y_i)}{\mathbb{E}(Y_i)^3} & = 0 \, , 
    \end{align}
    as $p_j \leq 1$, $p_j/m_{s_i}$ goes to zero in addition to the terms $m_{new_i}/(n_{s_i}m_{s_i})$ and $m_{new_i}/m_{s_i}^2$. Therefore, using the second order Taylor series approximation in equation \eqref{eq:expectationofaratio} and equations \eqref{eq:firstterm}, \eqref{eq:secondterm} and \eqref{eq:thirdterm}  we obtain 
     \[\lim_{i \to \infty}  \mathbb{E}\left( \frac{X}{Y} \right ) =  \lim_{i \to \infty}  \mathbb{E}\left( \frac{q_{j, i}}{\sum_{\ell=1}^k q_{\ell, i} } \right) = p_j \, .
    \]   
    Furthermore from the above computations we have
   %  \[
   % \left\vert  \mathbb{E}\left( \frac{X}{Y} \right ) - p_j \right\vert \leq \frac{c}{m_{s_i}} \, , 
   %  \]
   %  and
    \[
    \frac{\Cov(X_i,Y_i)}{\mathbb{E}(Y_i)^2} \leq \frac{c}{m_{s_i}} \, , 
    \]
    and
    \[
     \frac{\mathbb{E}(X_i) \Var(Y_i)}{\mathbb{E}(Y_i)^3} \leq \frac{c}{m_{s_i}^2} \, , 
    \]
    where each $c$ denotes a different constant. Thus we have
    \[
    \left\vert \mathbb{E}\left( \frac{q_{j, i}}{\sum_{\ell=1}^k q_{\ell, i} } \right) - p_j  \right\vert \leq \frac{c}{m_{s_i}} \, . 
    \]
    From Lemma \ref{lemma:expetationofRVratio} we know

    \[
    \Var\left(\frac{X}{Y} \right) = \frac{ \mathbb{E}(X)^2}{\mathbb{E}(Y)^2} \left( \frac{\Var(X)}{ \mathbb{E}(X)^2 } - 2\frac{ \Cov(X,Y)}{ \mathbb{E}(X) \mathbb{E}(Y) } + \frac{\Var(Y)}{ \mathbb{E}(Y)^2} \right) \, . 
    \]
    We consider the three terms inside the paranthesis separately.  The first term 
    \begin{align}
        \frac{\Var(X_i)}{ \mathbb{E}(X_i)^2 }  & = \frac{m_{s_i} p_j(1 - p_j)  + m_{new_i}\frac{c}{n_{s_i}}\left( 1 - \frac{c}{n_{s_i}} \right) }{ \left( m_{s_i}p_j + \frac{c m_{new_i}}{n_{s_i}} \right)^2 } \, , \\
        & = \frac{\frac{1}{m_{s_i}} p_j(1 - p_j)  + \frac{m_{new_i}}{m_{s_i^2}}\frac{c}{n_{s_i}}\left( 1 - \frac{c}{n_{s_i}} \right) }{ \left( p_j + \frac{c m_{new_i}}{n_{s_i} m_{s_i}} \right)^2 } \, , \\
        \lim_{i \to \infty} \frac{\Var(X_i)}{ \mathbb{E}(X_i)^2 }  & = \lim_{i \to \infty} \frac{\frac{1}{m_{s_i}} p_j(1 - p_j)  + \frac{m_{new_i}}{m_{s_i^2}}\frac{c}{n_{s_i}}\left( 1 - \frac{c}{n_{s_i}} \right) }{ \left( p_j + \frac{c m_{new_i}}{n_{s_i} m_{s_i}} \right)^2 } \, , \\
        \lim_{i \to \infty} \frac{\Var(X_i)}{ \mathbb{E}(X_i)^2 }  & = 0 \, .
    \end{align}
    The second term 
    \begin{align}
        \frac{ \Cov(X_i,Y_i)}{ \mathbb{E}(X_i) \mathbb{E}(Y_i) } & = \frac{m_{s_i} p_j(1 - p_j)  +  m_{new_i}\frac{c}{n_{s_i}}\left( 1 - \frac{c}{n_{s_i}} \right)}{ \left(m_{s_i}p_j + \frac{c m_{new_i}}{n_{s_i}} \right) \left( m_{s_i} + \frac{k c m_{new_i}}{n_{s_i}}\right) } \, , \\
        & = \frac{ \frac{1}{m_{s_i}} p_j(1 - p_j)  +   \frac{m_{new_i}}{m_{s_i^2}} \frac{c}{n_{s_i}} \left( 1 - \frac{c}{n_{s_i}} \right)}{ \left(p_j + \frac{c m_{new_i}}{n_{s_i}m_{s_i}} \right) \left( 1 + \frac{k c m_{new_i}}{n_{s_i}m_{s_i}}\right) } \, , \\
        \lim_{i \to \infty} \frac{ \Cov(X_i,Y_i)}{ \mathbb{E}(X_i) \mathbb{E}(Y_i) } = 0 \, .
    \end{align}
    The term term is similar to the first term. Combining these three terms we get using the first order approximation
    \[
    \lim_{i \to \infty} \Var\left( \frac{X}{Y} \right) = 0 \, . 
    \]
    Similar to the computation of the expectation we get 
    \[
    \Var\left( \frac{X}{Y} \right) \leq \frac{c}{m_{s_i}} \, . 
    \]
 \end{proof}

\subsection{Estimating $k$}
\begin{proposition}\label{prop:krelated1}
 Let $\{G_{n_i}\}_i$ be a sequence of sparse $(U,W)$-mixture graphs (Definition \ref{def:WURandomMixtureGraphs}) with dense and sparse parts $G_{d_i}$ and $G_{s_i}$ respectively. Let 
 $\bm{p} = (p_1, p_2, \ldots )$ be the mass-partition (Definition \ref{def:masspartition}) associated with $U$ with $k$ finite partitions.%non-zero elements with $\sum_{j = 1}^k p_j = 1$. 
 Let $\tilde{q}_{j,i}$ be the degree of the star in $G_{s_i}$ corresponding to $p_j \neq 0$. Let $q_{j,i}$ denote the degree of the corresponding vertex in $G_{n_i}$. Let $\deg v_{(\ell)}$ denote the $\ell$th largest degree in $G_{n_i}$. Suppose $\deg {v}_{(\ell)} = q_{\ell, i}$ for $\ell \in \{1, \ldots, k\}$ and let $\alpha \geq 0$ be a constant. Then for any $\ell \in \{1, \ldots, (k -1) \}$
 % \[ \lim_{i \to \infty} P \left( (1 + \alpha) q_{\ell, i} n_{d_i} > q_{k, i} q_{\ell+1, i} \right) = 0 \, ,
 % \]
 \[
 P \left( (1 + \alpha) q_{\ell, i} n_{d_i} > q_{k, i} q_{\ell+1, i} \right) \leq \frac{c}{m_{s_i}}
 \]
 where $n_{d_i}$  denotes the number of nodes in the dense part. Consequently, with high probability
 \[ \log q_{k, i} - \log(1 + \alpha ) n_{d_i} \geq \log  q_{\ell, i} - \log q_{\ell +1, i} \, . 
 \]
 That is, in the log scale the difference between successive top $k$ degrees is largest at $k$, where we have considered the highest degree in the dense part to be $n_{d_i}$. 
\end{proposition}
\begin{proof}
   Let $X_i = (1 + \alpha ) q_{\ell, i} n_{d_i} - q_{k, i}q_{\ell +1, i}$. We show that $P(X_i >0 )$ goes to zero. From Lemma \ref{lemma:WUrandomgraphsaboutU2} we know
   \begin{align}
       \mathbb{E}(q_{\ell, i}) & = m_{s_i}p_\ell + \frac{c m_{new_i}}{n_{s_i}} = m_{s_i}p_\ell + c_{e_i}  \, ,  \\
        \Var(q_{\ell, i}) & = m_{s_i} p_\ell(1 - p_\ell)  + m_{new_i}\frac{c}{n_{s_i}}\left( 1 - \frac{c}{n_{s_i}} \right)  \\
        & = m_{s_i} p_\ell(1 - p_\ell) + c_{v_i} \, ,
   \end{align}
   where $c_{e_i} = \frac{c m_{new_i}}{n_{s_i}} $ and $c_{v_i} =  m_{new_i}\frac{c}{n_{s_i}}\left( 1 - \frac{c}{n_{s_i}} \right)$. Note that $\frac{c_{e_i}}{m_{s_i}}$ and $\frac{c_{v_i}}{m_{s_i}}$ goes to zero as $m_{new_i} \in \mathcal{O}(n_{d_i}^2)$ and as $n_{d_i}/n_{s_i}$ goes to zero. 
   Given $m_{s_i}$, the random variables $q_{k, i}$ and $q_{\ell, i}$ are independent for $k \neq \ell$. Hence
   \begin{align}\label{eq:expectedvalXi}
       \mathbb{E}(X_i) & = \mathbb{E}\left( (1 + \alpha ) q_{\ell, i} n_{d_i} - q_{k, i}q_{\ell +1, i} \right) \, ,  \notag \\
       & = (1 + \alpha )n_{d_i} \mathbb{E}\left( q_{\ell, i}\right) - \mathbb{E}\left( q_{k, i}  \right)  \mathbb{E}\left( q_{\ell +1, i}\right) \, ,   \notag \\
       & = (1 + \alpha )n_{d_i} \left( m_{s_i}p_\ell + c_{e_i} \right) - \left( m_{s_i}p_k + c_{e_i} \right) \left( m_{s_i}p_{\ell+1} + c_{e_i} \right) \, ,  \notag\\ 
       \frac{\mathbb{E}(X_i)}{m_{s_i}^2} & = \frac{(1 + \alpha )n_{d_i}}{m_{s_i}} \left( p_\ell + \frac{c_{e_i}}{m_{s_i}} \right) - \left( p_k + \frac{c_{e_i}}{m_{s_i}} \right) \left( p_{\ell+1} + \frac{c_{e_i}}{m_{s_i}} \right) \, ,  \notag \\
        \lim_{i \to \infty} \frac{\mathbb{E}(X_i)}{m_{s_i}^2} & = - p_k p_{\ell+1} \, , 
    \end{align}
    as $n_{d_i}/m_{s_i}$ goes to zero because $n_{d_i}/n_{s_i}$ goes to zero and $m_{s_i} \in \Theta(m_{s_i})$. 
    The variance of the product of two independent random variables $A$ and $B$ is given by
    \[ \Var(AB) = \Var(A) \Var(B) + \mathbb{E}(A)^2 \Var(B) + \mathbb{E}(B)^2 \Var(A) \, . 
    \]
    Using the product formula we get
    \begin{align}\label{eq:varianceXifirstterm}
        \Var(q_{k,i} q_{\ell +1, i)}) & =\Var(q_{k,i}) \Var(q_{\ell +1, i)}) + \mathbb{E}(q_{k,i})^2 \Var(q_{\ell +1, i)}) + \mathbb{E}(q_{\ell +1, i)})^2 \Var(q_{k,i})  \, , \notag \\
             &  =  \left( m_{s_i} p_k(1 - p_k) + c_{v_i}  \right)  \left( m_{s_i} p_{\ell+1}(1 - p_{\ell+1}) + c_{v_i}  \right) \, , \notag \\
             & \quad + \left( m_{s_i}p_k + c_{e_i} \right)^2\left( m_{s_i} p_{\ell+1}(1 - p_{\ell+1}) + c_{v_i}  \right) \, , \notag \\
             & \quad + \left( m_{s_i}p_{\ell +1} + c_{e_i} \right)^2\left( m_{s_i} p_{k}(1 - p_{k}) + c_{v_i}  \right) \, , \notag \\
         \frac{\Var(q_{k,i} q_{\ell +1, i)})}{m_{s_i}^3} & = \frac{1}{m_{s_i}}\left(  p_k(1 - p_k) + \frac{c_{v_i}}{m_{s_i}}  \right)  \left(  p_{\ell+1}(1 - p_{\ell+1}) + \frac{c_{v_i}}{m_{s_i}}  \right) \, , \notag  \\  
         & \quad + \left( p_k + \frac{c_{e_i}}{m_{s_i}} \right)^2\left(  p_{\ell+1}(1 - p_{\ell+1}) + \frac{c_{v_i}}{m_{s_i}}  \right) \, , \notag \\
           & \quad + \left( p_{\ell+1} + \frac{c_{e_i}}{m_{s_i}} \right)^2\left(  p_{k}(1 - p_{k}) + \frac{c_{v_i}}{m_{s_i}}  \right) \, , \notag \\
        \lim_{i \to \infty}  \frac{\Var(q_{k,i} q_{\ell +1, i)})}{m_{s_i}^3} & = p_k^2 p_{\ell+1}(1 - p_{\ell+1})  +  p_{\ell+1}^2p_{k}(1 - p_{k})
    \end{align}
    As $n_{d_i}$ is not a random variable  we get
    \begin{align}\label{eq:varianceXisecondterm}
        \Var((1 + \alpha )n_{d_i} q_{\ell, i}) & = (1 + \alpha )^2 n_{d_i}^2 \Var(q_{\ell, i}) \, , \\
        & =  (1 + \alpha )^2 n_{d_i}^2 \left( m_{s_i} p_\ell(1 - p_\ell) + c_{v_i}  \right)  \, , \notag \\
        \frac{ \Var(n_{d_i} q_{\ell, i})}{m_{s_i}^3} & =  \frac{(1 + \alpha )^2n_{d_i}^2}{m_{s_i}^2} \left(  p_\ell(1 - p_\ell) + \frac{ c_{v_i}}{m_{s_i}}  \right) \, , \notag \\
         \lim_{i \to \infty }\Var(n_{d_i} q_{\ell, i}) & = 0\, , 
    \end{align}
    as $n_{d_i}/m_{s_i}$ goes to zero. Combining terms for $X_i = (1 + \alpha )q_{\ell, i} n_{d_i} - q_{k, i}q_{\ell +1, i}$ in equations \eqref{eq:expectedvalXi}, \eqref{eq:varianceXifirstterm} and \eqref{eq:varianceXisecondterm} we get
    \begin{align}\label{eq:varexpratio}
        \Var (X_i) & = \Var\left( (1 + \alpha )q_{\ell, i} n_{d_i} \right) + \Var \left( q_{k, i}q_{\ell +1, i} \right) \, , \notag  \\
        \lim_{i \to \infty} \frac{\Var (X_i)}{m_{s_i}^3} & = p_k^2 p_{\ell+1}(1 - p_{\ell+1})  +  p_{\ell+1}^2p_{k}(1 - p_{k})\, , \notag  \\
        \lim_{i \to \infty} \frac{ \Var(X_i)}{\mathbb{E}(X_i)^2} & = \lim_{i \to \infty} \frac{ \Var(X_i)/{m_{s_i}^4} }{ \mathbb{E}(X_i)^2/{m_{s_i}^4} } \, , \notag  \\ 
        & = \lim_{i \to \infty} \frac{\frac{1}{m_{s_i}} \left( p_k^2 p_{\ell+1}(1 - p_{\ell+1})  +  p_{\ell+1}^2p_{k}(1 - p_{k}) \right) }{(p_k p_{\ell+1})^2} =  0 \, , 
    \end{align}
    as $\Var(X_i) \in \Theta(m_{s_i}^3)$ and $\mathbb{E}(X_i) \in \Theta(m_{s_i}^2)$. Furthermore
    \begin{equation}
     \frac{ \Var(X_i)}{\mathbb{E}(X_i)^2} \leq \frac{c}{m_{s_i}} \, , 
    \end{equation}
    for some $c$. 
    As shown in equation  \ref{eq:expectedvalXi} for large $i$, $\mathbb{E}(X_i) \leq 0$. We know that
    \begin{align}\label{eq:xpositiveprobability}
P(X_i >0  ) & = P(X_i - \mathbb{E}(X_i) \geq - \mathbb{E}(X_i)) \, , \notag \\
& = P(X_i - \mathbb{E}(X_i) \geq |\mathbb{E}(X_i)|) \quad \text{as } -\mathbb{E}(X_i) = | \mathbb{E}(X_i)| \, ,  \notag \\
& \leq P(|X_i - \mathbb{E}(X_i)| \geq |\mathbb{E}(X_i)|)  \, , \notag \\
& \leq \frac{\Var(X_i)}{\mathbb{E}(X_i)^2}
\end{align}
where we have used Chebyshev's inequality. Using equations \eqref{eq:xpositiveprobability} and \eqref{eq:varexpratio} we obtain
\begin{align}
     P \left( (1 + \alpha)q_{\ell, i} n_{d_i} > q_{k, i} q_{\ell+1, i} \right)  & = P(X_i >0  )  \, , \\
    & \leq  \frac{\Var(X_i)}{\mathbb{E}(X_i)^2} =  0 \, ,  \\
    & \leq \frac{c}{m_{s_i}}
\end{align}
giving us
\begin{align}
    \lim_{i \to \infty}  P \left( (1 + \alpha)q_{\ell, i} n_{d_i} > q_{k, i} q_{\ell+1, i} \right)   =  0 \, . 
\end{align}

%  giving us
% \[
% P \left( (1 + \alpha)q_{\ell, i} n_{d_i} > q_{k, i} q_{\ell+1, i} \right) \leq \frac{c}{m_{s_i}}
% \]
Consequently 
\[  \lim_{i \to \infty} P \left((1 + \alpha) q_{\ell, i} n_{d_i} \leq q_{k, i} q_{\ell+1, i} \right) = 1 \, . 
\]
Therefore, with high probability we have
\begin{align}
 q_{k, i} q_{\ell+1, i}    & \geq  (1 + \alpha) q_{\ell, i} n_{d_i}\, , \\
  \log \left( \frac{ q_{k, i} q_{\ell+1, i}  }{(1 + \alpha) q_{\ell, i} n_{d_i} } \right)  & \geq 0 \, , \\
  \log q_{k, i}  - \log (1 + \alpha)n_{d_i} & \geq \log q_{\ell, i} - \log q_{\ell+1, i} \, . 
\end{align}
That is, if we take successive differences of the top-$k$ degrees in the log scale,  with high probability the difference between the $k$th largest degree and $(1 + \alpha)n_{d_i}$ is larger than the other successive degree differences for $\ell < k$ where we have used the $(1 + \alpha) n_{d_i}$ in place of the highest degree in the dense part $G_{d_i}$.  
\end{proof}

\begin{lemma}\label{lemma:densehighestdegree}
     Let $\{G_{n_i}\}_i$ be a sequence of sparse $(U,W)$-mixture graphs (Definition \ref{def:WURandomMixtureGraphs}) with dense and sparse parts $G_{d_i}$ and $G_{s_i}$ respectively. Let $\text{max\_deg}_{\widetilde{G}_{d_i}}$ denote the maximum degree in $G_{n_i}$ restricted to nodes from $G_{d_i}$. That is, $\text{max\_deg}_{\widetilde{G}_{d_i}}$ belongs to a node from the dense part $G_{d_i}$. Then 
     \[
      \mathbb{E} \left( \text{max\_deg}_{\widetilde{G}_{d_i}} \right) \leq (1 + \alpha) n_{d_i} \, , 
     \]
   where $\alpha$ depends on $W$ and the joining mechanism. 
\end{lemma}
\begin{proof}
Suppose $v_{\ell}$ is originally a vertex in the dense part $G_{d_i}$ Then, from Lemma \ref{lemma:expnewedges}
   \begin{align}
          \mathbb{E}\left( \deg_{G_{n_i}} v_\ell \right)  & = \deg_{G_{d_i}} v_\ell  + c_1 \frac{m_{new_i}}{n_{d_i}} \, ,  \\
          & \leq  \deg_{G_{d_i}} v_\ell  + c' \frac{m_{d_i}}{n_{d_i}}\,  , \\
          & =  \deg_{G_{d_i}} v_\ell  + c'n_{d_i} \frac{m_{d_i}}{n_{d_i}^2}\,  , \\
          & \leq n_{d_i}  + c'n_{d_i} \rho_i\,  , \\
          & = (1 + c'\rho_i) n_{d_i}
   \end{align}
    where number of nodes in the dense part $n_{d_i}$ is a large upper bound for the largest degree in $G_{d_i}$.  The constant $c_1$ and $m_{new_i}$ depend on the joining mechanism (Definition \ref{def:joiningrules}) and $m_{new_i} \leq c m_{d_i}$. The density of $G_{d_i}$ is denoted by $\rho_i$, which is a bounded quantity and depends on $W$. 
\end{proof}

\begin{lemma}\label{lemma:krelated2}
 Let $\{G_{n_i}\}_i$ be a sequence of sparse $(U,W)$-mixture graphs (Definition \ref{def:WURandomMixtureGraphs}) with dense and sparse parts $G_{d_i}$ and $G_{s_i}$ respectively. Suppose $W$ satisfies Assumption \ref{assump:1}. 
 %has a continuous degree function $D(x)$ (Definition \ref{def:degreefunction}).
 We consider the degrees in $G_{d_i}$. Let $r_{x_j,i}$ and $r_{x_j + \epsilon_i,i}$ denote the degree of two nodes corresponding to $D(x_j)$ and $D(x_j + \epsilon_i)$ for $D(x_j), D(x_j + \epsilon_i) > \xi > 0$ for some $\xi >0$. Suppose $\lim_{i \to \infty} \epsilon_i = 0$. Then using the second order approximation of the Taylor expansion, we have
 % \[
 % \lim_{i \to \infty}\mathbb{E} \left( \frac{r_{x_j,i} }{r_{x_j + \epsilon_i,i}} \right) = 1
 % \]
 \[
 \left\vert \mathbb{E} \left( \frac{r_{x_j,i} }{r_{x_j + \epsilon_i,i}} \right) - 1 \right\vert \leq \frac{c}{n_{d_i}}
 \]
 That is, if $W$ has a continuous degree function $D(x)$, then as $i$ goes to infinity, excluding small degrees the ratio of consecutive degrees in $G_{d_i}$ tend to one.
\end{lemma}
 \begin{proof}
 \SK{Might want to put in the variance as well. }
    The degree function acts like the edge probability. Thus, the expectation and variance of $r_{x_j,i}$ are given by
    \begin{align}
        \mathbb{E}(r_{x_j,i}) & = (n_{d_i} - 1) D(x_j) \, , \\
        \Var(r_{x_j,i}) & = (n_{d_i} - 1) D(x_j)(1 - D(x_j)) \, .
    \end{align}
    Using Lemma \ref{lemma:expetationofRVratio} we have
    \[ \mathbb{E}\left( \frac{X}{Y} \right ) =   \frac{\mathbb{E}(X)}{\mathbb{E}(Y)} - \frac{\Cov(X,Y)}{\mathbb{E}(Y)^2} +  \frac{\mathbb{E}(X) \Var(Y)}{\mathbb{E}(Y)^3}    \, . 
    \]
    for two random variables $X$ and $Y$. Given $D(x_j)$ the degrees of the two nodes are independent because each edge is independently sampled from a $\text{Bernoulli}(p)$ for $p = D(x_j), D(x_j + \epsilon_i)$. This gives us
    \begin{align}
        \mathbb{E} \left( \frac{r_{x_j,i} }{r_{x_j + \epsilon_i,i}} \right)  & =  \frac{(n_{d_i} - 1) D(x_j)}{(n_{d_i} - 1) D(x_j + \epsilon_i)} +  \frac{(n_{d_i} - 1) D(x_j) (n_{d_i} - 1) D(x_j + \epsilon_i)(1 - D(x_j+ \epsilon_i))  }{(n_{d_i} - 1)^3 D(x_j + \epsilon_i)^3 } \, , \\
        & = \frac{ D(x_j)}{ D(x_j + \epsilon_i)} + \frac{ D(x_j)  (1 - D(x_j+ \epsilon_i))  }{(n_{d_i} - 1) D(x_j + \epsilon_i)^2 } \, , 
    \end{align}
    As $n_{d_i}$ goes to infinity, the second term goes to zero and as $\epsilon_i$ goes to 0, we have the result. Continuity of $D(x)$ is used to say that $D(x_j + \epsilon_i)$ goes to $D(x_j)$. A discontinuous function with a gap at $x_j$ does not satisfy this condition. 

   Excluding small degrees such as 0, 1, and 2, we can think of consecutive degrees as realizations of nodes sampled with probabilities $D(x_j)$ and $D(x_j + \epsilon_i)$ for some $x_j$ where $D(x_j) > \xi > 0$. The condition $D(x_j) >  \xi >0$ guarantees that we stay away from small degrees. As $i$ goes gets larger, more and more nodes are sampled from the graphon $W$. As such $\epsilon_i$ goes to zero as $i$ tends to infinity. Therefore, the ratio of consecutive degrees go to 1. 
 \end{proof}

\PropestimatingKOne*
\begin{proof}
  Let $\tilde{q}_{j,i}$ be the degree of the star in $G_{s_i}$ corresponding to $p_j \neq 0$. Let $q_{j,i}$ denote the degree of the corresponding vertex in $G_{n_i}$. Then from the Order Preserving Property (Proposition \ref{prop:rankpreserving}) we know
   \[  \lim_{i \to \infty} P\left( \bigcap_{j = 1}^k \left( q_{j, i} = \deg_{G_{n_i}} v_{(j)} \right) \right) = 1 \, . 
   \]
   Combining with Proposition \ref{prop:krelated1} with high probability we have
\begin{equation}\label{eq:kest1}
 \log \deg_{G_{n_i}} v_{(k)} - \log(1 + \alpha ) n_{d_i} \geq \log  \deg_{G_{n_i}} v_{(\ell)} - \log \deg_{G_{n_i}} v_{(\ell + 1)} \, ,   
\end{equation}
 for $\ell \in \{1, \ldots, k-1\}.$  Let $\text{max\_deg}_{\widetilde{G}_{d_i}}$ denote the maximum degree in $G_{n_i}$ restricted to nodes from $G_{d_i}$. That is, $\text{max\_deg}_{\widetilde{G}_{d_i}}$ belongs to a node from the dense part $G_{d_i}$. For $W = 1$ we have $\text{max\_deg}_{\widetilde{G}_{d_i}} = n_{d_i} -1$ and 
\[
\mathbb{E} \left( \text{max\_deg}_{\widetilde{G}_{d_i}} \right) \leq (1 + c')n_{d_i} \, . 
\]
But for all other $W \neq 1$,  $\text{max\_deg}_{\widetilde{G}_{d_i}} < < n_{d_i}$.  For large enough $i$,  $(k+1)$st highest degree 
\[
\deg_{G_{n_i}} v_{(k+1)}  = \text{max\_deg}_{\widetilde{G}_{d_i}} \, . 
\]
Using Chernoff bounds, the probability
\[
P\left( \deg_{G_{n_i}} v_{(k+1)} > (1+ \epsilon)\mu \right) \leq \exp\left( -\frac{\mu \epsilon^2}{3}\right) \, , 
\]
making
\[
P\left(\deg_{G_{n_i}} v_{(k+1)} > (1 + \alpha) n_{d_i}  \right) \rightarrow 0 \, , 
\]
for appropriate $\alpha$ and $W \neq 1$. With high probability we have
\begin{align}
    (1 + \alpha)n_{d_i} &\geq \deg_{G_{n_i}} v_{(k+1)} \, , \\
    \log (1 + \alpha)n_{d_i} &\geq \log \deg_{G_{n_i}} v_{(k+1)} \, , \\
   \log \deg_{G_{n_i}} v_{(k)} - \log \deg_{G_{n_i}} v_{(k+1)}  & \geq \log \deg_{G_{n_i}} v_{(k)} -  \log (1 + \alpha)n_{d_i} \, , \\
   & \geq \log  \deg_{G_{n_i}} v_{(\ell)} - \log \deg_{G_{n_i}} v_{(\ell + 1)}  \, , 
\end{align}
where we have used equation \eqref{eq:kest1}. In the logscale the difference between $k$th highest vertex and $(k+1)$st highest vertex is larger than the $\ell$th highest vertex and the $(\ell +1)$st highest vertex for $\ell < k$.  That is, in the log scale, successive difference of the top $k$ degrees is highest at $k$. From Lemma \ref{lemma:krelated2} we know that for degrees in the dense part $G_{d_i}$, the ratios of successive degrees converge to 1, when we exclude small degrees. As we are taking ordered values $\deg_{G_{n_i}} v_{(\ell)}/\deg_{G_{n_i}} v_{(\ell + 1)}) \approx p_\ell/p_{\ell + 1} > 1$.  The  ratio $ \deg_{G_{n_i}} v_{(k)} /\deg_{G_{n_i}} v_{(k+1)} \approx m_{s_i}p_k / n_{d_i} D_{\max}$ where $D_{\max} = \max D(x)$. Thus, the log difference $ \log \deg_{G_{n_i}} v_{(k)} - \log \deg_{G_{n_i}} v_{(k+1)}$  is much greater than 1 for $i > I_0$. 
Therefore we have
\begin{equation}\label{eq:kmax}
     k = \max_\ell \left( \log  \deg_{G_{n_i}} v_{(\ell)} - \log \deg_{G_{n_i}} v_{(\ell + 1)} \right) \, ,
\end{equation}
with high probability.
\end{proof}

% \subsection{Piecewise linear regression and change point detection}\label{sec:fittinglines}

% \begin{lemma}\label{lemma:2lines}
% Suppose $(i, y_i)_{i=1}^M$  and $(i, y_i)_{i= M+1}^N$ be two sets of points that satisfy $y_i = m_1i + c_1 + \epsilon_i$  and $y_i = m_2i + c_2 + \eta_i$ respectively where $m_1 \neq m_2$, $c_1 \neq c_2$ and $\epsilon_i$ and $\eta_i$ are small for all $i$. Suppose we fit 2 line segments to $(i, y_i)_{i=1}^N$ at different break points $(j, y_j)$ for $j \in \{2, \ldots, N-1\}$ and compute the mean squared error (MSE) of the fit.  Then, the minimum MSE is achieved when the breakpoint is at $(M, y_M)$, i.e., the first line is fitted to $(i, y_i)_{i=1}^M$ and the second line is fitted to  $(i, y_i)_{i= M+1}^N$.
% \end{lemma}
% \begin{proof}
%     This is a standard result in changepoint detection \citep{killick2012optimal} and piecewise linear regression \citep{mczgee1970piecewise}. 
% \end{proof}

% Generally, changepoint detection methods do not assume the number of changepoints in advance. However, if we know the number of changepoints, then the task is much easier and the true changepoints do minimize the cost function, which in our case is the MSE.  Consequently, Lemma \ref{lemma:2lines} can be extended to multiple line segments. For instance, if we know that there are 2 sets of points $(i, y_i)_{i=1}^{M_1}$ and $(i, y_i)_{i=M_1+1}^{M_2}$  with different statistical properties, then fitting 3 lines at the appropriate breakpoints minimize the MSE.   

\begin{lemma}\label{lemma:3degreedistributions}{\textbf{(The 3 Degree Groups)}}
 Suppose $\{G_{n_i}\}_i$ is a sparse $(U,W)$-mixture graph sequence (Definition \ref{def:WURandomMixtureGraphs}) with $W$ having a continuous degree function (Definition \ref{def:degreefunction}).  Let $\bm{p} = (p_1, p_2, \ldots )$ be the mass-partition (Definition \ref{def:masspartition}) associated with $U$. Then there are 3 groups of nodes satisfying different degree expectations and variances. They are (1) \textbf{Group 1}: the large degrees generated from $U$  (2) \textbf{Group 2}: the degrees of nodes generated by $W$ and (3) \textbf{Group 3}: the very small degrees generated from $U$. 
\end{lemma}
\begin{proof}
The sparse part $G_{s_i}$ is a collection of stars with large degree hub nodes and degree-1 leaf nodes. Let $q_{j, i}$ be the degree of the hub node in $G_{n_i}$ corresponding to $p_j$. These are the Group 1 nodes. Then from Lemma \ref{lemma:WUrandomgraphsaboutU2}
\begin{align}
     \mathbb{E}(q_{j, i}) & = m_{s_i}p_j + \frac{c m_{new_i}}{n_{s_i}} \, , \\
    \Var(q_{j, i}) & = m_{s_i} p_j(1 - p_j)  + m_{new_i}\frac{c}{n_{s_i}}\left( 1 - \frac{c}{n_{s_i}} \right) \, . 
\end{align}
    In addition the the large degree hub nodes, the sparse part $G_{s_i}$ has a large number of nodes with degree 1. Let us denote the degree of these nodes in $G_{n_i}$ as $e_i$. These are Group 3 nodes. Their degree can change as a result of the joining process (Definition \ref{def:joiningrules}). Then
  \begin{align}
       \mathbb{E}(e_i) & = 1 + \frac{c m_{new_i}}{n_{s_i}} \, , \\
    \Var(e_i) & = m_{new_i}\frac{c}{n_{s_i}}\left( 1 - \frac{c}{n_{s_i}} \right) \, . 
  \end{align} 

  Next we consider nodes generated from $W$, the Group 2 nodes. Combining Lemma  \ref{lemma:WUrandomgraphs4} and Lemma \ref{lemma:krelated2} we know that nodes generated by $W$ have degree
   \begin{align}
        \mathbb{E}(r_{x_j,i}) & = (n_{d_i} - 1) D(x_j)  + \frac{c m_{new_i}}{n_{d_i}}\, , \\
        \Var(r_{x_j,i}) & = (n_{d_i} - 1) D(x_j)(1 - D(x_j)) + m_{new_i}\frac{c}{n_{d_i}}\left( 1 - \frac{c}{n_{d_i}} \right) \, .
    \end{align}
    where $r_{x_j,i} $ denotes the degree of a node in $G_{n_i}$ corresponding to $D(x_j)$. 
    Thus, the degree expectations and variances of these 3 groups of nodes are different.
    Furthermore as $n_{s_i}/n_{d_i}$ goes to infinity for sparse graphs, we have
    \begin{align}
        \mathbb{E}(q_{j, i}) & >  \mathbb{E}(r_{x_j,i})  >  \mathbb{E}(e_i) \, , \\
        \Var(q_{j, i}) & >  \Var(r_{x_j,i})  >  \Var(e_i) \, ,
    \end{align}
    for large $i$. 

\end{proof}

%% file: 10_E_Infinite_U.tex
\section{Estimating $U$ for infinite partitions}

\begin{lemma}\label{lemma:sparsepartsteeperline}
    Let $\{G_{n_i}\}_i$  be a sparse sequence of $(U,W)$-mixture graphs. Let $\bm{p}$ be the mass-partition corresponding to $U$ and suppose $\bm{p}$ has infinite non-zero elements.  %For a given graph $G_{n_i}$, let $G_{s_i}$ and $G_{d_i}$ denote the sparse and dense parts. 
    Let $k_i$ be the number of nodes in $G_{n_i}$ that are generated by $U$  with degrees higher than those generated by $W$. Let $\deg_{G_{n_i}} v_{(j)}$ denote the $j$th highest degree in $G_{n_i}$. 
    Then the line fitted to the points $(j, \log \deg_{G_{n_i}} v_{(j)})_{j = 1}^{k_i}$ using OLS regression is steeper than the line fitted to $\left\{ \left(j, \log \frac{m_{s_i}}{j} \right) \right\}_{j=1}^{k_i}$ for large enough $i$, i.e., if $\beta_{actual, i}$ and $\beta_{1, i}$  are the slopes of the lines fitted to points $\left\{(j, \log \deg_{G_{n_i}} v_{(j)})\right\}_{j = 1}^{k_i}$ and $\left\{\left(j, \log \frac{m_{s_i}}{j} \right)\right\}_{j = 1}^{k_i}$ respectively, then there exits $i > I_0$ such that
    \[
    |\beta_{actual, i}| > |\beta_{1, i}| \, . 
    \]
\end{lemma}
\begin{proof}
    First we show for any $\alpha >0$ and for some $N \in \mathbb{N}$ the OLS line fitted to the points $ \left\{\left(j, \log \frac{m_{s_i}}{j^{1 + \alpha}} \right) \right\}_{j = 1}^N$ is steeper than the OLS line fitted to the points $\left\{\left(j, \log \frac{m_{s_i}}{j} \right)\right\}_{j = 1}^N$.  For a set of points $\{(x_j, y_j)\}_{j = 1}^N$ the slope of the OLS line is given by
    \[
    \beta =  \frac{\sum_{j} x_j y_j - \bar{x}y_j }{\sum_j (x_j - \bar{x})^2} \, . 
    \]
    As the denominator of the slope $\beta$ is the same for both sets of points $\left\{\left(j, \log \frac{m_{s_i}}{j^{1 + \alpha}} \right) \right\}_{j = 1}^N$ and $\left\{\left(j, \log \frac{m_{s_i}}{j} \right)\right\}_{j = 1}^N$  we focus on the numerator. 
    Let $\beta_{num, \alpha} = \sum_{j} x_j y_j - \bar{x}y_j$ for the points $\left\{\left(j, \log \frac{m_{s_i}}{j^{1 + \alpha}} \right)\right\}_{j = 1}^N$ and $\beta_{num, 1} = \sum_{j} x_j y_j - \bar{x}y_j$ for the points $\left\{\left(j, \log \frac{m_{s_i}}{j} \right) \right\}_{j = 1}^N$. Then we show that
    \[
    \beta_{num, \alpha} < \beta_{num, 1} \, . 
    \]
    Noting that $\bar{x}$ denotes the mean of $x_i$ values, which is the same for both lines we compute
    \begin{align}
        \beta_{num, 1} - \beta_{num, \alpha} & = \sum_j  \left( j\log \frac{m_{s_i}}{j} - \bar{x}\log \frac{m_{s_i}}{j}  \right) - \sum_j \left( j \log \frac{m_{s_i}}{j^{1 + \alpha}} - \bar{x}\log \frac{m_{s_i}}{j^{1 + \alpha}}  \right) \, , \\
        & =\sum_j j   \left( \log \frac{m_{s_i}}{j} - \log \frac{m_{s_i}}{j^{1 + \alpha}} \right)  - \bar{x} \sum_j \left( \log \frac{m_{s_i}}{j}  - \log \frac{m_{s_i}}{j^{1 + \alpha}} \right)  \, , \\
        & = \sum_j \left( j \alpha \log j - \bar{x} \alpha \log j \right) \, , \\
        & = \alpha \sum_j \left( j \log j - \bar{x}  \log j \right) \, ,  \\
        & = \alpha \beta_{num, -1} \, , 
    \end{align}
    where $\beta_{num, -1} = \sum_{j} x_j y_j - \bar{x}y_j$  for the points $(i, \log j)_{j = 1}^N$. These points lie on the curve $y = \log x$  and as such its slope is positive.  This gives us
    \[
     \beta_{num, 1} - \beta_{num, \alpha} > 0 \, . 
    \]
    Dividing by the same denominator we get
    \[
    \beta_{1, i} > \beta_{\alpha, i} \, , 
    \]
    where $\beta_{\alpha, i}$ is the slope of the line fitted to the points $\left\{ \left(j, \log \frac{m_{s_i}}{j^{1 + \alpha}} \right)\right\}_{j = 1}^N$. But these slopes are negative because the $y$ values are decreasing. Therefore  
    \[
   |\beta_{\alpha, i}| > |\beta_{1, i}| \, .
    \]
    A mass-partition $\bm{p}$ with infinite non-zero entries  cannot decrease similar to $\left\{ c/j \right\}_{j}$ because the series $\sum_j c/j$ diverges to infinity. As $\sum_j p_j = 1$, the sequence $\{p_j\}_j$  has to decrease faster than $\left\{ c/j \right\}_{j}$. As shown above for a set of points decreasing faster then $\left\{ c/j \right\}_{j}$ the slope of the OLS line fitted to the log values is steeper than the slope of the line fitted to $\left\{\left(j, \log \frac{m_{s_i}}{j} \right) \right\}_{j = 1}^N$, making
\[
    |\beta_{actual, i}| > |\beta_{1, i}| \, . 
    \]
\end{proof}

For the next lemma we consider nodes with relatively large degrees generated by $W$. As such we consider the unique degree values and take degree values greater than a given percentile $C$. In practice $C$ can be the median unique degree value. 

\begin{lemma}\label{lemma:densepartflatterline}
    Let $\{G_{n_i}\}_i$  be a sparse sequence of $(U,W)$-mixture graphs with $W$ satisfying Assumption \ref{assump:1}. We consider nodes generated by $W$ with degrees greater than some percentile $C$ when considering unique degree values. Suppose there are $N_i$ such points in $G_{n_i}$. We consider $\{ (j,  \log( \deg_{G_{n_i}} v_{(\ell)}  ) ) \}_{j = 1}^{N_i}$   where the unique degree values are sorted in decreasing order. Then the line fitted to the points $\{(j,  \log( \deg_{G_{n_i}} v_{(\ell)}  ) )\}_{j = 1}^{N_i}$  is less steeper than the line fitted to the points $\{(j, \log n_{d_i}/j)\}_{j=1^{N_i}}$.  If $\beta_{dense,i}$ and $\beta_{1,i}$  are the slopes of the lines fitted to points $\{(j, \log \deg_{G_{n_i}}  v_{(j)})\}_{j = 1}^{N_i}$ and $\{(j, \log n_{d_i}/j)\}_{j = 1}^{N_i}$ respectively, then
   \[
    |\beta_{dense, i}| < |\beta_{1, i}| \, . 
    \]
\end{lemma}
\begin{proof}
    Assumption \ref{assump:1} is on $W$ having a continuous degree function $D(x)$ (Definition \ref{def:degreefunction}). When $W$ has a continuous degree function, for large enough $i$, the sorted unique degree values are mostly consecutive integers. Thus, the points $\{(j, \log \deg_{G_{n_i}}  v_{(j)})\}_{j = 1}^{N_i}$ have a large proportion of points in the form of $(j, \log(N_0 - j))$ for some $N_0$. The slope of the OLS line fitted to these points is much less steeper than the slope of the line fitted to $\{(j, \log n_{d_i}/j)\}_{j = 1}^{N_i}$. That is,
    \[
    |\beta_{dense, i}| < |\beta_{1, i}| \, . 
    \]
   % from Lemma \ref{lemma:krelatedTwo} the expectation of the ratio of consecutive degree values tend to 1. Thus, 
\end{proof}

\begin{lemma}\label{lemma:khatiIncreases}
Suppose $\{G_{n_i}\}_i$ is a sparse $(U,W)$-mixture graph sequence (Definition \ref{def:WURandomMixtureGraphs}) with $W$ having a continuous degree function (Definition \ref{def:degreefunction}).  Let $\bm{p} = (p_1, p_2, \ldots )$ be the mass-partition (Definition \ref{def:masspartition}) associated with $U$. Suppose U has infinite partitions. Let $G_{n_\ell}$ be a graph in the sequence $\{G_{n_i}\}_i$. We conduct Procedure \ref{proc:fit2lines} for  $G_{n_\ell}$ and estimate $\hat{k}_\ell$ using equation  \eqref{eq:khati}. Similarly we estimate $\hat{k}_i$ for each $G_{n_i}$.  Then as $i$ goes to infinity 
\[
\lim_{i \to \infty}P(\hat{k}_i > \hat{k}_{\ell}) = 1 \, . 
\]
% \[
%  P(\hat{k}_i > \hat{k}_{\ell}) \geq 1  - \exp (-c m_{s_i}) \quad \text{ \color{red} I feel this is } 1  - \exp (-c m_{s_i}/{n_{d_i}}) 
% \]
Furthermore, with high probability $\lim_{i \to \infty} \hat{k}_i = \infty$. That is,  for any $M \in \mathbb{N}$
\[
\lim_{i \to \infty} P(\hat{k}_i < M) = 0 \, . 
\]

\end{lemma}
%\lemmakhatiIncreases*
\begin{proof}
Let $G_{s_\ell}$ denote the sparse part of $G_{n_\ell}$.  Suppose the number of stars in the sparse part $G_{s_\ell}$ is given by $st_\ell$. We know that $\hat{k}_\ell < st_\ell$ because some stars have degrees that are indistinguishable from the dense part. Pick $r \in \mathbb{N}$ such that $\hat{k}_\ell <  r \leq st_\ell$. The star corresponding to $r$ has expected degree
\[
\mathbb{E}(q_{r, \ell}) = m_{s_\ell} p_{r} + \frac{c m_{new_\ell}}{n_{s_\ell}}
\]
from Lemma \ref{lemma:WUrandomgraphsaboutU2}. Furthermore as $\Var(q_{r, i})/\mathbb{E}(q_{r, i})^2$ goes to zero with $i$ the observed values of $q_{r, i}$ get more centered at its expected value. For sparse graphs $n_{s_i}/n_{d_i}$ goes to infinity. This results in $m_{s_i}/n_{d_i}$ going to infinity and 
\[
\mathbb{E}(q_{r, i}) = 
%q_{r,i} \approx 
m_{s_i} p_{r} + \frac{c m_{new_i}}{n_{s_i}} > C n_{d_i}
\]
for some $i$ where $C > 1$. Furthermore from Chernoff bounds 
\[
P( q_{r, i} < Cn_{d_i}) << P( q_{r, i} < (1 - \epsilon) \mu) ) = \exp\left(- \frac{\mu\epsilon^2}{2} \right)
\]
where $\mu = \mathbb{E}(q_{r, i})$. As $m_{s_i}$ goes to infinity, this probability goes to zero making
\[
\lim_{i \to \infty }P( q_{r, i} > Cn_{d_i}) = 1 \, .
\]
The nodes generated by $W$ have degrees less than $Cn_{d_i}$ for some $C$, and as $i$ goes to infinity, $q_{r, i}$ gets much larger than these nodes. As such, when fitting 2 lines as described in Procedure \ref{proc:fit2lines},  $q_{r, i}$ gets fitted with the first line for larger $i$, making $\hat{k}_i$ at least as big as $r$, which is greater than $\hat{k}_\ell$. Thus, 
\[
\lim_{i \to \infty}P(\hat{k}_i \geq r > \hat{k}_{\ell}) = 1 \, . 
\]
Next consider the graph sequence $\{G_{n_\ell} \}_{\ell}$. 
%and an increasing sequence $\{r_{\ell} \}_{\ell}$ such that $\hat{k}_\ell <  r_{\ell} \leq st_\ell$ where $st_\ell$ is the number of stars in the sparse part $G_{s_\ell}$. 
For a given $m_{s_\ell}$ the number of stars in the sparse part $G_{s_\ell}$ is finite, but as $\ell$ increases, $m_{s_\ell}$ increases and the number of stars $st_{\ell}$ increases making the sequence $\{st_{\ell}\}_{\ell}$ tend to infinity because $\bm{p}$ has infinite non-zero elements. Thus, we consider an increasing sequence $\{r_{\ell} \}_{\ell}$ such that $\hat{k}_\ell <  r_{\ell} \leq st_\ell$  and $\lim_{\ell \to \infty} r_{\ell} = \infty$. For every $r_{\ell}$ there exists $I_{\ell}$ such that for $i > I_{\ell}$ we have
\[
\hat{k}_{i} \geq r_{\ell} > \hat{k}_{\ell} \, , 
\]
with high probability. Therefore, the sequence %$\{\hat{k}_{i} \}_{i}$ 
\[
 \lim_{i \to \infty} \hat{k}_{i}  = \infty
\]
with high probability. That is, for any $M \in \mathbb{N}$
\[
\lim_{i \to \infty} P(\hat{k}_i < M) = 0 \, . 
\]
%Let $J$ denote the index at which this happens. Then $ \hat{k}_J > \hat{k}_{\ell}$.
\end{proof}

\proppjInfiniteK*
 \begin{proof}
    The proof is similar to that of Proposition \ref{prop:pj}.  Let $X_i = q_{j, i}$ and $Y_i = \sum_{\ell=1}^{\hat{k}_i} q_{\ell, i}$. The difference is that $k$ is replaced with $\hat{k}_i$.   Some of the expectations,  variances and covariance are changed to
        \begin{align}
        \mathbb{E}(X_i ) & = m_{s_i}p_j + \frac{c m_{new_i}}{n_{s_i}} \, ,  \\
        \mathbb{E}(Y_i ) & = m_{s_i}\sum_{\ell = 1}^{\hat{k}_i} p_{\ell}  + \frac{\hat{k}_i c m_{new_i}}{n_{s_i}} \, , \\
        \Var(X_i ) & = m_{s_i} p_j(1 - p_j)  + m_{new_i}\frac{c}{n_{s_i}}\left( 1 - \frac{c}{n_{s_i}} \right) \, , \\
        \Var(Y_i )    
        %\Var \left( \sum_{\ell = 1}^k m_{j, new_i} \right) \, , \\
        % & = \sum_{\ell = 1}^k \Var \left(m_{j, new_i} \right)  \, , \\
        & = m_{s_i}\sum_{\ell = 1}^{\hat{k}_i} p_{\ell} \left(1 - \sum_{\ell = 1}^{\hat{k}_i} p_{\ell} \right) +  \hat{k}_i  m_{new_i}\frac{c}{n_{s_i}}\left( 1 - \frac{c}{n_{s_i}} \right) \, ,\\
          \Cov(X_i, Y_i ) & =  m_{s_i} p_j(1 - p_j) + m_{new_i}\frac{c}{n_{s_i}}\left( 1 - \frac{c}{n_{s_i}} \right) \, . 
    \end{align}
    We turn out attention to the term $\hat{k}_i$ . We note that $\hat{k}_i \in O(m_{s_i}/n_{d_i})$.  The sequence $\{\hat{k}_i\}_i$  increases at a slower rate compared to $m_{s_i}/n_{d_i}$  depending on the rate $\{p_j\}_{j}$ decreases. For $\sum_j p_j$ to add to 1, $p_j$ cannot decrease like $c/j$ for come constant $c$, because $\sum_j c/j$ diverges to infinity. Hence $p_j \approx c/j^{1 + \alpha}$ is an option. The estimate $\hat{k}_i$ satisfies
    \[
    m_{s_i}p_{\hat{k}_i} \geq C n_{d_i} \, , 
    \]
    for some $C \in \mathbb{R}$ because the hub vertex corresponding to $p_{\hat{k}_i}$ needs have a larger degree than the degrees generated by $W$. Approximating  $p_j = c/j^{1 + \alpha}$ and letting $j \leq \hat{k}_i$ we have
   \begin{align}
       p_{j} & \geq C\frac{n_{d_i}}{m_{s_i}} \, , \\
       \frac{c}{j^{1 + \alpha}} & \geq C\frac{n_{d_i}}{m_{s_i}} \, , \\
       j^{1 + \alpha} & \leq C'\frac{m_{s_i}}{n_{d_i}} \, , \\
       j & \leq \left(C'\frac{m_{s_i}}{n_{d_i}}  \right)^{\frac{1}{1 + \alpha}} \, , 
   \end{align}
   making 
   \[
   \hat{k}_i =\max \left\{j : j  \leq \left(C'\frac{m_{s_i}}{n_{d_i}}  \right)^{\frac{1}{1 + \alpha}}  \right\} \, .
   \]
   As this is true for any $\alpha >0$ we have
   \[
   \hat{k}_i \in O(m_{s_i}/n_{d_i}) \, . 
   \]
This makes 
\begin{align}
    \mathbb{E}(Y_i ) & = m_{s_i}\sum_{\ell = 1}^{\hat{k}_i} p_{\ell}  + \frac{\hat{k}_i c m_{new_i}}{n_{s_i}} \, , \\
    \lim_{i \to \infty }\frac{1}{m_{s_i}}\mathbb{E}(Y_i ) &  = \lim_{i \to \infty } \sum_{\ell = 1}^{\hat{k}_i} p_{\ell}  + \frac{\hat{k}_i c m_{new_i}}{m_{s_i}n_{s_i}} \, , \\
    & = \sum_{\ell = 1}^{\hat{k}_i} p_{\ell} \, , 
\end{align}
following the same reasoning as in Proposition \ref{prop:pj} and because $\hat{k}_i \in O(m_{s_i}/n_{d_i})$.  That is,  $\hat{k}_i$ is bounded by the rate  $m_{s_i}/n_{d_i}$, and it is offset by $m_{new_i}/(m_{s_i} n_{s_i})$, which behaves like $n_{d_i}^2/m_{s_i}^2$. 
Similarly, 
\[
\lim_{i \to \infty } \frac{1}{m_{s_i}}\Var(Y_i )    
     = \sum_{\ell = 1}^{\hat{k}_i} p_{\ell} \left(1 - \sum_{\ell = 1}^{\hat{k}_i} p_{\ell} \right)  \, .\\
\]
Therefore the  term $\hat{k}_i$ does not affect the expectation or the variance estimates in a substantial way.  The rest is the same as in Proposition \ref{prop:pj}. 

As in Proposition \ref{prop:pj} using the second order Taylor series approximation of the expectation of a ratio of two random variables 
    \begin{equation}\label{eq:expectationofaratio}
        \mathbb{E}\left( \frac{X_i}{Y_i}    \right ) =   \frac{\mathbb{E}(X_i  }{\mathbb{E}(Y_i  )} - \frac{\Cov(X_i,Y_i )}{\mathbb{E}(Y_i)^2} +  \frac{\mathbb{E}(X_i ) \Var(Y_i)}{\mathbb{E}(Y_i)^3}    \, .
    \end{equation}
    we obtain 
  \begin{align}
      \lim_{i \to \infty} \mathbb{E}\left( \frac{X_i}{Y_i}   \right ) &= \lim_{i \to \infty} \frac{m_{s_i}p_j}{m_{s_i}\sum_{\ell = 1}^{\hat{k}_i} p_{\ell} } \, ,  \\
      & = \lim_{i \to \infty} \frac{p_j}{\sum_{\ell = 1}^{\hat{k}_i} p_{\ell} } \, ,  \\
      & =  p_j \, ,  
  \end{align}
and 
\[
\left \vert \mathbb{E}\left( \frac{X_i}{Y_i}   \right ) - p_j \right \vert  \leq \frac{c}{m_{s_i}} \, .
\]

  Similarly, the first order Taylor approximation of the variance is given by
 \[
    \Var\left(\frac{X_i}{Y_i}   \right) = \frac{ \mathbb{E}(X_i  )^2}{\mathbb{E}(Y_i  )^2} \left( \frac{\Var(X_i )}{ \mathbb{E}(X_i )^2 } - 2\frac{ \Cov(X_i,Y_i )}{ \mathbb{E}(X_i ) \mathbb{E}(Y_i) } + \frac{\Var(Y_i )}{ \mathbb{E}(Y_i )^2} \right) \, . 
    \]  
    The terms inside the parenthesis go to zero as $i$ goes to infinity making
\[
  \Var\left(\frac{X_i}{Y_i}   \right)  = 0\, , 
\]    
 and
 \[
 \Var\left(\frac{X_i}{Y_i}   \right)  \leq \frac{c}{m_{s_i}} \, . 
 \]
    
\end{proof}

\begin{lemma}\label{lemma:errorbounds}
    Let $\{G_{n_i}\}_i$ be a sequence of sparse $(U,W)$-mixture graphs (Definition \ref{def:WURandomMixtureGraphs}) with dense and sparse parts $G_{d_i}$ and $G_{s_i}$ respectively. Let $\bm{p} = (p_1, p_2, \ldots )$ be the mass-partition (Definition \ref{def:masspartition}) associated with $U$ which has infinite partitions. For a given graph $G_{n_i}$ we conduct Procedure \ref{proc:fit2lines} and estimate $\hat{k}_i$. Furthermore we estimate the mass partition $\hat{\bm{p}} = (\hat{p}_i, \hat{p}_2, \ldots )$ using equation \eqref{eq:hatpjforinfinite}. Suppose $p_i < 1/(i+1)^{1+\alpha}$, i.e., $\alpha$ gives an upper bound for the rate at which $p_i$ goes to zero. Then
    \[
    \sum_{j = \hat{k}_i}^{\infty} p_j \leq \frac{1}{\alpha \hat{k}_i^{\alpha}} \, . 
    \]
\end{lemma}
\begin{proof}
    We consider $f(x) = \frac{1}{x^{1+\alpha}}$ and have
    \begin{align}
        \frac{1}{(j+1)^{1 + \alpha}} & < \int_{j}^{j+1} \frac{1}{x^{1+\alpha}} \,  dx \, , \\
      \sum_{j = \hat{k}_i}^{\infty} p_j <   \sum_{j = \hat{k}_i}^{\infty}  \frac{1}{(j+1)^{1 + \alpha}} & \leq \int_{\hat{k}_i}^{\infty} \frac{1}{x^{1+\alpha}} \,  dx \, , \\ 
        & = \frac{1}{\alpha \hat{k}_i^{\alpha}} \, . 
    \end{align}
\end{proof}

%% file: 10_F_Experiments.tex
\section{Experiments}\label{app:experiments}

\subsection{Illustration with synthetic data}
\subsubsection{Degree prediction}

\begin{figure}[!ht]
    \centering
    \includegraphics[width=0.95\linewidth]{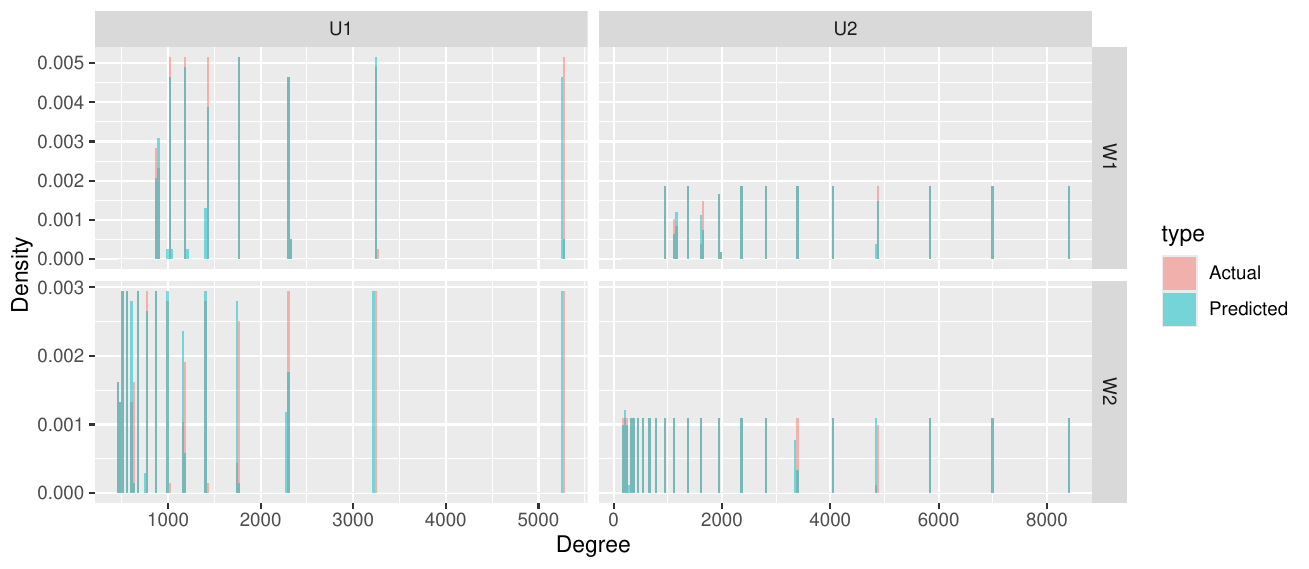}
    \caption{Densities of predicted degree using equation \eqref{eq:degreepred} for experiments with $(W_1, U_1)$, $(W_1, U_2)$, $(W_2, U_1)$ and $(W_2, U_2)$-mixture graphs.}
    \label{fig:experiment1to4Degrees}
    \centering
    \includegraphics[width=0.95\linewidth]{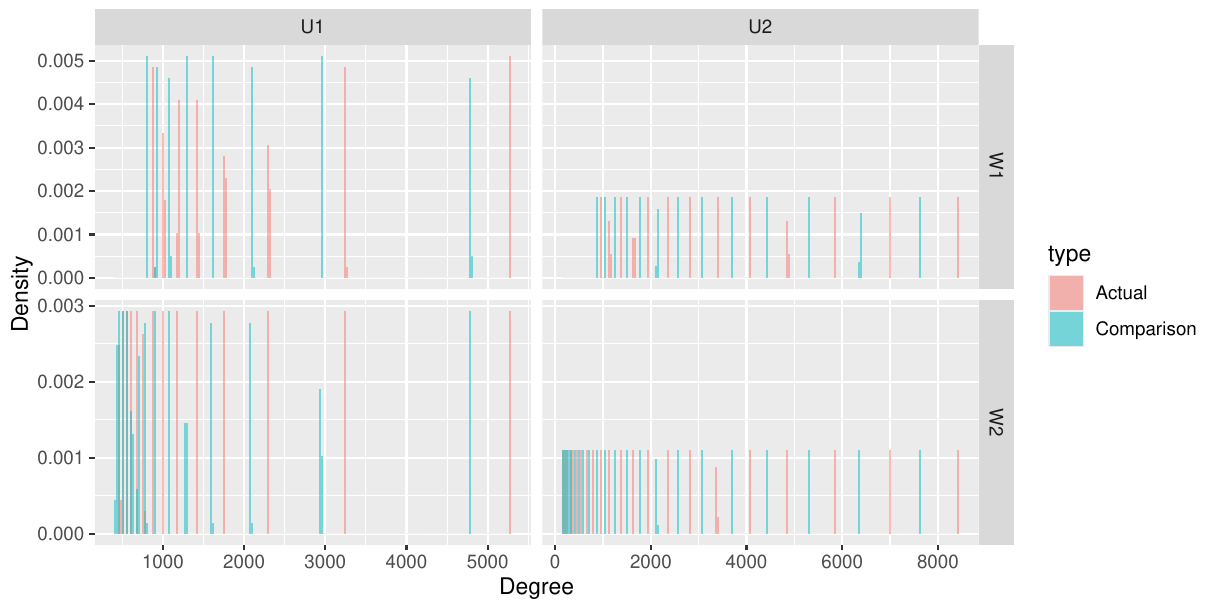}
    \caption{Densities of predicted degrees using equation \eqref{eq:degcompare} for experiments with $(W_1, U_1)$, $(W_1, U_2)$, $(W_2, U_1)$ and $(W_2, U_2)$-mixture graphs.}
    \label{fig:experiment1to4DegreesComparison}
\end{figure}

To predict the top-$k$ degrees using synthetic data we use $(U,W)$-mixture graphs with 4 combinations. We consider $W_1 = \exp(-(x+y))$, $W_2 = 0.1$  and $U_1$ with mass-partitions $\bm{p}_1 \propto \{1/j^{1.2}\}_{j=2}^{50}$ and $U_2$ with mass-partition $\bm{p}_2 \propto \{1/1.2^{j}\}_{j=2}^{50}$.  The mass-partitions are normalized so that they add up to 1.  For each $(U,W)$ combination we have training graphs with $n_{i} = 11,000$ nodes and test graphs with $n_{j} = 13,200$ nodes. We then predict the top-$k$ degrees in graph $G_{n_{j}}$ using equation \eqref{eq:degreepred} for  $k = \hat{k}_i$ in each instance. For the four experiments the estimated $\hat{k}_i$ values were  $8, 13, 14, 22$ respectively. 

As a baseline comparison method we used a scale-free graph property discussed in \cite{bollobas2001degree}.  For a scale-free graph with $n$ nodes the maximum degree behaves like $\Theta(\sqrt{n})$. Suppose $\deg_{G_{n_i}} v_{(\ell)}$ denotes the $\ell$th highest degree in $G_{n_i}$. Then, for an unseen graph $G_{n_j}$ with $n_j$ nodes we estimate the $\ell$th highest degree using
\begin{equation}\label{eq:degcompare}
   \deg_{G_{n_j}} \hat{v}_{(\ell)} = \deg_{G_{n_i}} {v}_{(\ell)} \times \sqrt{\frac{n_j}{n_i}} \, ,  
\end{equation}
as a means of comparison. Figures \ref{fig:experiment1to4Degrees} and \ref{fig:experiment1to4DegreesComparison} show the results of the four experiments using equations \ref{eq:degreepred} and \eqref{eq:degcompare} respectively. 

% \begin{figure}[!ht]
%     \centering
%     \includegraphics[width=0.95\linewidth]{Graphics/Degree_Prediction_Experiment_Comparison.pdf}
%     \caption{Predicted degree densities using equation \eqref{eq:degcompare} for experiments with $(W_1, U_1)$, $(W_1, U_2)$, $(W_2, U_1)$ and $(W_2, U_2)$-mixture graphs.}
%     \label{fig:experiment1to4DegreesComparison}
% \end{figure}

\subsubsection{Estimating $U$ with finite partitions}

\begin{figure}
    \centering
    \includegraphics[width=0.8\linewidth]{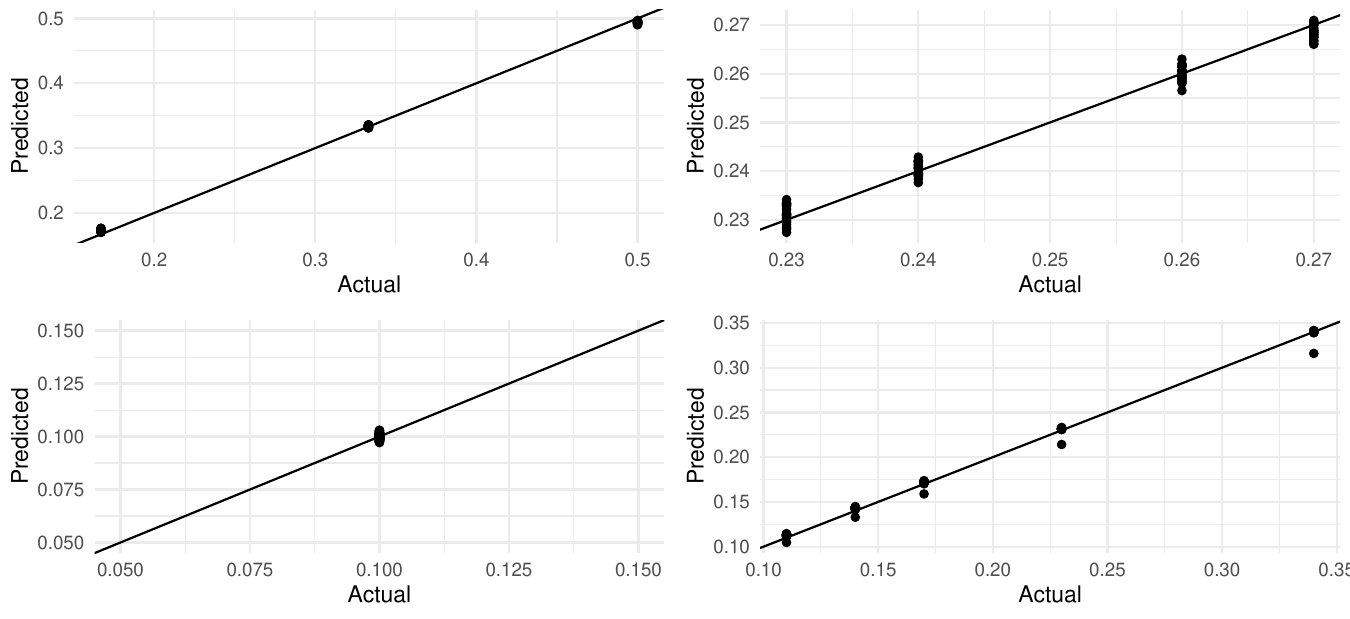}
    \caption{Experiments for $U$ with finite partitions. The actual $p_j$ in  mass-partition in $\bm{p} = (p_1, \ldots, p_n, 0, \ldots)$ is given on the $x$-axis and the predicted $\hat{p}_j$ using equation \eqref{eq:pjhat} on the $y$-axis.}
    \label{fig:FiniteUExperiment}
     \centering
    \includegraphics[width=0.8\linewidth]{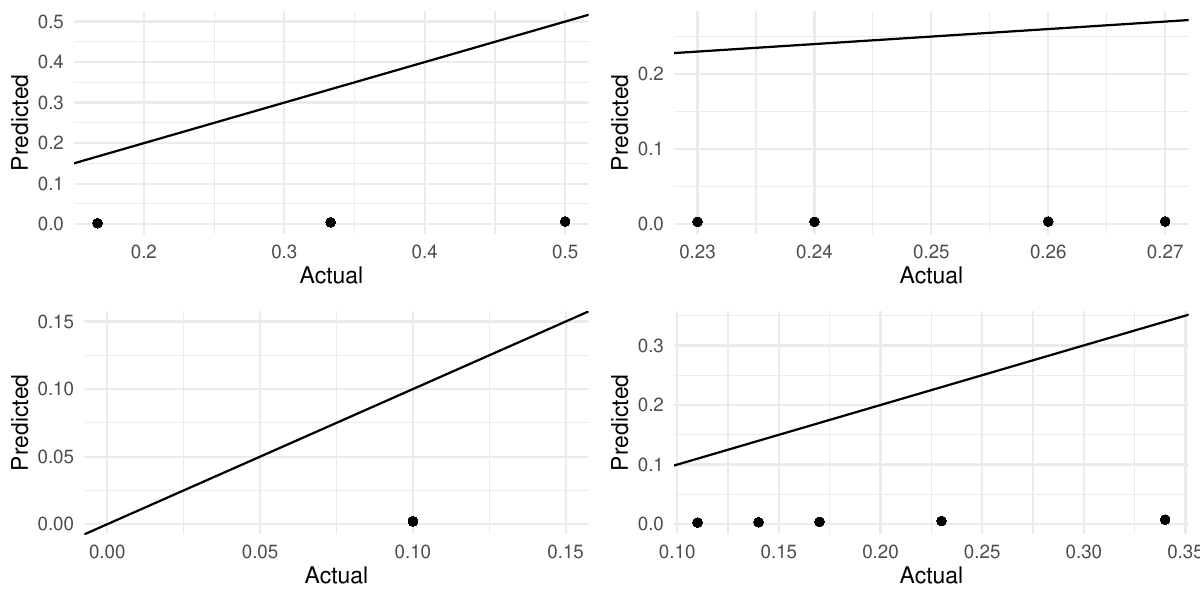}
    \caption{Experiments for $U$ with finite partitions. The actual $p_j$ in mass-partition in $\bm{p} = (p_1, \ldots, p_n, 0, \ldots)$ is given on the $x$-axis and the predicted  $\hat{p}_j$ using equation \eqref{eq:pjhatcompare} on the $y$-axis.}
    \label{fig:FiniteUExperimentComparison}
\end{figure}

We conduct 4 experiments with different mass-partitions $\bm{p}$ corresponding to $U$. We consider $\bm{p}_1 = \left(\frac{1}{2}, \frac{1}{3}, \frac{1}{6} \right)$ for experiment 1, $\bm{p}_2 = \left(0.27, 0.26, 0.24, 0.23 \right)$ for experiment 2, $\bm{p}_3 = \left(0.1, 0.1, \cdots, 0.1 \right)$, i.e., 0.1 repeated 10 times for experiment 3 and $\bm{p}_4 \propto \left( \frac{1}{2}, \frac{1}{3},  \frac{1}{4}, \frac{1}{5}, \frac{1}{6} \right)$ in experiment 4 where the mass-partition is obtained by dividing by $\sum_{i=2}^6 1/i$.  For each experiment we use a set of graphs with a given number of nodes and we estimate $\hat{k}_i$ using equation \eqref{eq:khati}, which gives the same result as equation \eqref{eq:khatfinite}.  Using $\hat{k}_i$ and equation \eqref{eq:pjhat} we estimate the mass-partition.

As a comparison method we use
\begin{equation}\label{eq:pjhatcompare}
   \hat{p}_j = \frac{ \deg_{G_{n_i}} v_{(j)} }{ \sum_{\ell = 1}^{n_i} \deg_{G_{n_i}} v_{\ell)} } \, , 
\end{equation}
i.e., the ratio of the $j$th highest degree to the sum of all degrees. Figures \ref{fig:FiniteUExperiment} and \ref{fig:FiniteUExperimentComparison} show the results of the 4 experiments using equations \eqref{eq:pjhat} and \eqref{eq:pjhatcompare} respectively.   

% \begin{figure}
%     \centering
%     \includegraphics[width=0.8\linewidth]{Graphics/Finite_U_Experiments_Comparison.pdf}
%     \caption{Experiments with finite non-zero elements in mass-partition $\bm{p} = (p_1, \ldots, p_n, 0, \ldots)$. The actual $p_j$ is given on the $x$-axis and the predicted $\hat{p}_j$ using equation \eqref{eq:pjhatcompare} on the $y$-axis.}
%     \label{fig:FiniteUExperimentComparison}
% \end{figure}

\subsubsection{Estimating $U$ with infinite partitions}

\begin{figure}[!ht]
    \centering
    \includegraphics[width=0.8\linewidth]{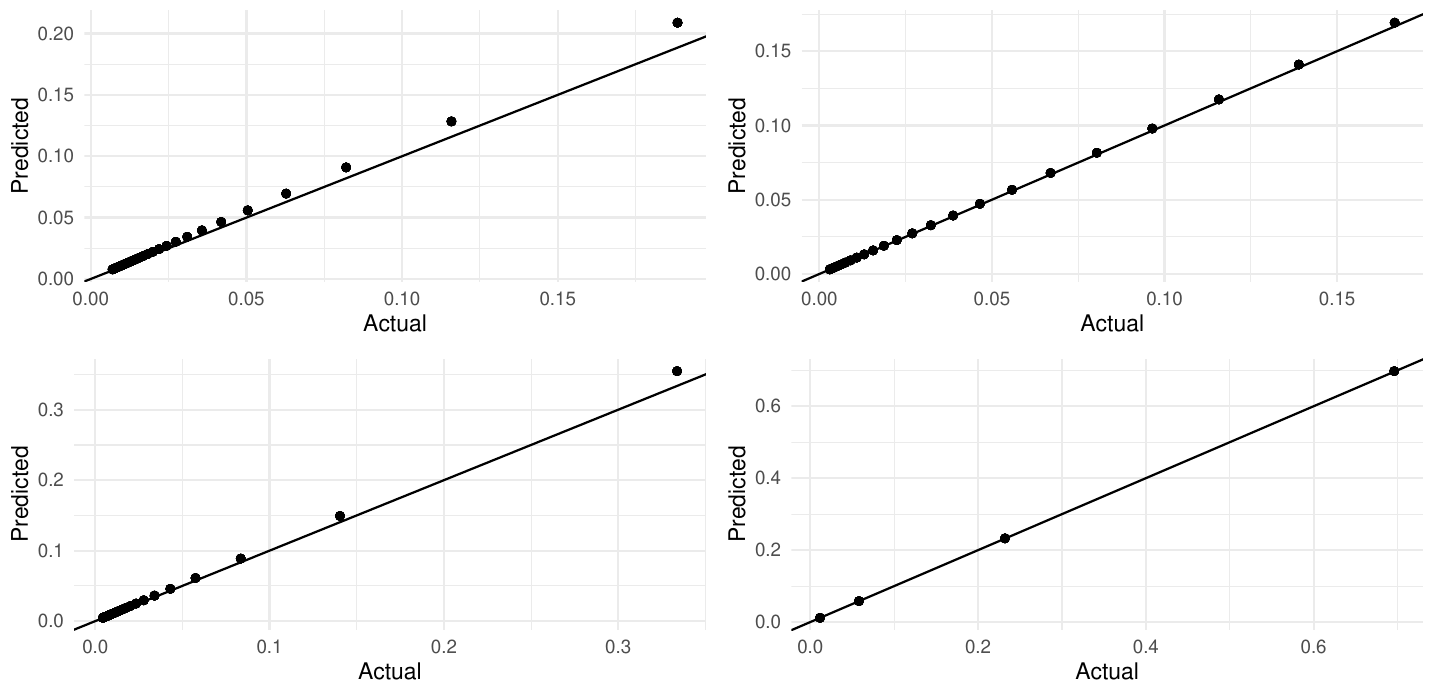}
    \caption{Experiments for infinite mass partition $\bm{p}$ with experiment 1 (top left), 2 (top right), 3 (bottom left) and 4 (bottom right). Equation \eqref{eq:pjhat} is used to estimate $p_i$.  }
    \label{fig:InfiniteUExperiments}
    \centering
    \includegraphics[width=0.8\linewidth]{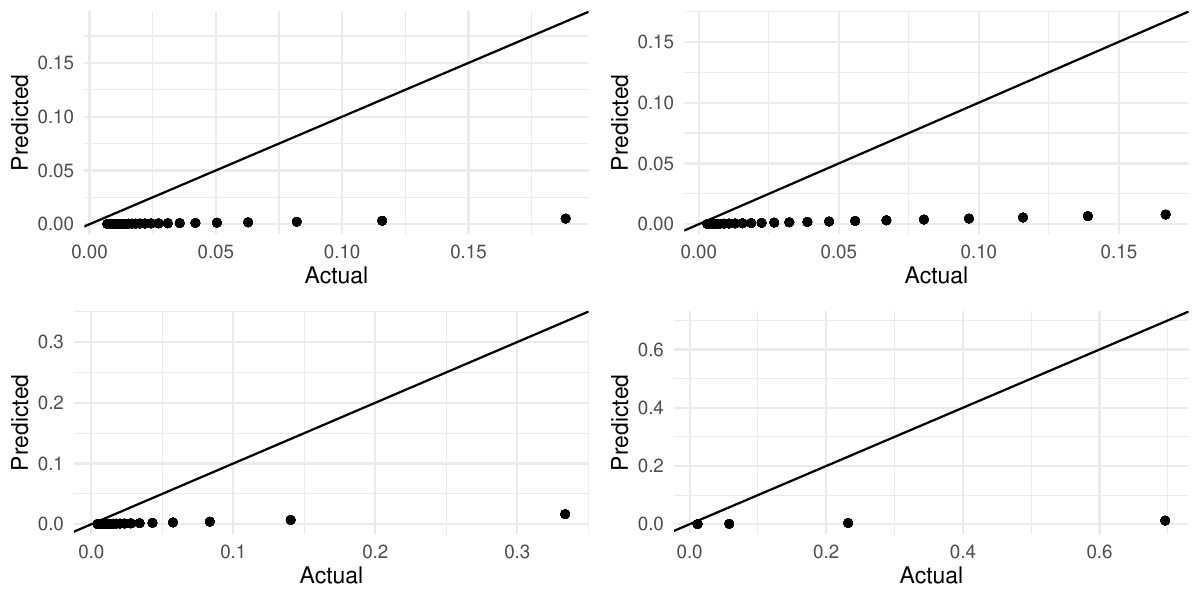}
    \caption{Experiments for infinite mass partition $\bm{p}$ with experiment 1 (top left), 2 (top right), 3 (bottom left) and 4 (bottom right). Equation \eqref{eq:pjhatcompare} is used to estimate $p_i$. }
    \label{fig:InfiniteUExperimentsComparison}
\end{figure}

In practice we cannot generate a mass-partition with infinite non-zero elements. As such we consider a finite number of partitions from an infinite sequence. For the 4 experiments we consider these sequences: for experiment 1 we want the $j$th element of the mass-partition $p_j \propto \frac{1}{{j + 1}^{1.2}}$. Noting that we cannot have a mass partition proportional to $\frac{1}{j}$ we consider a relatively low exponent of 1.2. For experiment 2 we consider $p_j \propto \frac{1}{{1.2}^{j+1}}$. This is a geometric series. For experiment 3 we consider $p_j \propto \frac{1}{{(j + 1)} \log(j+1)}$. For experiment 4 we consider $p_j \propto \frac{1}{(j+1)!}$. For all experiments we consider $\bm{p}$ to have 49 elements and rescale the mass-partition to add up to 1.  

For each experiment we use a set of graphs and using equation \eqref{eq:khati} we estimate $\hat{k}_i$, which are 30, 23, 30 and 4 for the four experiments. Figures \ref{fig:InfiniteUExperiments} and \ref{fig:InfiniteUExperimentsComparison} show the estimated and actual $p_j$ values using equations \eqref{eq:pjhat} and \eqref{eq:pjhatcompare} respectively.

For fast decreasing sequences such as $\frac{1}{(j+1)!}$ a smaller number of elements contribute to a larger sum  $\sum_{j=1}^{\hat{k}_i} p_j$  and as such both $k_i$ (the number of hubs with degrees greater than those generated by $W$) and the estimate $\hat{k}_i$ are small. Notwithstanding this,  the estimate $\hat{p}_j$ is quite accurate because the sum $\sum_{j=1}^{\hat{k}_i} p_j$ is quite large. In contrast, for slowly decreasing sequences such as $\frac{1}{{j + 1}^{1.2}}$, as the proportions decrease slowly, the estimate $\hat{k}_i$ is large, but the value $\sum_{j=1}^{\hat{k}_i} p_j$ is not that large. As such $\hat{p}_j$ deviates from $p_j$ quite a bit for small $j$. This is seen in the topleft subplot of Figure \ref{fig:InfiniteUExperiments}, which corresponds to experiment 1. The estimate $\hat{k}_i$ and the sum $\sum_{j=1}^{\hat{k}_i} p_j$ for these 4 experiments are given in Table \ref{tab:results2}.  

\begin{table}[!ht]
    \centering
        \caption{The estimates $\hat{k}_i$ and $\sum_{j=1}^{\hat{k}_i} p_j$ for  infinite $U$ experiments. }
    \begin{tabular}{cccccc} % cccccc p{2cm}p{2cm}p{2cm}p{2cm}p{2cm}p{2cm}
    \toprule
    Task & Description    &  Experiment 1 & Experiment 2 & Experiment 3 & Experiment 4  \\
    \midrule
    % \multirow{2}{*}{Top-$k$ deg.  }  & Proposed & 0.386 (0.337) &   0.225 (0.178) &  0.394 (0.288) &  0.341 (0.391) \\ 
    %      & Baseline &   9.078 (0.436) &  9.098 (0.238) & 9.185 (0.336) & 9.074 (0.488) \\
    % \hline
    % \multirow{2}{*}{Finite $U$} & Proposed &  1.800 (1.571)  & 0.602  (0.428) &   0.845   (0.644)  &   1.621  (1.448) \\
    % & Baseline & 98.709 (0.030)  &   98.702 (0.012) &    97.952 (0.024)  &   97.766 (0.025) \\
    %  \hline                   
    %  \multirow{2}{*}{Infinite $U$} 
    %                                         & Proposed & 10.917 (0.098) &  1.559 (0.069) &  6.351 (0.051) &     0.339  (0.171) \\
    %                                         & Baseline & 97.238 (0.002) &  95.233 (0.003 ) & 95.070   (0.003)  &  98.245  (0.003) \\
    % \hline \hline                                                    
    \multirow{2}{*}{Infinite $U$} 
                                             & $\hat{k}_i$ & 30 & 23 & 30 & 4 \\
                                             & $\sum_{j = 1}^{\hat{k}_i} p_j$ & 0.902 & 0.985 & 0.941 & 0.998 \\
    \bottomrule                                         
    \end{tabular}
    \label{tab:results2}
\end{table}

\subsection{Real datasets}

We use Facebook (FB) links \citep{viswanath2009activity}, Hep-PH Physics citations  \citep{leskovec2005graphs, gehrke2003overview}, MOOC interactions \citep{kumar2019predicting}, SMS \citep{wu2010evidence}, UCI Messages \citep{panzarasa2009patterns} and Yahoo messages \citep{yahoo} datasets. Each dataset is given as an edgelist with timestamped edges. The edges accumulate at different speeds in each dataset. In Yahoo messages dataset the edges accumulate quite fast and as such we use a 2-hourly time window to construct graphs. In contrast, edges accumulate slowly in Hep-PH citation dataset and we use monthly graphs. For the other datasets we use daily graphs. In each example we consider growing networks, i.e., new edges and nodes are added to the existing graph as time passes.  

Each dataset comprises a sequence of growing graphs $\{G_{n_i}\}_i$ and we consider $G_{n_i}$ for $i \in \{20, \ldots , 24 \}$ as training graphs. For each $G_{n_i}$ we select the test graph $G_{n_j}$ such that $j = \min_{k} \{ |G_{n_k}|: |G_{n_k}| - |G_{n_i}| \geq 500  \}$, i.e., $G_{n_j}$ is the first graph that has more than 500 nodes compared to $G_{n_i}$.  This is to ensure that the test graph is different from the training graph. Using the training graphs we predict the top-10 degrees of the test graphs.

For \cite{bollobas2001degree} we used equation \eqref{eq:degcompare} to predict degrees. \cite{caron2017sparse} have made their code available at \url{https://www.stats.ox.ac.uk/~caron/code/bnpgraph/}. We used their code to train their graph generation model on the training graph $G_{n_i}$ and using the trained model generated a graph having $n_k$ nodes, which we compared with $G_{n_k}$.  Kronecker graphs have 2 settings: deterministic and stochastic, and their code is available at \url{https://github.com/snap-stanford/snap/tree/master/examples/krongen}.  We generated Kronecker graphs with  (1) the default setting (stochastic $2\times 2$ initiator matrix), (2)  stochastic  setting with a $3 \times 3$ initiator matrix and  (3) the deterministic setting with a $3 \times 3$ initiator matrix.  For the stochastic settings, we first fitted a Kronecker model on the training graph $G_{n_i}$ and then using the fitted parameters generated the graph with $n_k$ nodes, which we compared with the test graph $G_{n_k}$. \cite{leskovec2010kronecker} explored the initiator matrix resulting from  a 3-star ($K_{1,3}$) with self loops added at each node. The adjacency matrix of this graph is  given by 
\begin{equation}\label{eq:KroneckerInitiator}
    A = \begin{pmatrix}
1 & 1 & 1 \\
1 & 1 & 0 \\
1 & 0 & 1
\end{pmatrix} \, . 
\end{equation}

We explored the deterministic case with this initiator matrix. Using the predicted top-$10$ degrees from each method, we computed the average MAPE and the standard deviation, which is given in Table \ref{tab:realworld}. We compared the performance  of the top 2 methods on each dataset using the Student's t-test  with $\alpha = 0.1$. For all datasets except UCI Messages we found that $(U,W)$-mixture was significantly better than \cite{bollobas2001degree}, which was the ranked the second. For UCI, while $(U,W)$-mixtures performed better on average, it was not significantly better. Figures \ref{fig:FB} to \ref{fig:Yahoo} show the actual and predicted degrees of the test graphs with Actual = Predicted line drawn. In some instances, the $y=x$ line is out of range.

\begin{figure}[!p]
    \centering
    \includegraphics[height=0.13\textheight]{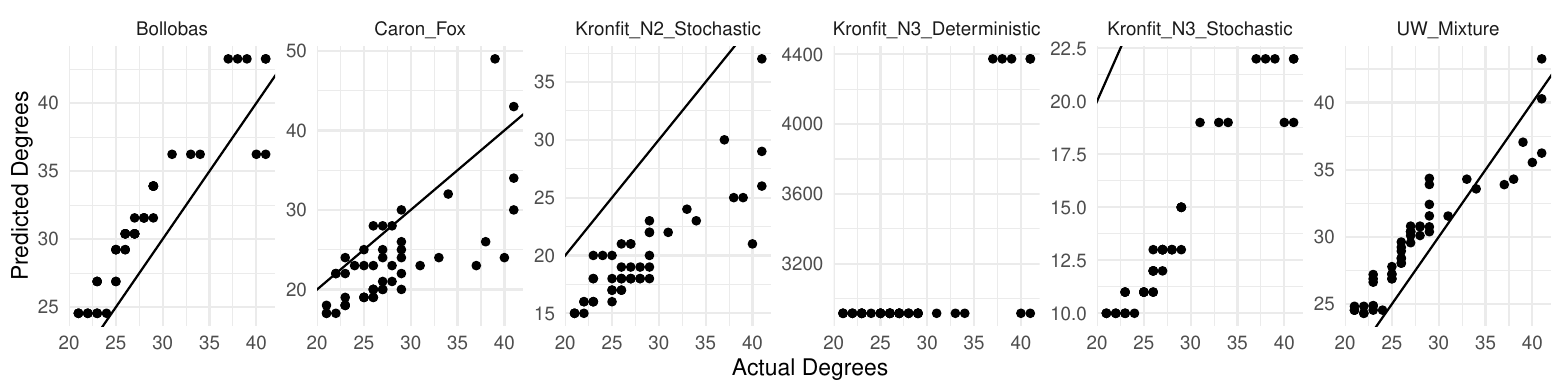}
    \caption{Facebook links dataset}
    \label{fig:FB}
    \includegraphics[height=0.13\textheight]{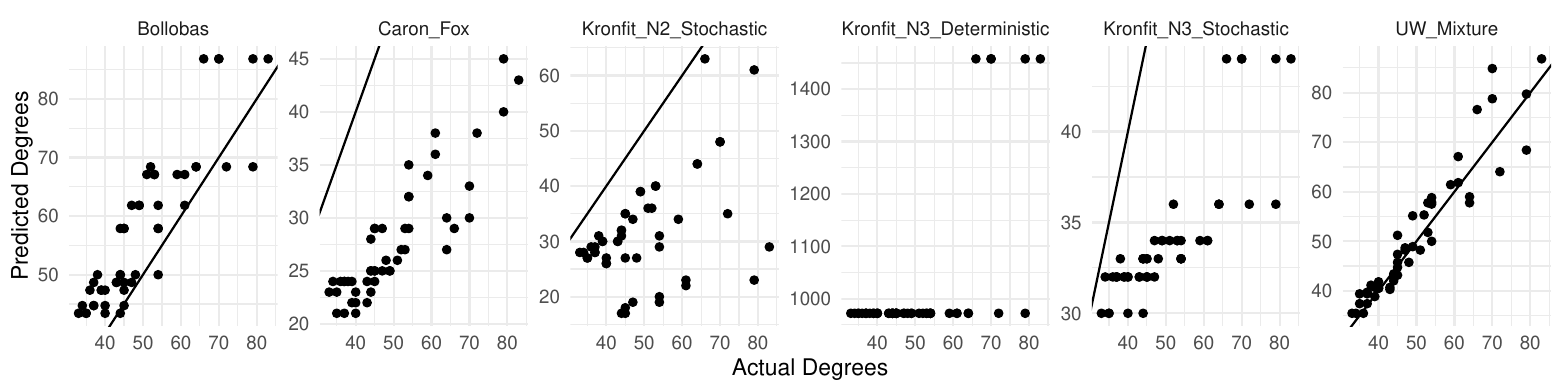}
    \caption{HEP-PH dataset}
    \label{fig:HEPPH}
    \includegraphics[height=0.13\textheight]{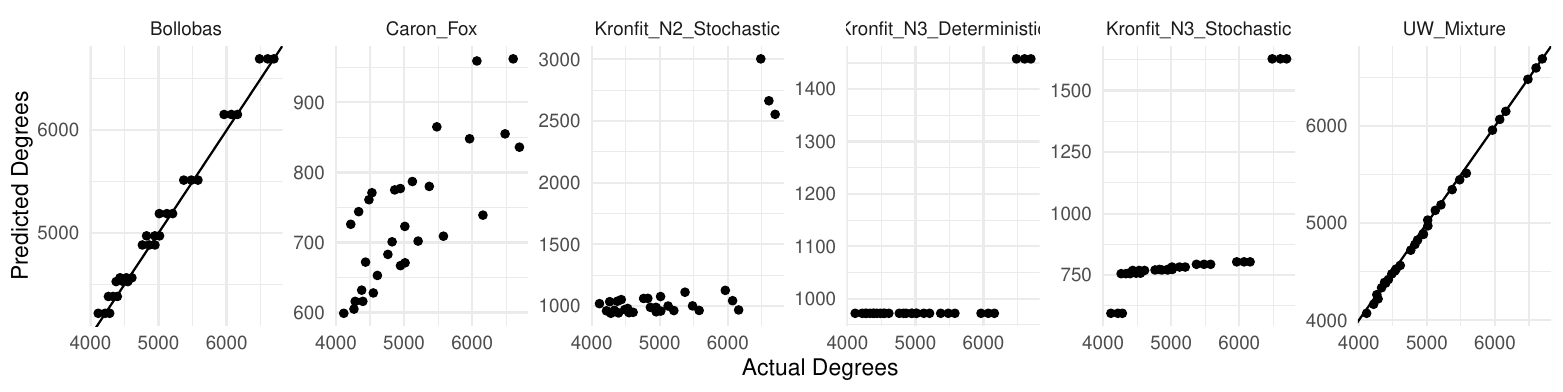}
    \caption{MOOC dataset}
    \label{fig:MOOC}
    \includegraphics[height=0.13\textheight]{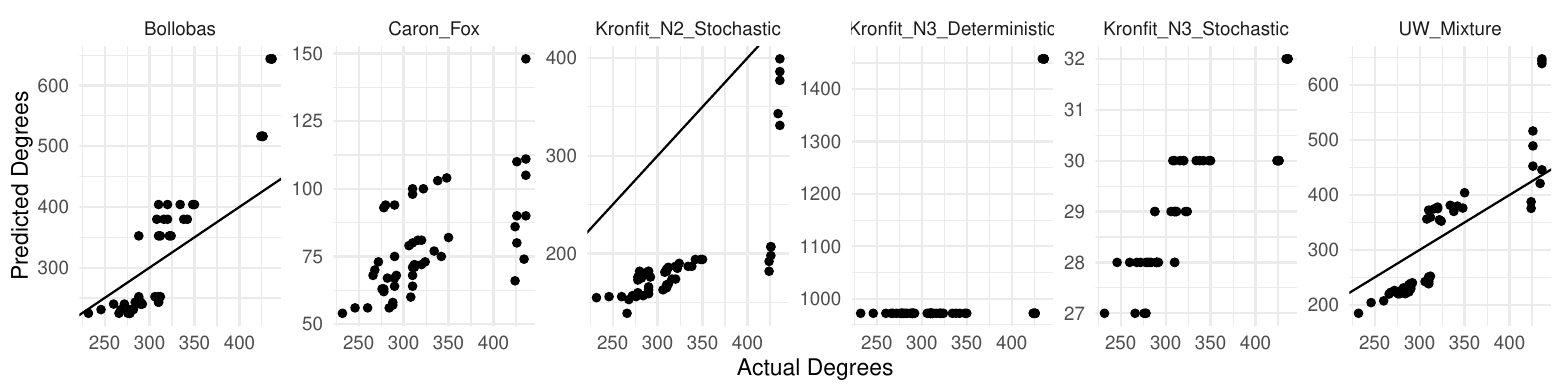}
    \caption{UCI Messages dataset}
    \label{fig:UCI}
    \includegraphics[height=0.13\textheight]{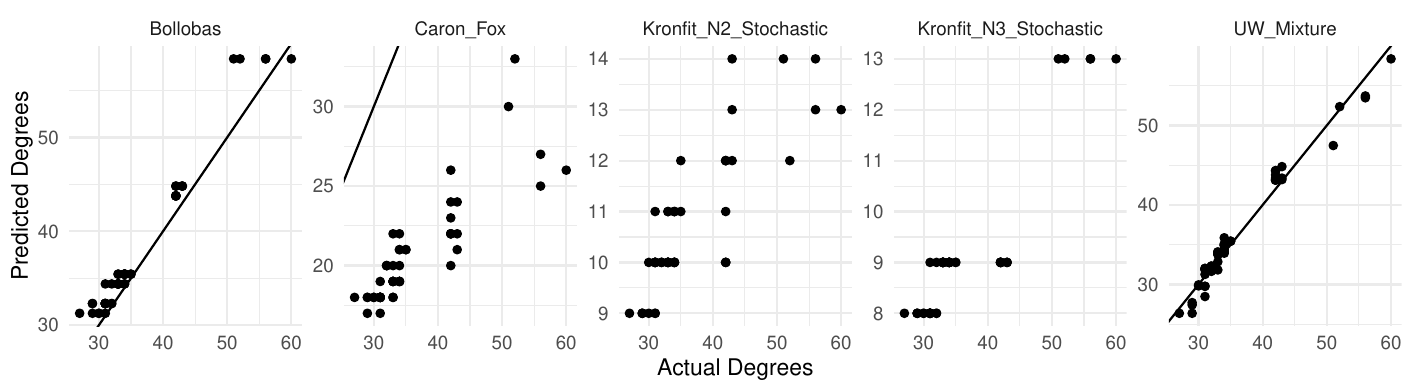}
    \caption{SMS dataset}
    \label{fig:SMS}
    \includegraphics[height=0.13\textheight]{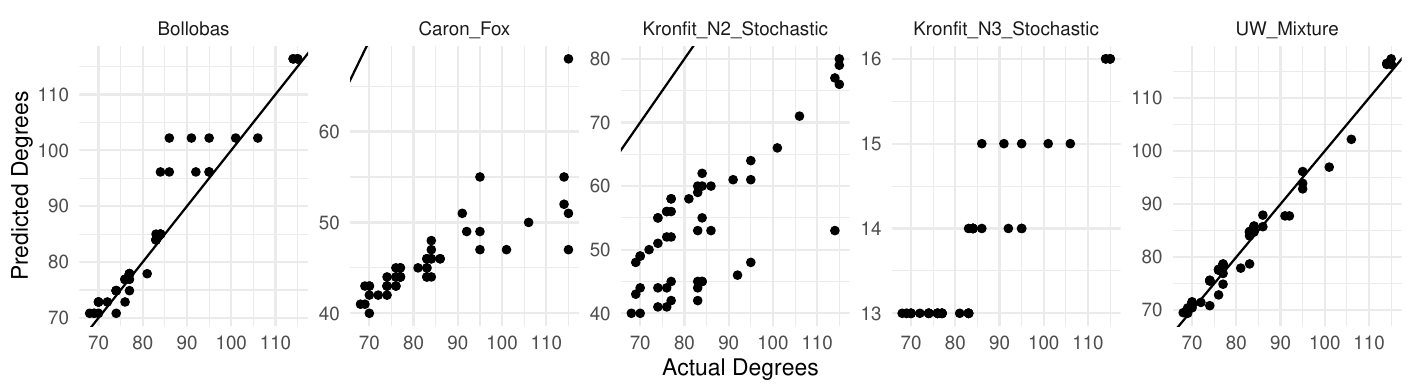}
    \caption{Yahoo dataset}
    \label{fig:Yahoo}
\end{figure}

\subsubsection{HEP-TH dataset comparison as in \cite{leskovec2005graphs}}

\cite{leskovec2010kronecker} have used the initiator matrix $A$ given in equation \eqref{eq:KroneckerInitiator} to generate graphs similar to the HEP-TH dataset using the deterministic setting.  For the stochastic setting, they have replaced 1s with $\alpha=0.41$ and 0s with $\beta=0.11$.  With these two initiators we ran their algorithm and verified we could produce the degree distribution in their paper. In their paper they say that Kronecker graph results qualitatively match the original results. We found this to be true in our simulation as well. Table \ref{tab:KroneckerHepTH} gives the actual and predicted top-10 degrees using $(U,W)$-mixture graphs, Kronecker deterministic and stochastic versions and the absolute percentage error $|\hat{y} - y| \times 100/y$. %For this exercise we did not use a training/test test. 

%Indeed, in Figure 7(a) top row we see that the actual largest degrees are in the $10^3$ range, whereas the deterministic version (the second row in the figure) produces the largest degrees in the $10^4$ range. The stochastic version (the third row in the figure) produces the largest degrees in the $10^2$ range. We observe these degree ranges in the kronfit predictions.

\begin{table}[t]
    \centering
    \caption{The top-10 degree comparison results on HEP-TH dataset where APE denotes the absolute percentage error $|\hat{y} - y| \times 100/y$.}
    \begin{tabular}{p{1cm}p{1.5cm}p{2cm}p{1.5cm}p{1.5cm}p{1.5cm}p{1.5cm}}
    \toprule
          Actual & Predicted $(U,W)$ degree & Predicted Kronecker Deterministic. & Predicted Kronecker Stochastic & APE $(U,W)$ & APE Kronecker Det. & APE Kronecker Stoch. \\
    \midrule
 2468 & 2249 & 32768 & 74 & 8.89 & 1227.71 & 97.00  \\
 1797 & 1633 & 16384 & 51 & 9.14 & 811.74 & 97.16 \\
 1653 & 1503 & 16384 & 47 & 9.09 & 891.17 & 97.16 \\
 1369 & 1281 & 16384 & 45 & 6.44 & 1096.79 & 96.71 \\
 1308 & 1253 & 16384 & 44 & 4.19 & 1152.60 & 96.64 \\
 1219 & 1147 & 16384 & 44 & 5.92 & 1244.05 & 96.39 \\
 1218 & 1124 & 16384 & 43 & 7.69 & 1245.16 & 96.47 \\
 1165 & 1101 & 16384 & 42 & 5.51 & 1306.35 & 96.39 \\
 1124 & 1019 & 16384 & 42 & 9.35 & 1357.65 & 96.26 \\
 1038 & 996 & 16384 & 42 & 4.01 & 1478.42 & 95.95 \\
 \bottomrule
    \end{tabular}
    \label{tab:KroneckerHepTH}
\end{table}